\documentclass{article}

\PassOptionsToPackage{numbers,sort&compress}{natbib}
\usepackage[final,main]{neurips_2026}

\usepackage[utf8]{inputenc}
\usepackage[T1]{fontenc}
\usepackage[colorlinks=true,linkcolor=magenta,citecolor=cyan!70!black,urlcolor=blue!70!black]{hyperref}
\usepackage{url}
\usepackage{booktabs}
\usepackage{amsfonts}
\usepackage{amsmath}
\usepackage{amssymb}
\usepackage{amsthm}
\usepackage{nicefrac}
\usepackage{microtype}
\usepackage[table,dvipsnames]{xcolor}   %
\usepackage{graphicx}
\usepackage[percent]{overpic}
\usepackage{float}
\graphicspath{{images/}}
\usepackage{multirow}
\usepackage{pifont}                     %
\usepackage{wrapfig}                    %
\usepackage{xspace}
\usepackage{subcaption}                 %
\usepackage{enumitem}                   %
\usepackage{colortbl}                   %
\usepackage{mwe}                        %
\usepackage{tikz}                       %
\usetikzlibrary{positioning}            %
\usepackage{pgfplots}                   %
\pgfplotsset{compat=1.18}
\usepackage[most]{tcolorbox}            %
\usepackage[capitalize,noabbrev]{cleveref}  %
\crefname{appendix}{Appendix}{Appendices}
\Crefname{appendix}{Appendix}{Appendices}
\crefname{equation}{Eq.}{Eqs.}
\Crefname{equation}{Eq.}{Eqs.}

\usepackage[toc]{appendix}
\usepackage{etoc}
\usepackage{minitoc}
\etocsettocstyle{\section*{Appendix}}{}

\newcommand{\sname}{StayFair\xspace}
\newcommand{\method}{\sname}

\newcommand{\eg}{\emph{e.g.}\xspace}
\newcommand{\ie}{\emph{i.e.}\xspace}

\newtheorem{theorem}{Theorem}
\newtheorem{lemma}{Lemma}

\theoremstyle{definition}
\newtheorem{definition}{Definition}

\AfterEndEnvironment{theorem}{\vspace{-\parskip}}
\AfterEndEnvironment{lemma}{\vspace{-\parskip}}
\AfterEndEnvironment{proposition}{\vspace{-\parskip}}
\AfterEndEnvironment{corollary}{\vspace{-\parskip}}
\AfterEndEnvironment{definition}{\vspace{-\parskip}}
\AfterEndEnvironment{assumption}{\vspace{-\parskip}}
\AfterEndEnvironment{remark}{\vspace{-\parskip}}

\newif\ifpreliminaryassub
\preliminaryassubfalse

\definecolor{targetshade}{HTML}{E3EEF9}
\newcommand{\sdmark}[1]{\raisebox{-0.35ex}{\scriptsize\textcolor{black!45}{$\pm$#1}}}
\newcommand{\pmv}[2]{#1\sdmark{#2}}
\newcommand{\pmvb}[2]{\textbf{#1}\sdmark{#2}}

\tcbset{
  keyeq/.style={
    enhanced,
    colback=blue!3!white,
    colframe=blue!45!black,
    boxrule=0.8pt,
    arc=3pt,
    colbacktitle=blue!45!black,
    coltitle=white,
    fonttitle=\bfseries\small,
    title={Bias decomposition},
    left=6pt, right=6pt, top=4pt, bottom=4pt, boxsep=1pt,
    before skip=8pt, after skip=8pt,
  }
}

\title{\textit{Stay Fair}! Ensuring Group Fairness in Diffusion Models Across Guidance Scales}

\author{%
Myeongsoo Kim$^{1}$ \quad
Eunji Kim$^{2,}$\thanks{The work was done prior to joining Amazon.}  \quad
Minwoo Chae$^{1}$ \quad
Sangwoo Mo$^{1}$ \\
$^{1}$POSTECH \quad
$^{2}$Amazon
}

\begin{document}
\maketitle
\setcounter{footnote}{0}
\renewcommand{\thefootnote}{\arabic{footnote}}

\etocdepthtag.toc{mtchapter}
\etocsettagdepth{mtchapter}{subsection}
\etocsettagdepth{mtappendix}{none}
\faketableofcontents

\begin{abstract}
Diffusion models steer conditional generation with a tunable guidance scale
to trade off prompt alignment and diversity.
However, existing debiasing techniques are optimized for a single scale, degrading fairness when users adjust this parameter.
We trace this behavior to a previously overlooked source by decomposing total bias into two components: a model bias and a guidance bias. While prior work primarily targets the former, we show that the \emph{guidance bias} varies with the guidance scale, persisting even in models that have already been debiased.
To address this, we extend Strong Demographic Parity to guidance and derive a condition under which the target distribution retains its group ratio across guidance scales.
\textcolor{black}{Building on this condition}, we propose \method, which designs fair guidance algorithms in both regimes.
For classifier guidance, it equalizes the classifier's output distributions across groups; for classifier-free guidance, it shifts the null embedding by a prompt-dependent offset.
Because \method modifies only the guidance step,
it is orthogonal to model debiasing and can be layered onto existing fair diffusion models.
Across class-conditional and text-to-image generation, \method decouples fairness from the guidance scale without sacrificing image quality.
\textcolor{black}{Our decomposition also enables a more detailed analysis of the sources of bias in diffusion models.}
\par\smallskip\noindent Project page: \url{https://kim-myeong-soo.github.io/stayfair/}
\end{abstract}

\section{Introduction}
\label{sec:intro}

Recent advances in diffusion models have enabled high-quality conditional generation, particularly in text-to-image synthesis~\citep{rombach2022high, esser2024scaling}. 
Guidance~\citep{dhariwal2021diffusion, ho2022classifier} is central to the widespread adoption of diffusion models, and its scale, chosen by the user, sets how strongly conditioning information influences generation.
However, higher guidance scales amplify not only the condition but also its inherent social bias, increasing major-group prevalence in generated samples~\citep{kim2025rethinking}.
This shift can therefore undermine group fairness by moving the distribution away from the target distribution, such as the data distribution~\citep{kim2024unbiased} or a uniform distribution~\citep{shenfinetuning}, depending on the notion of fairness.

This raises the question: \emph{Do existing models maintain their reported fairness when users change the guidance scale?}
Prior work typically evaluates fairness at a fixed scale, leaving its preservation across scales unexamined.
Social bias and its amplification in vanilla diffusion models are evaluated at a fixed guidance scale~\citep{seshadri2024bias}, while debiasing methods fit and evaluate models at a fixed scale~\citep{shenfinetuning}.
\cref{fig:motivation} illustrates the resulting concern: bias in the vanilla model increases with guidance scale, while the debiased model's group distribution drifts from its target.
\textcolor{black}{Thus, group fairness should not only hold at a single guidance scale but also remain robust as users adjust the scale to their preferences.}

We formalize this requirement as \emph{guidance scale-invariant group fairness}, extending group fair generation from a fixed scale to the range of scales users choose.
To understand why it breaks, we propose a \emph{bias decomposition} that separates total bias into a static \emph{model bias} and a scale-dependent \emph{guidance bias}.
\textcolor{black}{Unlike prior work targeting model bias at a fixed scale, we pursue \emph{fair guidance} to mitigate this scale-dependent guidance bias, thereby improving average and worst total bias as guidance varies.}
To provide a rigorous foundation, we extend Strong Demographic Parity (SDP)~\citep{jiang2020wasserstein} to guidance, identifying a condition under which the target distribution remains fair across guidance scales. 
\textcolor{black}{Building on this principle, we introduce} \method, which addresses guidance bias in both major guidance regimes: for classifier guidance (CG), it approximates \textcolor{black}{this condition} \textcolor{black}{via} a debiased classifier with group fairness constraints, while for classifier-free guidance (CFG), it uses an adaptive null prompt to rebalance bias without sacrificing prompt alignment, \textcolor{black}{across model architectures.}

\textcolor{black}{We evaluate whether \method stabilizes group ratios across scales while preserving image quality, and whether this improves group fairness.}
\textcolor{black}{We observe this behavior in class-conditional generation with CG on ADM-G~\citep{dhariwal2021diffusion} (\cref{sec:exp_cg}) and text-to-image generation with CFG on SD1.5~\citep{rombach2022high}, SD3~\citep{esser2024scaling}, and FLUX.2~\citep{flux-2-2025} (\cref{sec:exp_t2i}).}
\textcolor{black}{Moreover, \method mitigates average and worst total bias across guidance scales in vanilla models (\cref{sec:exp_t2i_vanilla}), and worst total bias while maintaining average debiasing performance in debiased models (FT~\citep{shenfinetuning} and UCE~\citep{gandikota2024unified}; \cref{sec:exp_t2i_debiased}).}
\textcolor{black}{Modifying the guidance mechanism is thus an effective way to mitigate guidance bias, even compared with existing debiasing methods extended across guidance scales.}
\textcolor{black}{We also examine how guidance bias contributes to asymmetric gender bias amplification~\citep{girrbach2025large} in SD3 and FLUX.2 (\cref{sec:discussion}).}
\textcolor{black}{\method further extends to multi-group attributes such as race and to unseen prompts (\cref{sec:exp_analysis}).}

\begin{figure}[t]
  \centering
  \includegraphics[width=\linewidth,trim=0 155 0 68,clip]{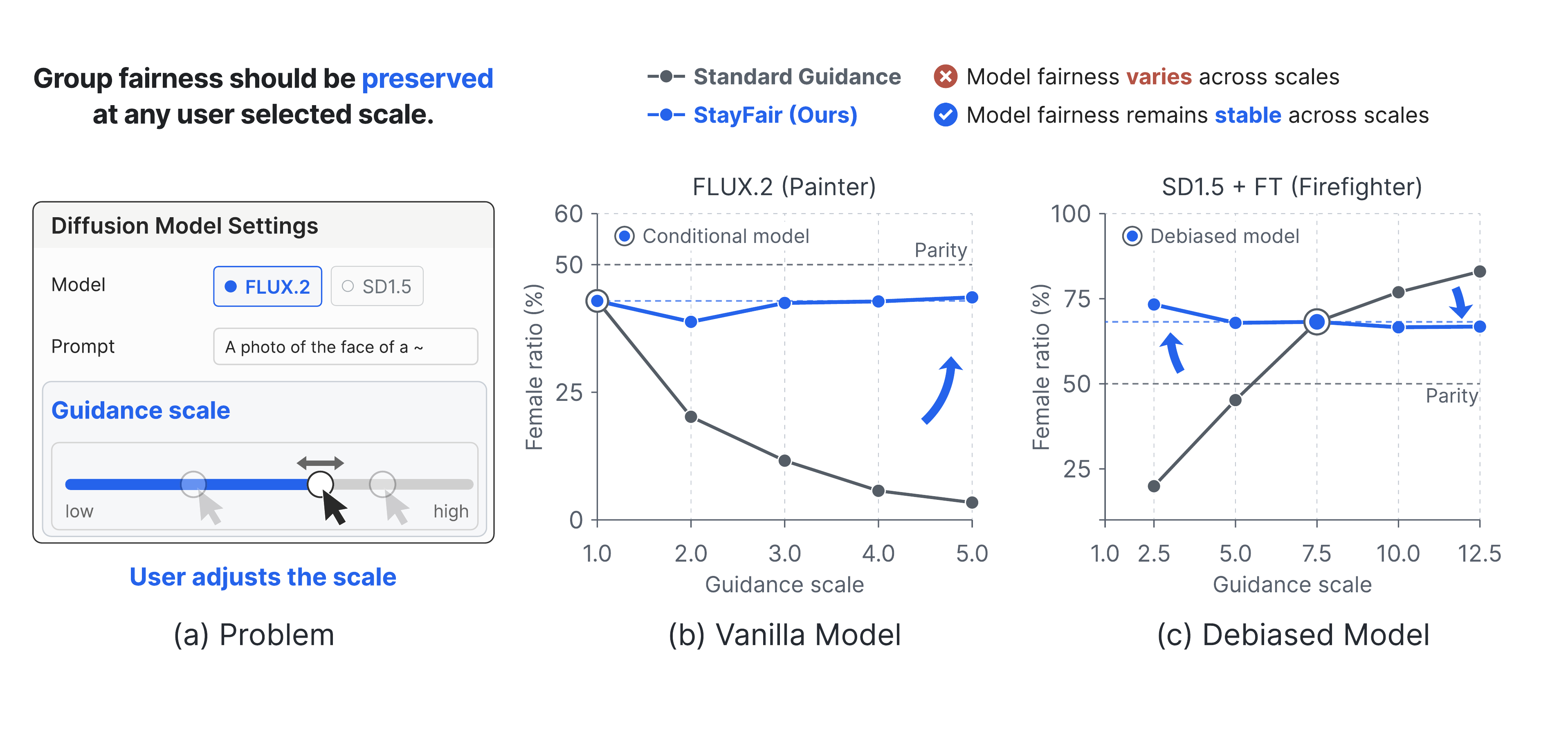}\\[1pt]
  \makebox[\linewidth][l]{%
    \hspace*{0.162\linewidth}\makebox[0pt][c]{\footnotesize (a) Problem}%
    \hspace*{0.338\linewidth}\makebox[0pt][c]{\footnotesize (b) Vanilla Model}%
    \hspace*{0.326\linewidth}\makebox[0pt][c]{\footnotesize (c) Debiased Model}}
  \caption{%
    \textbf{\method maintains group fairness as users vary the guidance scale.}
    (a) Users adjust the guidance scale to generate outputs that match their preferences.
    Since fairness should hold regardless of users' scale choices within the specified range, we extend the group fairness requirement across that range of scales.
    (b,c) Prior work that examines bias or optimizes fairness at a fixed guidance scale is insufficient for fairness across scales: group ratios drift in both vanilla (FLUX.2~\citep{flux-2-2025}) and debiased (SD1.5 + FT~\citep{shenfinetuning}) models.
    \method stabilizes group ratios across scales, mitigating average and worst total bias depending on the setting.
    }
  \label{fig:motivation}
  \vspace{-12pt}
\end{figure}

\section{Related Work} \label{sec:rw}

\textbf{Group fairness in diffusion models.}
Diffusion models exhibit biased distributions over sensitive attributes such as gender, race, age, and culture~\citep{cho2023dall, luccioni2023stable, bird2023typology, barroso2024racial}.
\textcolor{black}{Recent text-to-image (T2I) diffusion models often amplify these biases, exaggerating the group imbalances in their training data}~\citep{bianchi2023easily, seshadri2024bias}.
The underlying mechanisms remain poorly understood, as exemplified by the stronger amplification for male-associated occupations than for female-associated ones~\citep{girrbach2025large, wu2024stable}.
\method decomposes total bias into model and guidance components, helping explain bias amplification.

\textbf{Debiasing diffusion models.}
\textcolor{black}{One line of work, \emph{attribute-based methods}, conditions on sensitive attributes through guidance~\citep{friedrich2023fair, parihar2024balancing}, prompt embeddings~\citep{teo2025fairqueue}, or latent manipulation~\citep{li2024self} to control group distributions.}
\textcolor{black}{However, they risk visual artifacts or unintended changes to non-target content~\citep{teo2025fairqueue,park2025fair}.}
\textcolor{black}{Another line, \emph{train-time methods}, either fine-tunes the model with fairness-aware objectives}~\citep{shenfinetuning, kim2024destereotyping} \textcolor{black}{or updates specific modules} such as text embeddings~\citep{chuang2023debiasing, hanlightfair} or cross-attention layers~\citep{gandikota2024unified}.
\textcolor{black}{These methods mitigate learned biases by updating model parameters, allowing generation from the original prompt without additional attribute conditioning.}
\textcolor{black}{Both lines of work focus on fairness at a fixed guidance scale, which does not guarantee fairness across scales. \method addresses this gap by reducing guidance bias.}

\nocite{mo2019mining,kim2026rethinking}

\textbf{Guidance in diffusion models.}
Guidance improves image quality and condition alignment at the cost of diversity~\citep{rombach2022high, esser2024scaling}.
Common variants include CG~\citep{dhariwal2021diffusion, bansal2023universal, ye2024tfg}, CFG~\citep{ho2022classifier, kynkaanniemi2024applying, sadat2024eliminating}, and autoguidance (AG)~\citep{karras2024guiding, hong2023improving, ahn2024self}, differing in how the conditioning signal is constructed.
Recent work~\citep{kim2025rethinking, roos2026met} reports that varying the guidance scale also shifts demographic composition, exposing a fairness dimension to this tradeoff.
However, no prior method directly reduces guidance bias.
\method debiases guidance to maintain fairness across user-controlled scales.

\textbf{Group fairness in classifiers.}
Group fairness has been widely studied in computer vision~\citep{pessach2022review,kim2024discovering,bae2025reasoning}.
Generally, it requires that some statistic of the predictor be matched across sensitive groups.
Common criteria such as demographic parity~\citep{dwork2012fairness} and equalized odds~\citep{hardt2016equality} are defined through group-wise averages of predictions.
\textcolor{black}{These criteria are typically addressed through constrained classification~\citep{zafar2017fairness, agarwal2018reductions} or group-distributionally robust training~\citep{sagawa2020distributionally, idrissi2022simple}.}
A stronger notion, SDP~\citep{jiang2020wasserstein}, requires prediction distributions to match across groups.
Prior work enforces SDP by minimizing discrepancies between group-conditional distributions, using Wasserstein distance~\citep{chzhen2020fair}, total variation~\citep{farokhi2021optimal}, or optimal-transport objectives~\citep{kim2025fairness}.
Building on this view, we characterize the classifier fairness required to preserve group ratios under guidance and extend SDP to guidance functions.

\section{Decomposing Bias in Guided Diffusion}
\label{sec:prob}
In this section, we first review diffusion models and guidance (\cref{sec:prob_preliminary}).
We extend group fair generation from a fixed guidance scale to a range of user-selected scales and introduce a decomposition of total bias into model and guidance components (\cref{sec:prob_formulation}).
Finally, we outline two scenarios: guidance bias emerging in vanilla models and persisting in debiased models (\cref{sec:prob_impact}).

\subsection{Preliminary: diffusion models and guidance}
\label{sec:prob_preliminary}
\emph{Diffusion models}~\citep{ho2020denoising,songscore} generate clean data $X_0 \sim p_0(x)$ by reversing a forward corruption process $\{X_t\}$ that progressively perturbs $X_0$ with Gaussian noise.
Following the EDM formulation~\citep{karras2022elucidating}, we consider $X_t = X_0 + \sigma_t \epsilon$ with $\epsilon \sim \mathcal{N}(0,I)$, where $\sigma_t$ denotes the noise level at time $t$.
This induces the marginal density of $X_t$ as $p_t(x) = p_0(x) * \mathcal{N}(x; 0, \sigma_t^2 I)$.
In \textcolor{black}{the} probability flow formulation, generation proceeds by denoising $X_T$ into $X_0$ via 
\begin{equation}
    dX_t = -\dot{\sigma}_t\,\sigma_t\,\nabla \log p_t(X_t)\, dt\textcolor{black}{.}
    \label{eq:edm}
\end{equation}
We approximate the score using a noise-prediction network $\epsilon_\theta(x_t,t)$~\citep{songscore}.
We write the model score as $\nabla \log p_\theta(x_t)=-\epsilon_\theta(x_t,t)/\sigma_t$, suppressing $t$ when it is clear from the noisy input $x_t$.

\emph{Guidance} was introduced to sample from a target distribution $p^w_\theta(x_t \mid y)\propto p_\theta(x_t)p_\theta(y \mid x_t)^w$ at each timestep $t$, where the guidance scale $w$ controls the strength of the condition $y$.
It replaces the score in \cref{eq:edm} with a guided score $\nabla \log p^w_\theta(x_t \mid y)$:
\begin{align}
  \nabla \log p^w_\theta(x_t \mid y) 
  = \nabla \log p_\theta(x_t) + w\,\nabla \log f(x_t, t),
  \label{eq:guided_score}
\end{align}
where $f(x_t, t)$ is a guidance function that typically approximates $p(y \mid x_t)$. 
We denote by $X_t \sim q^w_t$ the marginal density induced by integrating \cref{eq:edm} with this guided score, and by $Q^w := q^w_0$ the actual generated distribution~\citep{zheng2024characteristic, li2025provable, azangulovadaptive}.

The regimes in this paper differ in how they realize $f$.
\emph{Classifier guidance}~\citep{dhariwal2021diffusion}
trains the time-dependent classifier $f_\phi(x_t, t)$ with noisy input $x_t$ and timestep $t$. 
\emph{Classifier-free guidance}~\citep{ho2022classifier}
takes an implicit classifier via Bayes' rule 
\begin{equation}
    \nabla \log p_\theta(y \mid x_t) = \nabla \log p_\theta(x_t \mid y) - \nabla \log p_\theta(x_t \mid \emptyset) 
    \label{eq:cfg}
\end{equation}
where $\emptyset$ denotes the null prompt, representing the unconditional model.
\emph{Autoguidance}~\citep{karras2024guiding}
replaces $\nabla \log p_\theta(x_t \mid \emptyset)$ with $\nabla \log p_{\theta_{\text{weak}}}(x_t \mid y)$, using a weaker model under the same $y$.

\begingroup\color{black}
The \emph{group distribution} gives the probabilities of sensitive attribute groups $a\in\mathcal{A}$ in the generated distribution $q_0^w$. Let $g:\mathcal{X}\to\mathcal{A}$ map images to groups and $\mathbf{1}\{\cdot\}$ denote the indicator function. Then,
\begin{equation*}
    Q^w(a\mid y) := \int_{\mathcal{X}} \mathbf{1}\{g(x_0)=a\}\,q_0^w(x_0\mid y)\,\mathrm{d}x_0.
\end{equation*}
\par\endgroup

\newpage
\suppressfloats[t]
\subsection{Problem formulation}
\label{sec:prob_formulation}
We consider generation under a condition $y$ at a user-chosen guidance scale $w$, with images assigned to sensitive attribute groups $a\in\mathcal{A}$.
Let $Q^w(a\mid y)$ denote the resulting group distribution and $T(a\mid y)$ the target, \eg, uniform over $\mathcal{A}$ or the group distribution of the data.
\emph{Group fair generation} asks that $\mathrm{Bias}^w(a\mid y) := Q^w(a\mid y)-T(a\mid y)=0$, typically at a single scale $w$.

Our goal is to achieve group fairness, $\mathrm{Bias}^w(a\mid y)=0$, for any chosen scale $w\in[w_{\min},w_{\max}]$.
We thus introduce a \emph{bias decomposition} into model bias $\mathrm{Bias}_{\mathrm{M}}(a\mid y)$ and guidance bias $\mathrm{Bias}_{\mathrm{G}}^w(a\mid y)$.
\begin{tcolorbox}[keyeq, ams equation]
    \underbrace{\mathrm{Bias}^w(a\mid y)}_{\text{\footnotesize Total bias}}
    = \underbrace{Q^w(a\mid y)-Q^{w_{\mathrm{ref}}}(a\mid y)}_{\text{\footnotesize Guidance bias: }\scriptstyle\mathrm{Bias}_{\mathrm{G}}^w(a\mid y)}
    + \underbrace{Q^{w_{\mathrm{ref}}}(a\mid y)-T(a\mid y)}_{\text{\footnotesize Model bias: }\scriptstyle\mathrm{Bias}_{\mathrm{M}}(a\mid y)},
    \label{eq:bias_decomp}
\end{tcolorbox}

Our goal requires a fair model ($\mathrm{Bias}_{\mathrm{M}}{=}0$) matching the target distribution at $w_{\mathrm{ref}}$ and \emph{fair guidance} satisfying $\mathrm{Bias}_{\mathrm{G}}^{w}{=}0$ for any scale $w\in[w_{\min},w_{\max}]$ the user chooses.
\begin{definition}[Fair guidance]\label{def:fair_guidance}
    Guidance is \emph{fair} if $\mathrm{Bias}_{\mathrm{G}}^{w}(a\mid y)=0$ for all $w\in[w_{\min},w_{\max}]$ and groups $a\in\mathcal{A}$, \ie, the group distribution is invariant to $w$.
\end{definition}
We pursue fair guidance that can be used together with the debiased diffusion models developed in prior work.
Existing approaches focus on model bias, leaving guidance bias largely unaddressed.
We therefore target the guidance function $f$ in \cref{eq:guided_score}.

\subsection{The overlooked impact of guidance bias}
\label{sec:prob_impact}

\textcolor{black}{We reveal the overlooked consequences of $\mathrm{Bias}_{\mathrm{G}}^{w}$ in conventional diffusion sampling through two scenarios: its emergence in \emph{vanilla models} and persistence in \emph{debiased models} (\cref{fig:motivation}).}
First, a vanilla model follows its data distribution; applying guidance amplifies that distribution toward the majority group, more strongly as the scale grows.
Second, even debiased models lose fairness as $w$ deviates from $w_{\mathrm{ref}}$, because their group distributions drift from the target without scale-aware guidance.
Together, these observations indicate that \emph{guidance scale-invariant} group fairness is not reached by model debiasing, even when it is fitted across a range of guidance scales (\cref{tab:debiased_results}); it requires a direct intervention on the guidance mechanism itself.

\begin{figure}[t]
    \centering
    \includegraphics[width=\linewidth]{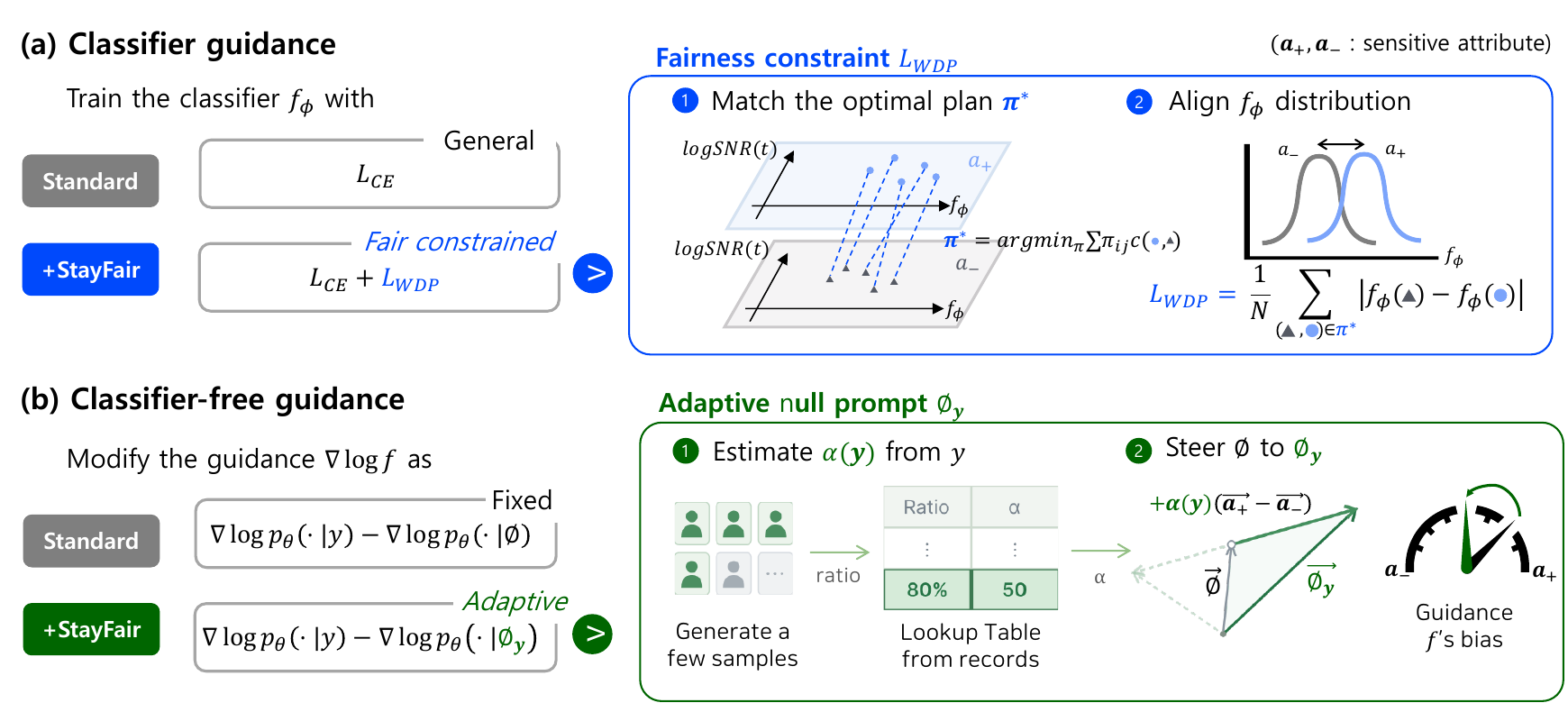}
    \caption{
        \textbf{Overview of \method.} 
        In (a) \textit{classifier guidance}, \method debiases the classifier $f_\phi$ by adding the fairness constraint $\mathcal{L}_{\mathrm{WDP}}$.
        It matches samples between groups using a joint cost on noise level and $f_\phi$ value, aligning the groupwise $f_\phi$ distributions.
        In (b) \textit{classifier-free guidance}, \method replaces the fixed null prompt $\emptyset$ with an adaptive $\emptyset_y$ by shifting its embedding along the attribute direction.
        $\alpha(y)$ is estimated from group ratios in a few generated images using searched $\alpha$ records.
    }
  \label{fig:method_pipeline}
\end{figure}

\section{\method: Fair Diffusion Across Guidance Scales} \label{sec:method}

To reduce guidance bias, we derive a sufficient condition for fairness across scales (\cref{sec:fair_condition}), use it to design \method for CG and CFG (\cref{sec:fair_guidance}), and show how \method combines with existing diffusion models (\cref{sec:combine_fair}).

\subsection{Sufficient condition for fair guidance}
\label{sec:fair_condition}
We study the group-conditional distribution of the guidance function $f:\mathcal{X}\times[0,T]\to\mathbb{R}_+$.
Since $p_t^w(\cdot\mid y)\propto p_t^0(\cdot\mid y)f(\cdot,t)^w$, group-dependent bias in $f$ transfers to $p_t^w$ and grows with $w$.
We therefore extend SDP~\citep{jiang2020wasserstein} to $f$.
\begin{definition}[SDP of the guidance function]\label{def:sdp}
    The guidance function $f$ satisfies Strong Demographic Parity if, for all $t\in[0,T]$ and $a,a'\in\mathcal{A}$,
    \begin{equation}
        \textcolor{black}{\mathrm{Law}}\bigl(f(X_t,t)\mid g(X_0){=}a,\,y\bigr)
        \;=\;
        \textcolor{black}{\mathrm{Law}}\bigl(f(X_t,t)\mid g(X_0){=}a',\,y\bigr).
    \end{equation}
\end{definition}
Under SDP, the target distribution $p^w_t(\cdot \mid y)$ preserves the group ratios across guidance scales.
\begin{theorem}[SDP preserves fairness in the target distribution]\label{thm:targetdist}
    Suppose that $f(\cdot,t)$ satisfies SDP (\cref{def:sdp}) with respect to $p_t(\cdot\mid y)$. Then, for every $w\ge 0$, $t\in[0,T]$, and $a,a'\in\mathcal{A}$, we have
    \begin{equation}
        p_t^w(a\mid y) / p_t^w(a'\mid y)
        \;=\;
        p_t^0(a\mid y) / p_t^0(a'\mid y).
        \label{eq:prop-ratio_main}
    \end{equation}
    Consequently, $p_t^w(\cdot \mid y)$ retains its group ratio for every guidance scale $w$.
\end{theorem}

\cref{thm:main} shows how this controls the actual process $q_t^w$ in a Gaussian setting (proof in \cref{app:proofs}).
\begin{theorem}[SDP transfers to the actual distribution; informal]\label{thm:main}
    In a Gaussian setting with a log-affine endpoint potential, if $f$ satisfies \cref{def:sdp}, then $\mathrm{Bias}_{\mathrm{G}}^{w}(a\mid y)=0$ for all $w, w_{\mathrm{ref}} \ge 0$ and $a\in\mathcal{A}$.
\end{theorem}

\subsection{Debiasing guidance in practice}
\label{sec:fair_guidance}

Building on \cref{thm:targetdist}, \method approximately balances $f$ in both guidance regimes. \cref{fig:method_pipeline} summarizes the CG and CFG instantiations, detailed below, and we then relate it to autoguidance.

\textbf{Classifier guidance.} \label{sec:cg_ours}
In CG, the guidance function is an explicit classifier $f_\phi(\cdot,t)$, so we regularize its outputs toward SDP during training.
We adopt Wasserstein Demographic Parity (WDP)~\citep{jiang2020wasserstein}, a tractable relaxation of SDP that measures the 1-Wasserstein distance between the $f_\phi(\cdot,t)$ distributions of two groups, \ie, $\Delta\mathrm{WDP}_t(a,a')=0$ if and only if SDP holds at timestep $t$.
\begin{equation}
    \Delta\mathrm{WDP}_t(a,a') \;=\; \mathcal{W}_1\bigl(\textcolor{black}{\mathrm{Law}}(f_\phi(X_t,t)\mid g(X_0){=}a, y),\, \textcolor{black}{\mathrm{Law}}(f_\phi(X_t,t)\mid g(X_0){=}a', y)\bigr).
    \label{eq:wdp_def}
\end{equation}

To enforce \(\Delta\mathrm{WDP}_t(a,a')=0\), we minimize an empirical WDP loss between cross-group classifier outputs within a local timestep window to account for varying noise levels.
For each group pair, we match outputs \(u_i\) and \(v_j\) using an optimal transport plan \(\pi^*= \arg\min_{\pi\in\Pi}\sum_{i,j}\pi_{ij}\,c\bigl((u_i,t_i),(v_j,t'_j)\bigr),\) where \(\Pi\) is the set of balanced transport plans.
The transport cost is noise-aware $c\bigl((u,t),(v,t')\bigr)=|u-v|+ \gamma|\mathrm{logSNR}(t)-\mathrm{logSNR}(t')|$ with parameter $\gamma$.
The empirical loss $\mathcal{L}_{\mathrm{WDP}}(a,a') = \tfrac{1}{n}\sum_{i,j} \pi_{ij}^*|u_i - v_j|$ combines with the cross-entropy loss $\mathcal{L}_{\mathrm{CE}}$ as
\begin{equation}
    \mathcal{L} = \mathcal{L}_{\mathrm{CE}} + \lambda \sum_{a,a'\in\mathcal{A}} \mathcal{L}_{\mathrm{WDP}}(a,a').
    \label{eq:cg_obj}
\end{equation}

\textbf{Classifier\nobreakdash-free guidance.} \label{sec:cfg_ours}
In CFG, \(f(\cdot,t)\) is only implicitly defined (\cref{eq:cfg}), so direct regularization is infeasible.
\textcolor{black}{We instead aim to reduce group bias in \(f\) without enforcing SDP.}
We replace the fixed null prompt \(\emptyset\) with an adaptive \(\emptyset_y\) (\cref{eq:stayfair_cfg}), whose embedding is shifted along attribute directions to correct the bias of \(p_\theta(x_t\mid\emptyset_y)\) for prompt \(y\) (\cref{eq:ours}).
\begin{equation}
    \nabla\log f(x_t,t) := \nabla\log p_\theta(x_t|y) - \nabla\log p_\theta(x_t|\emptyset_y) \label{eq:stayfair_cfg}
\end{equation}
\begin{equation}
    E(\emptyset_y) := E(\emptyset) + \sum_i \alpha_i(y)\frac{E(a_i^+)-E(a_i^-)}{\|E(a_i^+)-E(a_i^-)\|} \label{eq:ours}
\end{equation}

where $E$ is the text encoder and $a_i^\pm$ are opposing attribute texts (\eg, male/female); we omit $i$ for a single attribute. Attribute directions follow embedding subtraction~\citep{couairon2022embedding}.
We choose the null prompt $\emptyset$ (\eg, \textcolor{black}{an empty string, a space, or a padding token}) for each tokenizer so that special tokens occupy the same positions as in $a^\pm$.
We use \textcolor{black}{embedding-level} control for \textcolor{black}{architecture-agnostic} application.

\newpage
\begin{wrapfigure}[15]{r}{0.5\linewidth}
    \vspace{-1pt}
    \captionsetup{skip=8pt,font=normalsize}
    \centering
    \begin{overpic}[width=\linewidth,percent]{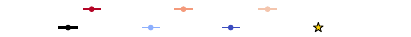}
        \put(25.63,6.07){\tiny $\alpha{=}{-}10$}
        \put(48.01,6.07){\tiny $\alpha{=}{-}5$}
        \put(68.51,6.07){\tiny $\alpha{=}{-}2.5$}
        \put(19.88,1.52){\tiny CFG}
        \put(40.03,1.52){\tiny $\alpha{=}5$}
        \put(59.52,1.52){\tiny $\alpha{=}10$}
        \put(79.40,1.52){\tiny \textcolor{black}{\method (Search)}}
    \end{overpic}\\[9pt]
    \begin{overpic}[width=\linewidth,percent]{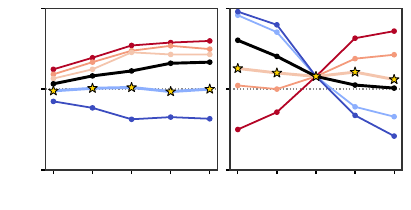}
        \put(13.02,4.30){\makebox[0pt][c]{\tiny 2.5}} \put(22.54,4.30){\makebox[0pt][c]{\tiny 5.0}} \put(32.07,4.30){\makebox[0pt][c]{\tiny 7.5}} \put(41.60,4.30){\makebox[0pt][c]{\tiny 10.0}} \put(51.12,4.30){\makebox[0pt][c]{\tiny 12.5}}
        \put(57.97,4.30){\makebox[0pt][c]{\tiny 2.5}} \put(67.49,4.30){\makebox[0pt][c]{\tiny 5.0}} \put(77.02,4.30){\makebox[0pt][c]{\tiny 7.5}} \put(86.55,4.30){\makebox[0pt][c]{\tiny 10.0}} \put(96.07,4.30){\makebox[0pt][c]{\tiny 12.5}}
        \put(7.80,8.60){\makebox[0pt][r]{\tiny 0}} \put(7.80,28.29){\makebox[0pt][r]{\tiny 50}} \put(7.80,47.99){\makebox[0pt][r]{\tiny 100}}
        \put(32.07,0.30){\makebox[0pt][c]{\scriptsize Guidance scale $w$}}
        \put(77.02,0.30){\makebox[0pt][c]{\scriptsize Guidance scale $w$}}
        \put(0.30,29.29){\rotatebox{90}{\makebox[0pt][c]{\scriptsize Female ratio (\%)}}}
        \put(32.07,51.3){\makebox[0pt][c]{\scriptsize Journalist (SD1.5)}}
        \put(77.02,51.3){\makebox[0pt][c]{\scriptsize Fashion Designer (FT)}}
    \end{overpic}
    \caption{\textbf{Slope varies monotonically with $\alpha$ in both models, so we can find the flattest curve, which minimizes guidance bias.}
    }
    \label{fig:alpha_ablation}

\end{wrapfigure}

Adjusting $\alpha$ changes the guidance bias of $\nabla\log f$ in the opposite direction to the bias of $p_\theta(x_t\mid\emptyset_y)$, since the null-branch score $\nabla\log p_\theta(x_t\mid\emptyset_y)$ is subtracted in \cref{eq:stayfair_cfg}.
Motivated by the observation that bias amplification is smaller when $p_\theta(a\mid y)$ and $p_\theta(a\mid\emptyset_y)$ have similar bias levels, we adjust the latter while retaining its role as the null branch in CFG.

\textcolor{black}{We now introduce two variants, \method (Search) and \method (Estimate), to select $\alpha(y)$ that minimizes guidance bias across scales.}
\textcolor{black}{First, \method (Search) uses binary search to find $\alpha(y)$ that approximately equalizes group ratios at the lowest and highest scales.}
\textcolor{black}{Since the curve's slope changes monotonically with $\alpha$ (\cref{fig:alpha_ablation}), the search selects the $\alpha(y)$ that yields the flattest curve.}
Second, \method (Estimate) assigns $\hat{\alpha}(y)$ from a precomputed lookup table indexed by the group ratios observed at the lowest and highest scales, \textcolor{black}{leveraging the observation that} similar initial proportions change similarly with scale (\cref{fig:ablation_alpha}).
\textcolor{black}{Procedures are in \cref{app:alpha-procedure}; \method denotes \method (Search) unless stated otherwise.}

\textbf{Relation to autoguidance.}
\method lies between CFG and AG~\citep{karras2024guiding} in how it trades condition alignment against bias amplification.
AG is the special case $\emptyset_y = y$ with a weaker model: it lowers guidance bias, but also gains less condition alignment.
CFG is the case $\emptyset_y = \emptyset$: it gains alignment strongly, and carries a high guidance bias.
\method instead makes $\emptyset_y$ independent of the condition's sensitive attribute and keeps CFG's alignment gain while reducing guidance bias toward AG's level.
A comparison with PAG is in \cref{app:additional-results} (\cref{tab:app_sd15_pag}).

\subsection{Combining with fair diffusion models}
\label{sec:combine_fair}
\method can be applied to a diffusion model and carries its group distribution across scales.
Given the model $\epsilon_\theta^{\text{Fair}}$ at $w_{\mathrm{ref}}$ and a scale $w$, we predict $\epsilon_\theta^{\text{StayFair}}$ by adding fair guidance (\cref{eq:stayfair_cfg}) for $w-w_{\mathrm{ref}}$:
\begin{equation}
    \epsilon_\theta^{\text{StayFair}}(x_t,t,y)
    =
    \underbrace{\epsilon_\theta^{\text{Fair}}(x_t,t,y)}_{\text{fair model at } w=w_{\mathrm{ref}}}
    \;+\;
    (w{-}w_{\mathrm{ref}})
    \underbrace{
    \bigl(\epsilon_\theta(x_t,t,y) - \epsilon_\theta(x_t,t,\emptyset_y)\bigr)
    }_{\text{fair guidance}} .
    \label{eq:ft_ref}
\end{equation}
The model $\epsilon_\theta^{\text{Fair}}$ has $\mathrm{Bias}_{\mathrm{M}}{=}0$ as its objective, since a vanilla model follows the data distribution at $w_{\mathrm{ref}}{=}1$ (without guidance) and a debiased model follows the target distribution at its debiasing scale $w_{\mathrm{ref}}$.
Fair guidance has $\mathrm{Bias}_{\mathrm{G}}^{w}{=}0$ as its objective (\cref{def:fair_guidance}).
Meeting both objectives gives $\mathrm{Bias}^w(a\mid y)=0$ for all $w$ in the bias decomposition of \cref{eq:bias_decomp}.

\textcolor{black}{Our main analysis evaluates fair guidance in standard sampling using both vanilla diffusion models and models debiased during training, without relying on attribute conditioning.}
\textcolor{black}{\Cref{app:inference-based-methods} analyzes attribute-based methods Fair Diffusion~\citep{friedrich2023fair} and Weak Guidance~\citep{kim2025rethinking} across guidance scales, showing how adapting $\emptyset_y$ reduces their bias.}

\section{Experiments} \label{sec:experiment}

\textcolor{black}{We evaluate \method in class-conditional generation with CG (\cref{sec:exp_cg}) and in text-to-image generation with CFG (\cref{sec:exp_t2i}), covering vanilla (\cref{sec:exp_t2i_vanilla}) and debiased generation (\cref{sec:exp_t2i_debiased}).}
\textcolor{black}{We use our bias decomposition to analyze asymmetric gender bias amplification (\cref{sec:discussion}).}
\textcolor{black}{We also examine multi-group attributes and unseen prompts (\cref{sec:exp_analysis}).}
Our main comparisons assess whether \method stabilizes group ratios across scales while preserving condition alignment and image quality, and how reducing guidance bias affects average and worst total bias.

\begin{figure}[!t]
\begin{minipage}{\linewidth}
\captionsetup{type=table}
    \centering
    \caption{\textbf{\method reduces guidance bias while improving image quality in class-conditional generation with CG.}
    On CelebA, we report guidance bias (Range) and image quality using fair-intra versions of FID and FD$_{\text{DINOv2}}$.
    Both compare generated samples with a reference set matched to the target group ratios.
    Best in bold; our target metric shaded; $\pm$: std over six seeds.}
    \label{tab:bias_fid_by_model}
    \footnotesize
    \renewcommand{\sdmark}[1]{\textcolor{black!45}{\,$\pm$#1}}
    \renewcommand{\arraystretch}{0.95}
    \begin{tabular}{lccc ccc}
        \toprule
        & \multicolumn{3}{c}{Smiling} & \multicolumn{3}{c}{Blond Hair} \\
        \cmidrule(lr){2-4} \cmidrule(lr){5-7}
        Classifier & \cellcolor{targetshade}Range (\%) $\downarrow$ & FID $\downarrow$ & FD$_\text{DINOv2}$ $\downarrow$
               & \cellcolor{targetshade}Range (\%) $\downarrow$ & FID $\downarrow$ & FD$_\text{DINOv2}$ $\downarrow$ \\
        \midrule
        CG                 & \cellcolor{targetshade}\pmv{20.6}{2.2}          & 17.99          & 87.51
                           & \cellcolor{targetshade}\pmv{39.2}{3.5}          & 22.70          & 151.04 \\
        RW                 & \cellcolor{targetshade}\pmv{11.0}{3.4}          & 16.96          & 83.39
                           & \cellcolor{targetshade}\pmv{32.9}{3.1}          & 20.05          & \textbf{136.55} \\
        GDRO               & \cellcolor{targetshade}\pmv{11.9}{3.5}          & 17.50          & 85.28
                           & \cellcolor{targetshade}\pmv{29.2}{2.8}          & 22.29          & 152.70 \\
        \textbf{\method (ours)}     & \cellcolor{targetshade}\pmvb{\phantom{0}3.5}{1.2} & \textbf{16.66} & \textbf{82.56}
                           & \cellcolor{targetshade}\pmvb{12.4}{2.6} & \textbf{19.68} & 139.65 \\
        \bottomrule
    \end{tabular}
\end{minipage}
\par\vspace{6pt}
\input{figures/fig_cg_combined_v3}
\vspace{-12pt}
\end{figure}

\subsection{Class-conditional generation with CG}
\label{sec:exp_cg}
\textbf{Setup.}
We apply CG to ADM-G~\citep{dhariwal2021diffusion} on CelebA~\citep{liu2015deep} ($64{\times}64$), using the unconditional model from~\citet{ning2023input}.
We consider gender as the sensitive attribute and two conditions with different imbalance levels: \emph{Smiling} (weak) and \emph{Blond Hair} (strong).
Gender labels are classified by using a pretrained classifier~\citep{karkkainenfairface}.
We sweep $w \in \{0, 1, 3, 5, 7, 9, 11, 13\}$, generating $256$ images per condition and scale with six seeds for the gender ratio and $10{,}000$ for image quality.
Further implementation details are in \cref{app:experimental-details}.

\textbf{Baselines.}
We compare \method against a vanilla classifier and two debiased ones in the same CG pipeline: RW~\citep{idrissi2022simple} (data-level) and GDRO~\citep{sagawa2020distributionally} (optimization-level).

\textbf{Metrics.}
Our target metric is the guidance bias, measured by the bias range ($\max-\min$ of the female ratio across $w$).
Image quality is measured by \emph{fair-intra} FID~\citep{heusel2017gans, nam2023breaking} and FD$_\text{DINOv2}$~\citep{stein2023exposing}, using a reference set adjusted to the target distribution.

\textbf{Experimental results.}
\cref{tab:bias_fid_by_model} shows that \method substantially reduces guidance bias in CG: the bias range drops from 20.6 to 3.5 on \emph{Smiling} and from 39.2 to 12.4 on \emph{Blond Hair}, while achieving the best fair-intra FID in both conditions and the best FD$_\text{DINOv2}$ on \emph{Smiling}.
While RW and GDRO also reduce the bias range, their gains remain \textcolor{black}{suboptimal} because they do not match guidance-score distributions across groups, whereas \method does.

\textcolor{black}{\textbf{Qualitative comparison.}
\cref{fig:cg_combined} shows that \method maintains the unguided female ratio across $w$, whereas CG drifts toward the majority group.
Compared with CG, \method's debiased classifier mitigates the collapse of minority-group modes, as suggested by the preserved visual context and minority gender of the unguided sample ($w=0$).}

\subsection{Text-to-image generation with CFG}
\label{sec:exp_t2i}

\textbf{Setup.}
We measure gender--occupation bias at five guidance scales: for each of 33 occupations~\citep{seshadri2024bias}, we generate $100$ images with three seeds from the prompt \texttt{``A photo of the face of a/an \{occupation\}''}.
The scales are evenly spaced: $2.5$--$12.5$ on SD1.5 (default $7.5$), $1.5$--$7.5$ on SD3 (default $7.0$), and $1.0$--$5.0$ on FLUX.2 (default $4.0$).
Gender is labeled by zero-shot CLIP~\citep{radford2021learning} after YOLO~\citep{ultralyticsYOLO11} person detection (\cref{app:experimental-details}).

\textbf{Metrics.}
\textcolor{black}{Guidance bias is the bias range ($\max-\min$ female ratio across scales).
We report average and worst total bias (the female ratio's absolute deviation from $50\%$) across scales.
CLIP score~\citep{hessel2021clipscore} measures condition alignment; aesthetic score~\citep{rsinemaAestheticScorer} and PickScore~\citep{kirstain2023pick} measure image quality.}

\subsubsection{Vanilla generation}
\label{sec:exp_t2i_vanilla}
\label{sec:exp_cfg_cond}
\textbf{Baselines.}
We compare \method against CFG using vanilla SD1.5~\citep{rombach2022high}, SD3-medium~\citep{esser2024scaling}, and FLUX.2 [klein] 9B Base~\citep{flux-2-2025}.

\begin{table}[!t]
    \centering
    \captionsetup{skip=3pt}
    \caption{\textbf{\method reduces guidance bias while maintaining image quality comparable to CFG.}
    For gender--occupation prompts, we report total bias (Avg., Worst), guidance bias (Range), and image quality over five guidance scales, with results averaged across occupations.
    \method reduces average and worst bias in vanilla generation (a), and lowers worst bias while maintaining average debiasing performance in debiased generation (b).
    \method achieves smaller bias ranges than both Multi baselines that extend debiasing across scales, highlighting the benefit of modifying the guidance mechanism. Best in bold; our target metric shaded; $\pm$: std over three seeds.}
    \label{tab:combined_results}
    \renewcommand{\thesubtable}{(\alph{subtable})}
    \captionsetup[subtable]{labelformat=simple,font=small,skip=4pt}
    \begin{subtable}{\linewidth}
    \centering
    \phantomsubcaption
    \label{tab:table1_2_results}
    \footnotesize
    \renewcommand{\arraystretch}{0.95}
    \setlength{\tabcolsep}{2pt}
\begin{tabular*}{\textwidth}{@{}p{\dimexpr\linewidth-368pt+62pt\relax} p{64pt} *{2}{>{\centering\arraybackslash}p{31pt}} >{\columncolor{targetshade}[\tabcolsep][0.5pt]\centering\arraybackslash}p{31pt} >{\columncolor{targetshade}[0.5pt][\tabcolsep]\centering\arraybackslash}p{40pt} >{\centering\arraybackslash}p{37pt} >{\centering\arraybackslash}p{40pt}@{}}
        \toprule
        \multicolumn{2}{@{}l}{\textbf{(a) Vanilla generation}} & \multicolumn{3}{c}{Bias (\%) $\downarrow$} & \multicolumn{3}{c}{Quality $\uparrow$} \\
        \cmidrule(lr){3-5} \cmidrule(lr){6-8}
        Model & Guidance & Avg. & Worst & Range & CLIP & Aesth. & Pick \\
        \midrule
        \multirow{2}{*}{SD1.5}
                & CFG            & \pmv{28.0}{0.3} & \pmv{32.5}{0.5} & \pmv{10.7}{0.7} & \pmv{27.61}{0.02} & \pmvb{6.70}{0.01} & \pmvb{19.99}{0.02} \\
                & \textbf{\method (ours)} & \pmvb{27.6}{0.6} & \pmvb{30.8}{0.5} & \pmvb{\phantom{0}7.2}{0.3} & \pmvb{27.62}{0.02} & \pmvb{6.70}{0.01} & \pmv{19.96}{0.01} \\
        \arrayrulecolor{black!35}\midrule[0.4pt]\arrayrulecolor{black}
        \multirow{2}{*}{SD3}
                & CFG            & \pmv{39.7}{0.5} & \pmv{43.0}{0.4} & \pmv{\phantom{0}9.1}{0.3} & \pmv{27.36}{0.01} & \pmvb{6.62}{0.00} & \pmvb{20.29}{0.00} \\
                & \textbf{\method (ours)} & \pmvb{33.9}{0.5} & \pmvb{36.6}{0.5} & \pmvb{\phantom{0}6.3}{0.2} & \pmvb{27.38}{0.02} & \pmv{6.61}{0.00} & \pmv{20.23}{0.00} \\
        \arrayrulecolor{black!35}\midrule[0.4pt]\arrayrulecolor{black}
        \multirow{2}{*}{FLUX.2}
                & CFG            & \pmv{39.4}{0.2}  & \pmv{44.3}{0.2}  & \pmv{15.9}{0.6} & \pmvb{28.18}{0.02} & \pmvb{6.93}{0.00} & \pmv{21.15}{0.01} \\
                & \textbf{\method (ours)} & \pmvb{32.7}{0.3} & \pmvb{36.0}{0.1} & \pmvb{\phantom{0}8.4}{0.9} & \pmv{28.17}{0.02} & \pmvb{6.93}{0.00} & \pmvb{21.18}{0.01} \\
        \bottomrule
\end{tabular*}
    \end{subtable}
    \par\vspace{4pt}
    \begin{subtable}{\linewidth}
    \centering
    \phantomsubcaption
    \label{tab:debiased_results}
    \footnotesize
    \renewcommand{\arraystretch}{0.95}
    \setlength{\tabcolsep}{2pt}
    \def\debiasedpanelheading{\multicolumn{3}{@{}l}{\textbf{(b) Debiased generation}}}
    \providecommand{\debiasedpanelheading}{\multicolumn{3}{@{}l}{}}
\begin{tabular*}{\textwidth}{@{}p{\dimexpr\linewidth-372pt\relax} p{62pt} p{64pt} *{2}{>{\centering\arraybackslash}p{31pt}} >{\columncolor{targetshade}[\tabcolsep][0.5pt]\centering\arraybackslash}p{31pt} >{\columncolor{targetshade}[0.5pt][\tabcolsep]\centering\arraybackslash}p{40pt} >{\centering\arraybackslash}p{37pt} >{\centering\arraybackslash}p{40pt}@{}}
    \toprule
    \debiasedpanelheading & \multicolumn{3}{c}{Bias (\%) $\downarrow$} & \multicolumn{3}{c}{Quality $\uparrow$} \\
    \cmidrule(lr){4-6} \cmidrule(lr){7-9}
    Method & Scales & Guidance & Avg. & Worst & Range & CLIP & Aesth. & Pick \\
    \midrule
    \multirow{4}{*}{UCE}
            & Single & CFG            & \pmv{18.0}{0.7} & \pmv{23.4}{1.4} & \pmv{13.2}{0.7} & \pmv{27.28}{0.01} & \pmv{6.70}{0.01} & \pmvb{19.96}{0.01} \\
            & Multi (random) & CFG            & \pmv{19.1}{0.8} & \pmv{24.2}{1.2} & \pmv{13.7}{1.1} & \pmvb{27.38}{0.02} & \pmv{6.71}{0.01} & \pmv{19.89}{0.01} \\
            & Multi (per-scale) & CFG            & \pmvb{17.7}{0.6} & \pmv{29.4}{0.7} & \pmv{30.6}{0.4} & \pmv{27.32}{0.02} & \pmvb{6.71}{0.00} & \pmv{19.92}{0.01} \\
            & Single & \textbf{\method (ours)} & \pmv{18.5}{0.6} & \pmvb{22.2}{1.1} & \pmvb{\phantom{0}7.6}{0.7} & \pmv{27.28}{0.01} & \pmv{6.70}{0.01} & \pmvb{19.96}{0.01} \\
    \arrayrulecolor{black!35}\midrule[0.4pt]\arrayrulecolor{black}
    \multirow{4}{*}{FT}
            & Single & CFG            & \pmv{14.1}{0.3} & \pmv{21.3}{0.6} & \pmv{19.3}{0.8} & \pmv{27.54}{0.03} & \pmv{6.77}{0.01} & \pmv{20.03}{0.02} \\
            & Multi (random) & CFG            & \pmv{16.5}{0.9} & \pmv{21.3}{0.9} & \pmv{12.0}{0.6} & \pmvb{27.63}{0.02} & \pmvb{6.79}{0.01} & \pmvb{20.07}{0.01} \\
            & Multi (per-scale) & CFG            & \pmvb{13.6}{0.4} & \pmv{22.0}{0.6} & \pmv{17.9}{0.7} & \pmv{27.62}{0.03} & \pmv{6.77}{0.01} & \pmv{20.03}{0.02} \\
            & Single & \textbf{\method (ours)} & \pmv{13.8}{0.2} & \pmvb{18.2}{0.4} & \pmvb{\phantom{0}9.4}{0.4} & \pmv{27.53}{0.02} & \pmv{6.77}{0.01} & \pmv{20.03}{0.02} \\
    \bottomrule
\end{tabular*}

    \end{subtable}
    \vspace{-6pt}
\end{table}

\newpage
\begin{wrapfigure}[12]{r}{0.439\linewidth}
\vspace{0pt}
\centering
        \makebox[\linewidth][l]{\hspace*{0.53864\linewidth}\makebox[0pt][c]{%
        {\tiny\shortstack{ \raisebox{-1pt}{\includegraphics[width=9pt,trim={462bp 143.73bp 78bp 27.83bp},clip]{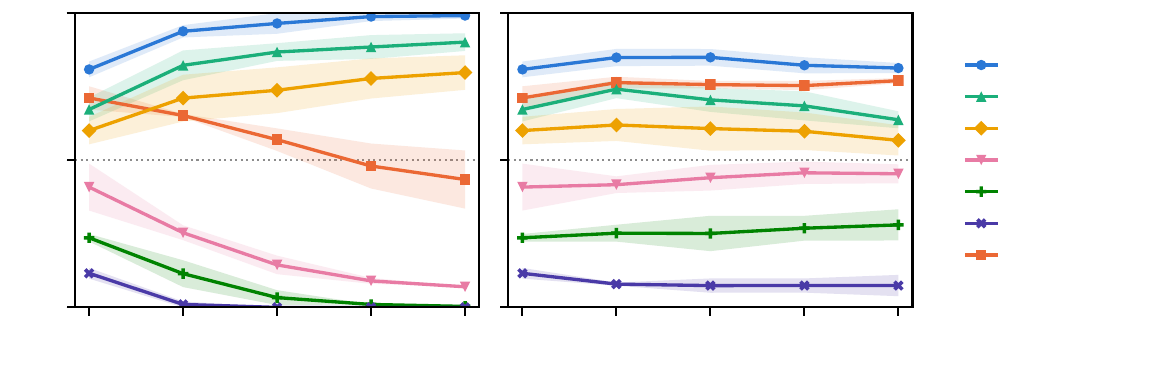}}\,teacher\hspace{6pt}\raisebox{-1pt}{\includegraphics[width=9pt,trim={462bp 128.56bp 78bp 43.00bp},clip]{flux_fr_curves_bare.pdf}}\,librarian\hspace{6pt}\raisebox{-1pt}{\includegraphics[width=9pt,trim={462bp 113.37bp 78bp 58.19bp},clip]{flux_fr_curves_bare.pdf}}\,assistant\\[1pt]
        \raisebox{-1pt}{\includegraphics[width=9pt,trim={462bp 98.15bp 78bp 73.41bp},clip]{flux_fr_curves_bare.pdf}}\,doctor\hspace{6pt}\raisebox{-1pt}{\includegraphics[width=9pt,trim={462bp 82.93bp 78bp 88.63bp},clip]{flux_fr_curves_bare.pdf}}\,scientist\hspace{6pt}\raisebox{-1pt}{\includegraphics[width=9pt,trim={462bp 67.75bp 78bp 103.81bp},clip]{flux_fr_curves_bare.pdf}}\,poet\hspace{6pt}\raisebox{-1pt}{\includegraphics[width=9pt,trim={462bp 52.55bp 78bp 119.01bp},clip]{flux_fr_curves_bare.pdf}}\,fashion designer}}}}\\[11pt]  %
        \hspace*{0.04\linewidth}\begin{overpic}[width=0.96\linewidth,height=0.423274\linewidth,trim={0 0 118bp 0},clip,percent]{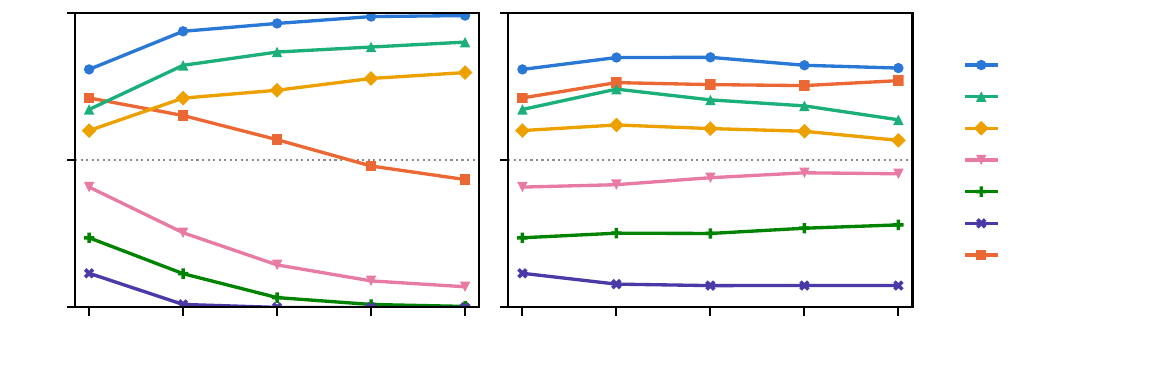}
            \put(30.233,7.750){\tikz[overlay,x=\unitlength,y=\unitlength]{\fill[black,opacity=0.035] (0,0) rectangle (20.507,34.850);}}
            \put(77.502,7.750){\tikz[overlay,x=\unitlength,y=\unitlength]{\fill[black,opacity=0.035] (0,0) rectangle (20.506,34.850);}}
            \put(9.714,3.375){\makebox[0pt][c]{\tiny 1.0}} \put(19.973,3.375){\makebox[0pt][c]{\tiny 2.0}} \put(30.233,3.375){\makebox[0pt][c]{\tiny 3.0}} \put(40.480,3.375){\makebox[0pt][c]{\tiny 4.0}} \put(50.740,3.375){\makebox[0pt][c]{\tiny 5.0}}
            \put(56.996,3.375){\makebox[0pt][c]{\tiny 1.0}} \put(67.242,3.375){\makebox[0pt][c]{\tiny 2.0}} \put(77.502,3.375){\makebox[0pt][c]{\tiny 3.0}} \put(87.749,3.375){\makebox[0pt][c]{\tiny 4.0}} \put(98.008,3.375){\makebox[0pt][c]{\tiny 5.0}}
            \put(6.795,6.862){\makebox[0pt][r]{\tiny 0}}
            \put(6.795,24.333){\makebox[0pt][r]{\tiny 50}}
            \put(6.795,41.805){\makebox[0pt][r]{\tiny 100}}
            \put(30.233,-0.690){\makebox[0pt][c]{\tiny Guidance scale $w$}}
            \put(77.502,-0.690){\makebox[0pt][c]{\tiny Guidance scale $w$}}
            \put(-2.5,25.118){\rotatebox{90}{\makebox[0pt][c]{\tiny Female ratio (\%)}}}
            \put(30.233,45.173){\makebox[0pt][c]{\tiny FLUX.2 + CFG}}
            \put(77.502,45.173){\makebox[0pt][c]{\tiny FLUX.2 + \method}}
        \end{overpic}
\caption{\textbf{Unlike CFG, \method maintains female ratios across guidance scales.} Further results are in \cref{app:additional-results}.}
\label{fig:scatter_sd3_ours_opt}
\vspace{-6pt}
\end{wrapfigure}

\textbf{Experimental results.}
\method reduces bias ranges, lowering average and worst bias across models while maintaining image quality comparable to CFG.
On FLUX.2, \method reduces the bias range from 15.9 to 8.4 and average and worst bias by 6.7 and 8.3 percentage points (pp), respectively (\cref{tab:table1_2_results}).
SD3 and FLUX.2 show larger bias reductions than SD1.5 because their evaluation includes lower scales, where guidance bias primarily arises.
\cref{fig:scatter_sd3_ours_opt} shows that \method keeps female ratios on FLUX.2 stable across scales.
It thus maintains the model distribution, avoiding CFG's extreme bias in some occupations even within the recommended scale range ($w=3.0$--$5.0$).

\textcolor{black}{\textbf{Qualitative comparison.}
\cref{fig:qual_t2i} shows that \method preserves the minority gender across guidance scales while retaining visual context similar to CFG; more examples are in \cref{app:additional-qualitative}.}

\subsubsection{Debiased generation}
\label{sec:exp_t2i_debiased}
\label{sec:exp_cfg_fair}
\textbf{Baselines.}
We use SD1.5 debiased with FT~\citep{shenfinetuning} or UCE~\citep{gandikota2024unified} at $w_{\mathrm{ref}}{=}7.5$ (Single) and compare \method against CFG.
We also test whether extending existing debiasing methods across guidance scales can resolve guidance bias.
Multi (random) samples $w$ during debiasing, while Multi (per-scale) trains a separate model for each $w$ (\cref{app:debias-extension-failures}).

\textbf{Experimental results.}
First, \method reduces bias ranges with UCE and FT, improving worst bias while maintaining their average debiasing performance.
\cref{tab:debiased_results} shows bias ranges falling from 13.2 to 7.6 with UCE and from 19.3 to 9.4 with FT, with worst bias reduced by 1.2 and 3.1 pp, respectively, and image quality comparable to CFG.
Second, \method achieves lower bias ranges than both Multi variants: modifying guidance better preserves group distributions across scales than extending debiasing.
\cref{app:debias-extension-failures} analyzes failure cases of both Multi variants.

\begin{figure}[t]
  \centering
  \setlength{\tabcolsep}{1pt}
  \newcommand{\qright}{tea}                      %
  \newcommand{\qrightlabel}{\ifdefstring{\qright}{lib}{Librarian (minority: \textcolor{gM}{male})}{Teacher (minority: \textcolor{gM}{male})}}
  \definecolor{gF}{RGB}{214,69,65}\definecolor{gM}{RGB}{46,117,182}   %
  \setlength{\fboxsep}{0pt}\setlength{\fboxrule}{0.8pt}
  \newcommand{\qcrop}[2]{\fcolorbox{#1}{white}{\includegraphics[width=0.083\linewidth]{qual_t2i_flux/#2}}}
  \newcommand{\rowlabelshort}[1]{\raisebox{0.036\linewidth}{\rotatebox[origin=c]{90}{\footnotesize #1}}}
  \newcommand{\rowlabellong}[1]{\raisebox{0.037\linewidth}{\rotatebox[origin=c]{90}{\footnotesize #1}}}

  \begin{tabular}{@{} c ccccc @{\hspace{6pt}} ccccc @{}}
    & \multicolumn{5}{c}{\raisebox{-2pt}[0pt][0pt]{\small Scientist (minority: \textcolor{gF}{female})}} & \multicolumn{5}{c}{\raisebox{-2pt}[0pt][0pt]{\small \qrightlabel}} \\[1pt]
    & {\footnotesize $w{=}1.0$} & {\footnotesize $w{=}2.0$} & {\footnotesize $w{=}3.0$} & {\footnotesize $w{=}4.0$} & {\footnotesize $w{=}5.0$}
    & {\footnotesize $w{=}1.0$} & {\footnotesize $w{=}2.0$} & {\footnotesize $w{=}3.0$} & {\footnotesize $w{=}4.0$} & {\footnotesize $w{=}5.0$} \\[1pt]
    \rowlabelshort{CFG}
      & \qcrop{gF}{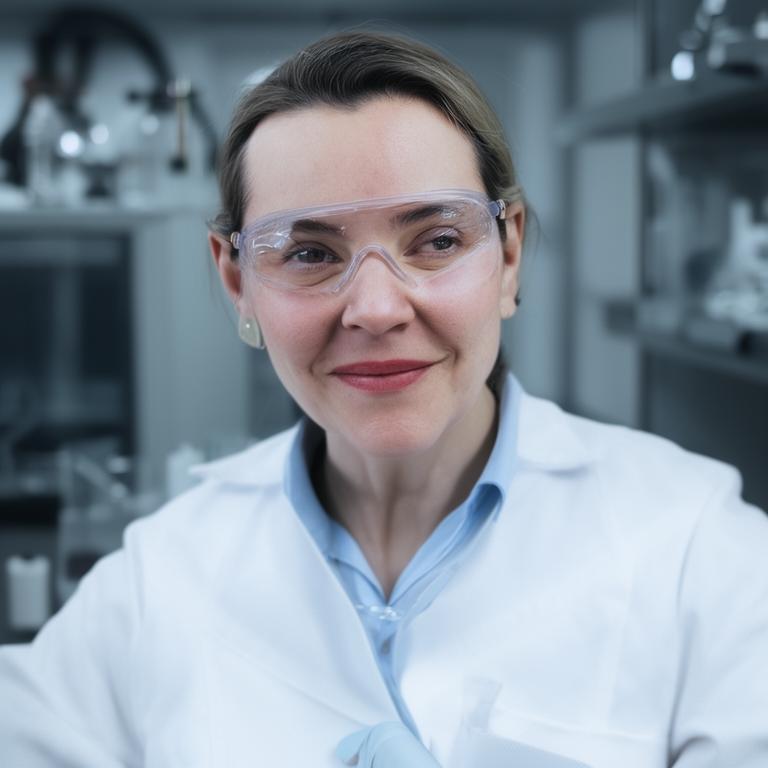}
      & \qcrop{gM}{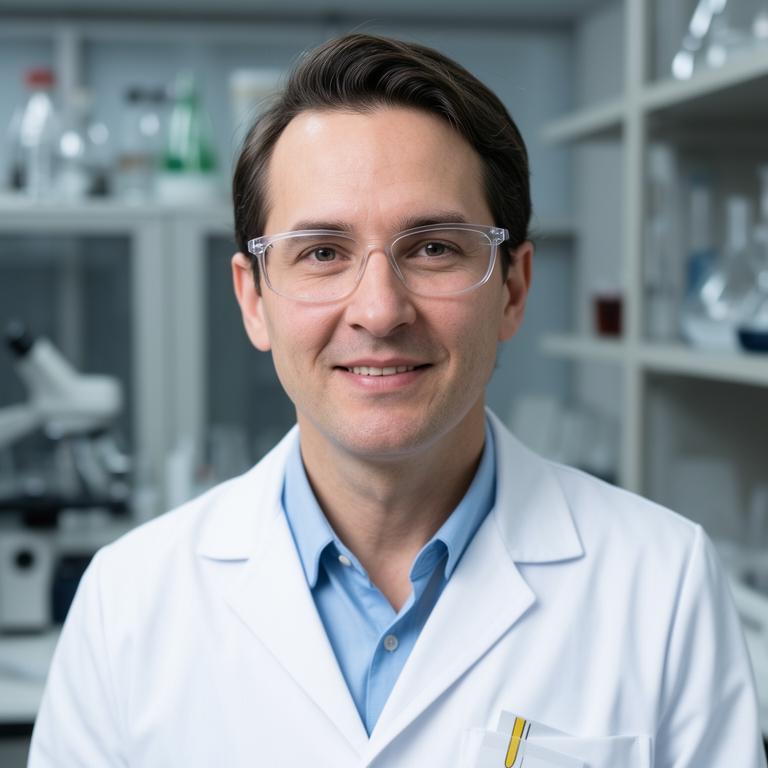}
      & \qcrop{gM}{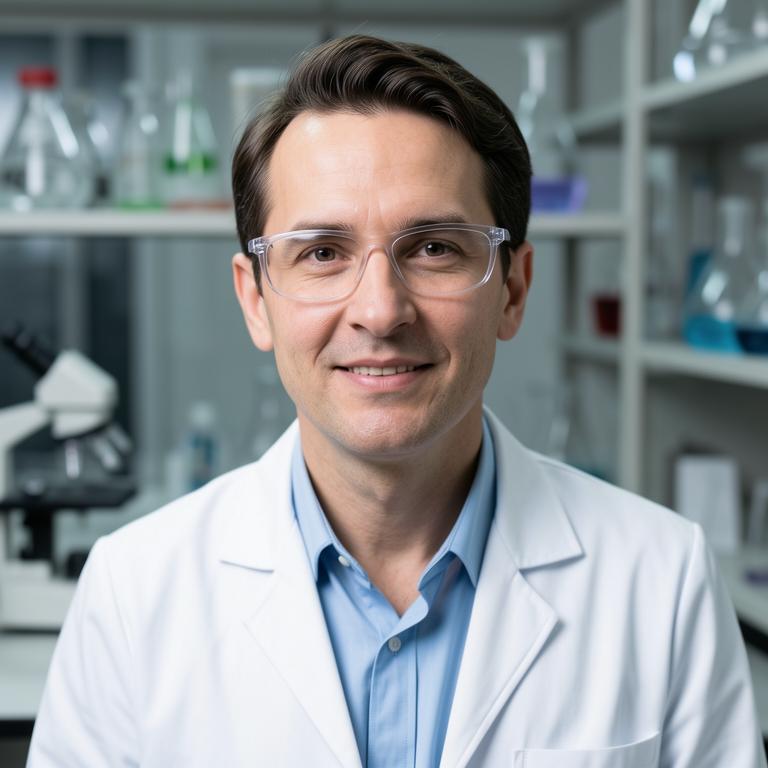}
      & \qcrop{gM}{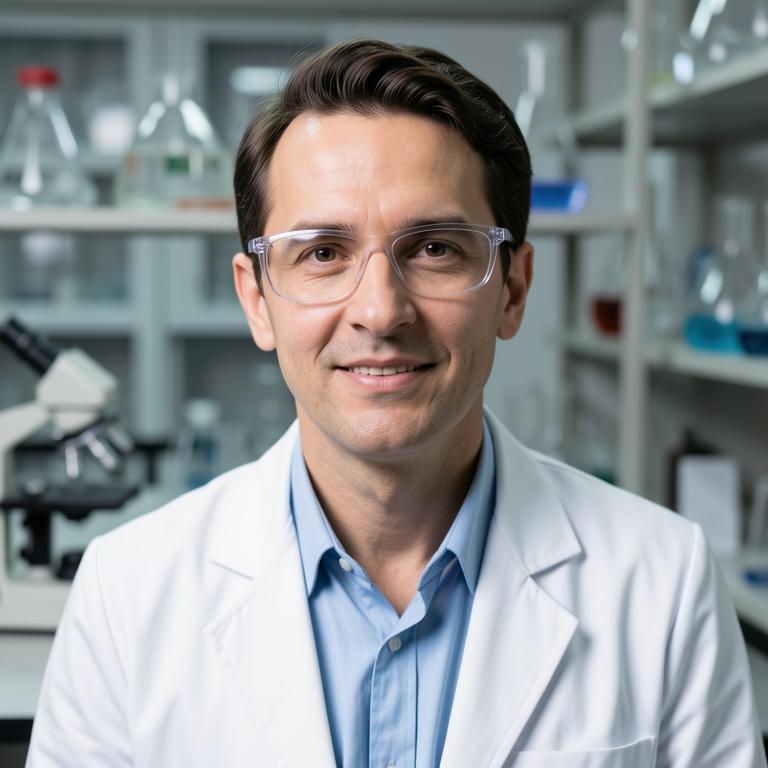}
      & \qcrop{gM}{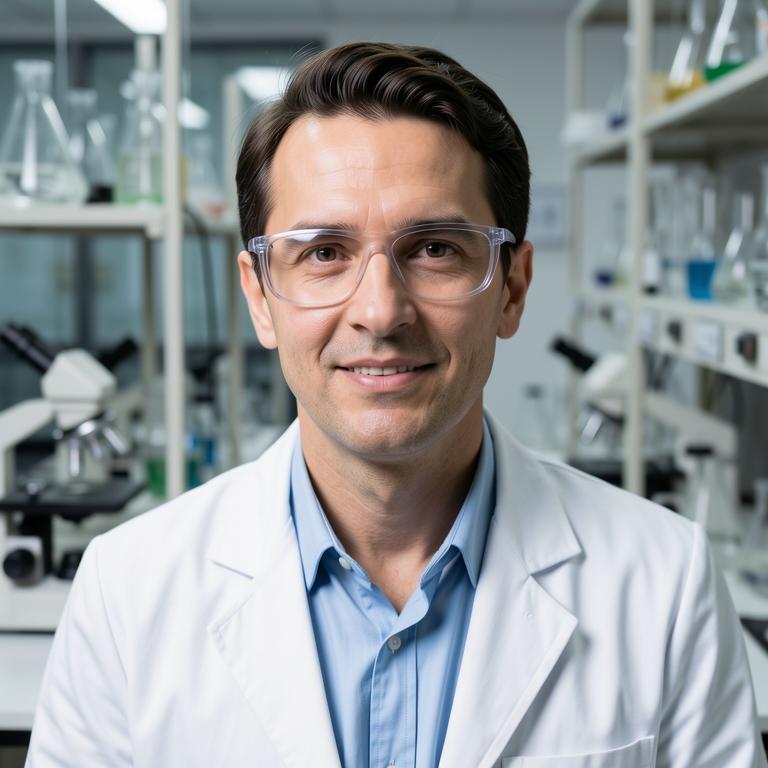}
      & \qcrop{gM}{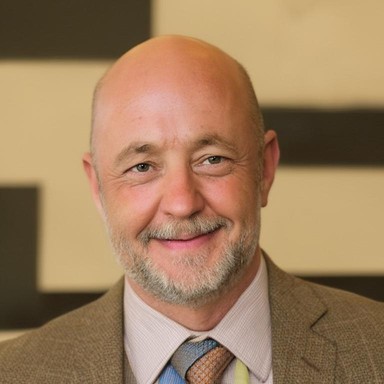}
      & \qcrop{gM}{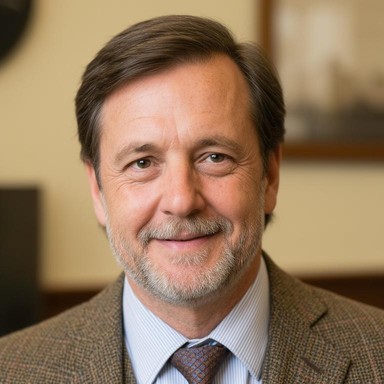}
      & \qcrop{gF}{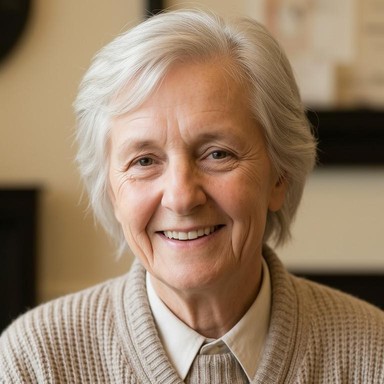}
      & \qcrop{gF}{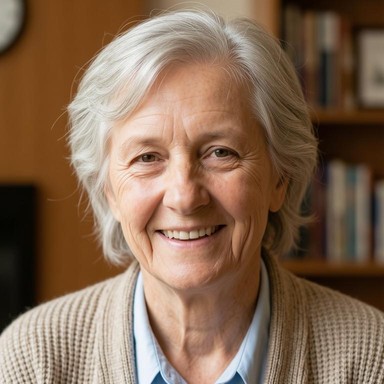}
      & \qcrop{gF}{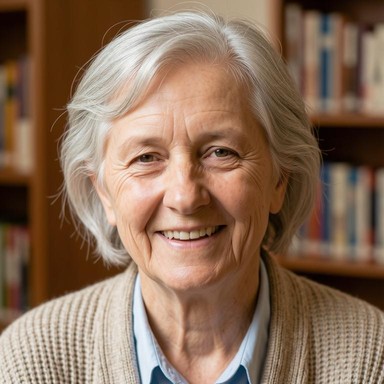} \\[1pt]
    \rowlabellong{\method}
      & \qcrop{gF}{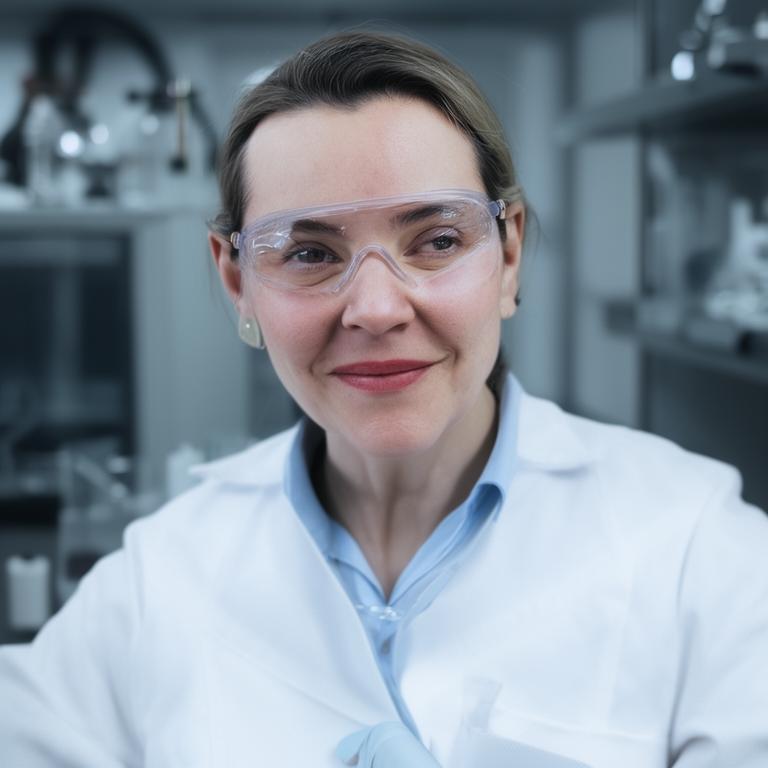}
      & \qcrop{gF}{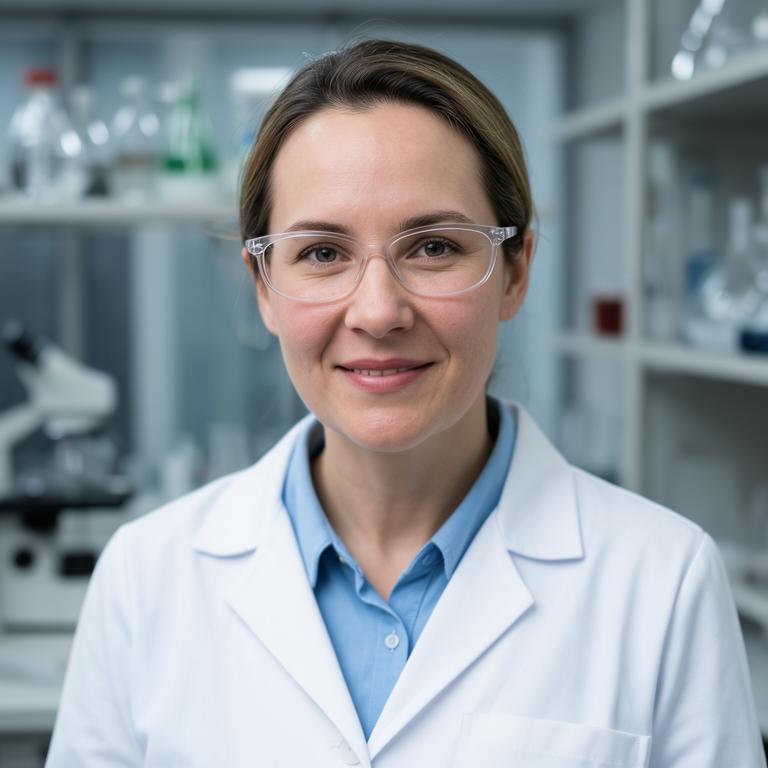}
      & \qcrop{gF}{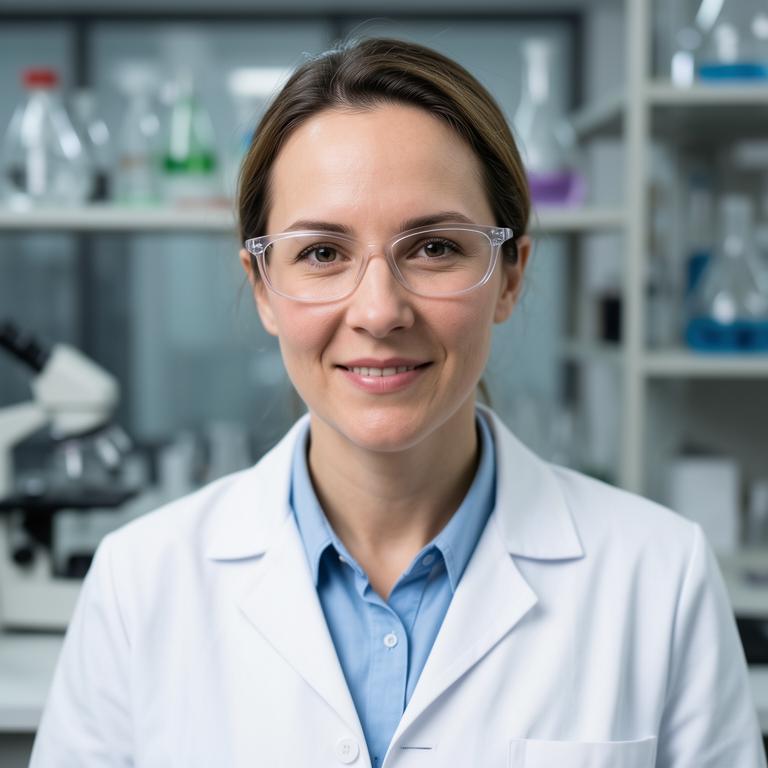}
      & \qcrop{gF}{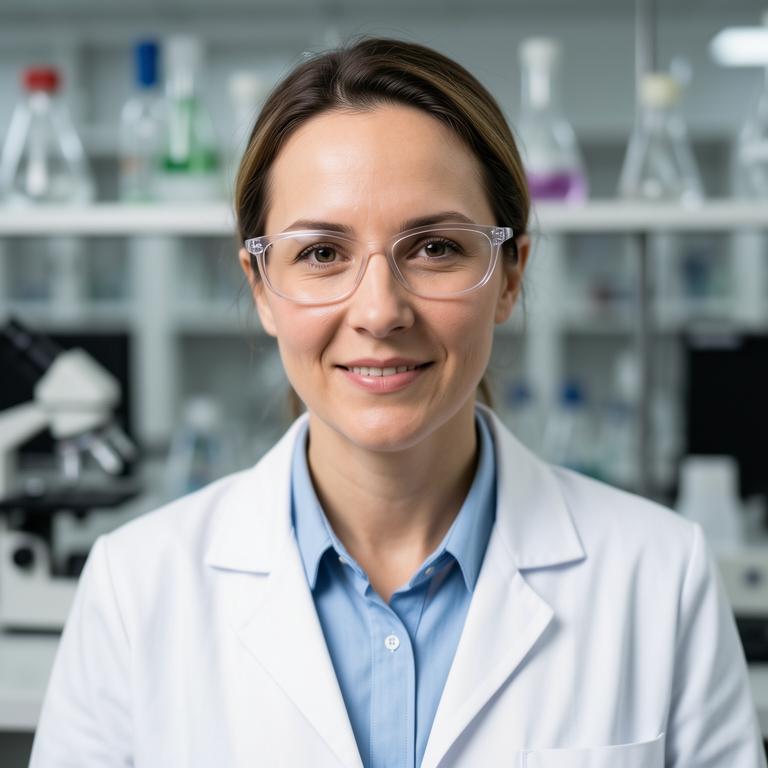}
      & \qcrop{gF}{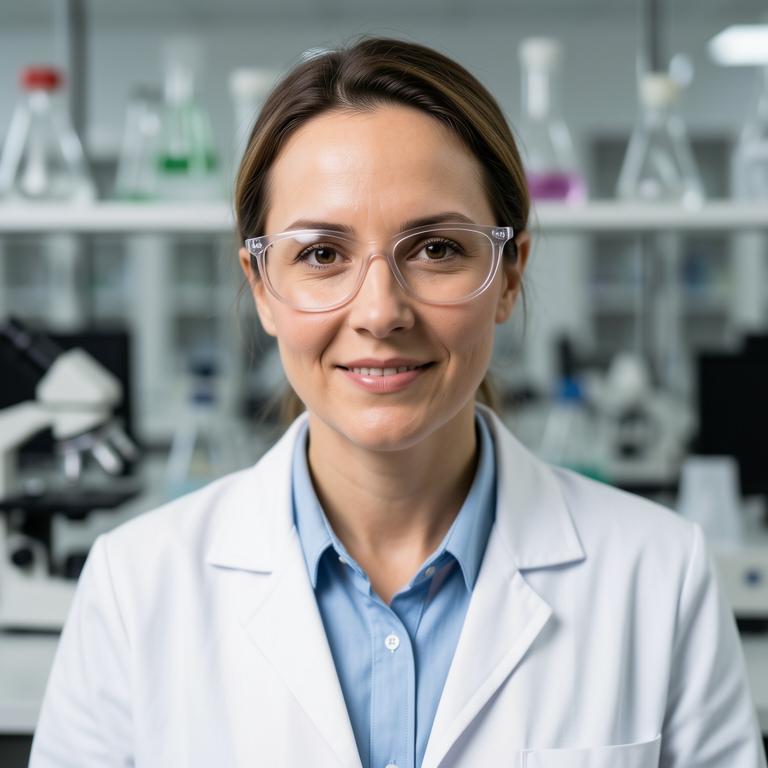}
      & \qcrop{gM}{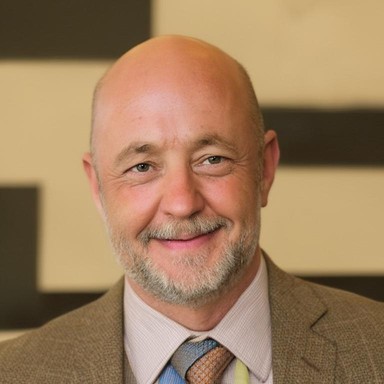}
      & \qcrop{gM}{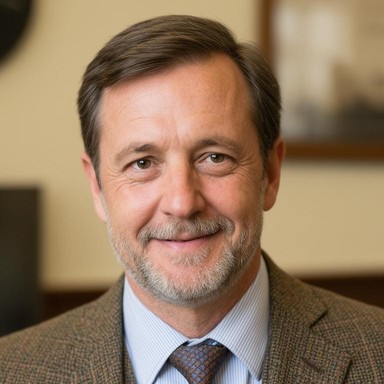}
      & \qcrop{gM}{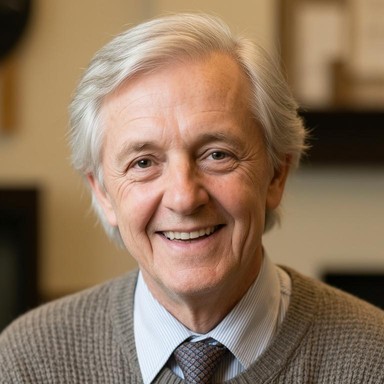}
      & \qcrop{gM}{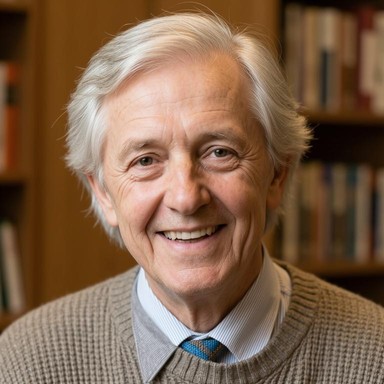}
      & \qcrop{gM}{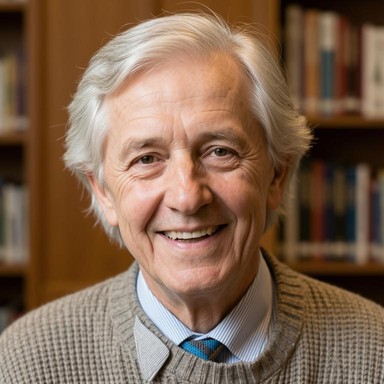} \\
  \end{tabular}
  \caption{\textbf{\method preserves minority gender across guidance scales.}
  Border: \textcolor{gF}{female}, \textcolor{gM}{male}.}
  \label{fig:qual_t2i}
  \vspace{-12pt}
\end{figure}

\newpage

\begingroup
\setlength{\intextsep}{10pt}
\begin{figure}[H]
  \centering
    {\footnotesize Real-world majority:\,\textcolor[HTML]{D62728}{\raisebox{.2ex}{\scalebox{.7}{$\bullet$}}}\,Female\ \textcolor[HTML]{1F4E9C}{\raisebox{.2ex}{\scalebox{.7}{$\bullet$}}}\,Male\hspace{2.5em}Guidance bias:\,{\setlength{\fboxsep}{0pt}\setlength{\fboxrule}{0.2pt}\fcolorbox[HTML]{BBBBBB}{F2DFB8}{\rule{0pt}{3pt}\rule{3.6pt}{0pt}}}\,Amplify\ {\setlength{\fboxsep}{0pt}\setlength{\fboxrule}{0.2pt}\fcolorbox[HTML]{BBBBBB}{E6E6E6}{\rule{0pt}{3pt}\rule{3.6pt}{0pt}}}\,Mitigate}\par\vspace{0pt}
    \begin{Overpic}[percent]{%
      \phantom{\includegraphics[width=\linewidth]{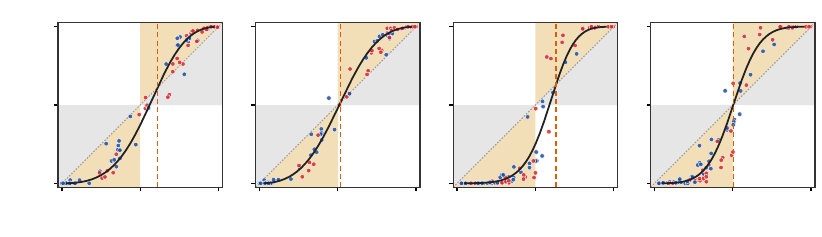}}%
      \llap{\includegraphics[width=\linewidth,trim=0 0 0 8bp,clip]{appB_asym_shift_4_eep50_bls_bare.pdf}}}
      \put(7.70,24.31){\tiny \textcolor{black}{$R^2{=}0.963$}}
      \put(7.55,4.55){\makebox[0pt][c]{\tiny 0}} \put(19.17,4.55){\makebox[0pt][c]{\tiny \textcolor[HTML]{D95F02}{61}}} \put(26.61,4.55){\makebox[0pt][c]{\tiny 100}}
      \put(17.08,2.78){\makebox[0pt][c]{\scriptsize Female ratio (\%, $w{=}1.5$)}}
      \put(6.44,6.16){\makebox[0pt][r]{\tiny 0}} \put(6.44,15.69){\makebox[0pt][r]{\tiny 50}} \put(6.44,25.22){\makebox[0pt][r]{\tiny 100}}
      \put(3.28,16.07){\rotatebox{90}{\makebox[0pt][c]{\scriptsize Female ratio (\%, $w{=}7.5$)}}}
      \put(31.76,24.31){\tiny \textcolor{black}{$R^2{=}0.964$}}
      \put(31.60,4.55){\makebox[0pt][c]{\tiny 0}} \put(41.47,4.55){\makebox[0pt][c]{\tiny \textcolor{black}{52}}} \put(50.66,4.55){\makebox[0pt][c]{\tiny 100}}
      \put(41.13,2.78){\makebox[0pt][c]{\scriptsize Female ratio (\%, $w{=}1.5$)}}
      \put(28.72,16.07){\rotatebox{90}{\makebox[0pt][c]{\scriptsize Female ratio (\%, $w{=}7.5$)}}}
      \put(55.81,24.31){\tiny \textcolor{black}{$R^2{=}0.907$}}
      \put(55.65,4.55){\makebox[0pt][c]{\tiny 0}} \put(65.18,4.55){\makebox[0pt][c]{\tiny 50}} \put(67.72,4.55){\makebox[0pt][c]{\tiny \textcolor[HTML]{D95F02}{63}}} \put(74.71,4.55){\makebox[0pt][c]{\tiny 100}}
      \put(65.18,2.78){\makebox[0pt][c]{\scriptsize Female ratio (\%, $w{=}1.0$)}}
      \put(52.78,16.07){\rotatebox{90}{\makebox[0pt][c]{\scriptsize Female ratio (\%, $w{=}5.0$)}}}
      \put(79.86,24.31){\tiny \textcolor{black}{$R^2{=}0.916$}}
      \put(79.71,4.55){\makebox[0pt][c]{\tiny 0}} \put(89.31,4.55){\makebox[0pt][c]{\tiny \textcolor{black}{50}}} \put(98.77,4.55){\makebox[0pt][c]{\tiny 100}}
      \put(89.24,2.78){\makebox[0pt][c]{\scriptsize Female ratio (\%, $w{=}1.0$)}}
      \put(76.83,16.07){\rotatebox{90}{\makebox[0pt][c]{\scriptsize Female ratio (\%, $w{=}5.0$)}}}
    \end{Overpic}\par\vspace{-6pt}
    \noindent\makebox[\linewidth][l]{%
      \hspace*{0.1708\linewidth}\makebox[0pt][c]{\small (a) SD3, CFG}%
      \hspace*{0.2405\linewidth}\makebox[0pt][c]{\small (b) SD3, $\alpha{=}{-}50$}%
      \hspace*{0.2405\linewidth}\makebox[0pt][c]{\small (c) FLUX.2, CFG}%
      \hspace*{0.2406\linewidth}\makebox[0pt][c]{\small (d) FLUX.2, $\alpha{=}{-}2000$}}
  \par\vspace{2pt}
  \caption{\textcolor{black}{\textbf{Guidance amplifies gender bias asymmetrically in T2I diffusion models.} Each point is an occupation--template pair for the 25 occupations with BLS statistics (\cref{tab:occupations}). Solid curves are logit-space fits; vertical dashed lines mark crossings with $y=x$. In (a,c), deviations from $y=x$ indicate stronger amplification of male than female bias. In (b,d), asymmetry is reduced after shifting $\emptyset$ toward male to obtain $\emptyset_y$ (\cref{eq:ours}).}}
  \label{fig:ablation_alpha}%
  \renewcommand{\thesubfigure}{(\alph{subfigure})}%
  \phantomsubcaption\label{fig:ablation_alpha-a}%
  \phantomsubcaption\label{fig:ablation_alpha-b}%
  \phantomsubcaption\label{fig:ablation_alpha-c}%
  \phantomsubcaption\label{fig:ablation_alpha-d}%
\end{figure}
\endgroup

\vspace{-18pt}
\subsection{Understanding bias amplification through bias decomposition}
\label{sec:discussion}
\begingroup\color{black}
\textbf{Setup.} Using our bias decomposition, we examine how guidance bias contributes to T2I models overrepresenting male majorities more often than female ones relative to workforce ratios~\citep{girrbach2025large,gupta2026newer}. Real-world majority labels follow U.S. workforce statistics~\citep{bls2025occupation}. Following \cref{sec:exp_t2i}, we assess CFG's effect by comparing female ratios at low and high scales over 4 prompt templates and the 25 occupations with available statistics (100 pairs; \cref{tab:occupations}).

\par\endgroup

\begingroup\color{black}

\textbf{Experimental results.}
Guidance bias amplifies the real-world majority gender more often and more strongly in male-majority than in female-majority occupations.
In \cref{fig:ablation_alpha-a,fig:ablation_alpha-c}, guidance amplifies the majority gender in 75.0\% of male-majority pairs but only 70.8\% of female-majority pairs on SD3, and in 96.2\% vs.\ 39.6\% on FLUX.2.
The amplification is also stronger, with a mean increase of 8.4 vs. 5.2 pp on SD3 and 18.2 vs. 13.1 pp on FLUX.2 among amplified pairs.

Our analysis indicates that female bias in the unconditional distribution $p_\theta(\cdot\mid\emptyset)$ biases guidance toward men (\cref{eq:stayfair_cfg}), leading to stronger male than female amplification.
In \cref{fig:ablation_alpha-a,fig:ablation_alpha-c}, the crossings with $y=x$ exceed $50\%$ ($61\%$ on SD3 and $63\%$ on FLUX.2), reflecting female bias in $p_\theta(\cdot\mid\emptyset)$.
In \cref{fig:ablation_alpha-b,fig:ablation_alpha-d}, reducing this female bias mitigates asymmetric amplification.
We reduce female bias in $p_\theta(\cdot\mid\emptyset)$ by shifting $\emptyset$ toward male \textcolor{black}{via} $\emptyset_y$ (\cref{eq:ours}).
\textcolor{black}{The crossings shift to $52\%$ on SD3 and $50\%$ on FLUX.2, while the amplification gap across each crossing narrows ($3.8\!\to\!0.3$ pp on SD3 and $8.2\!\to\!4.7$ pp on FLUX.2).}
These results suggest that asymmetric amplification by guidance helps explain why T2I models overrepresent male majorities more often than female ones.
\par\endgroup

\subsection{Additional analysis}
\label{sec:exp_analysis}

\begin{wraptable}[6]{r}{0.35\linewidth}
    \vspace{-14pt}
    \renewcommand{\arraystretch}{0.82}
    \begin{minipage}[t]{\linewidth}
    \centering
    \footnotesize
    \setlength{\tabcolsep}{1pt}
    \caption{\textcolor{black}{\textbf{Results on race (SD3).}}}
    \label{tab:race_tv}
    \vspace{-6pt}
    \begin{tabular*}{\linewidth}{@{\extracolsep{\fill}} l cc>{\columncolor{targetshade}}c @{}}
        \toprule
        & \multicolumn{3}{c}{Bias (\%) $\downarrow$} \\
        \cmidrule(lr){2-4}
        Guidance & Avg. & Worst & Range \\
        \midrule
        CFG              & \pmv{48.3}{0.4} & \pmv{53.0}{0.2} & \pmv{20.2}{0.9} \\
        \method & \pmvb{45.1}{0.1} & \pmvb{49.9}{0.4} & \pmvb{16.0}{0.3} \\
        \bottomrule
    \end{tabular*}
\end{minipage}%

\end{wraptable}
\textbf{Extension to multi-group attributes.}
\textcolor{black}{\Cref{tab:race_tv} shows that \method reduces guidance bias and total bias across four racial groups.}
\textcolor{black}{We set the attribute direction in \cref{eq:ours} per occupation using the groups with the largest rise and fall in group ratio under CFG.}
\textcolor{black}{These results support extending our null-prompt adjustment beyond binary attributes (\cref{app:race}).}

\begin{wraptable}[8]{r}{0.43\linewidth}
    \vspace{3pt}
    \begin{minipage}[t]{\linewidth}
    \centering
    \footnotesize
    \setlength{\tabcolsep}{1pt}
    \caption{\textcolor{black}{\textbf{Results on unseen prompt (SD3).}}}
    \label{tab:alpha_rule}
    \vspace{-6pt}
    \begin{tabular*}{\linewidth}{@{\extracolsep{\fill}} l cc>{\columncolor{targetshade}}c @{}}
        \toprule
        & \multicolumn{3}{c}{Bias (\%) $\downarrow$} \\
        \cmidrule(lr){2-4}
        Guidance & Avg. & Worst & Range \\
        \midrule
        CFG              & \pmv{38.8}{0.1} & \pmv{42.2}{0.2} & \pmv{9.9}{0.3} \\
        \textbf{\method (Estimate)} & \textcolor{black}{\pmv{35.8}{0.1}} & \textcolor{black}{\pmv{39.1}{0.1}} & \textcolor{black}{\pmv{8.6}{0.2}} \\
        \method (Search) & \pmv{34.4}{0.2} & \pmv{37.5}{0.3} & \pmv{7.8}{0.4} \\
        \bottomrule
    \end{tabular*}
\end{minipage}%

\end{wraptable}
\textcolor{black}{\textbf{Evaluating \method\ (Estimate) on unseen prompts.}}
\textcolor{black}{\Cref{tab:alpha_rule} shows that \method (Estimate) achieves intermediate performance between CFG and \method (Search) in the unseen prompt setting.}
\textcolor{black}{To evaluate unseen prompts, we exclude all records for the test occupation--template pair when constructing the lookup table.}
\textcolor{black}{We bootstrap group ratios from five-image samples at the lowest and highest scales to evaluate expected performance (\cref{app:alpha-analysis}).}
\textcolor{black}{These results suggest that our lookup-based approach can mitigate guidance bias on unseen prompts with a sampling budget of 10 images.}

\vspace{-3pt}
\section{Conclusion}
\vspace{-3pt}
\label{sec:conclusion}

\label{sec:conclusion_summary}
\textcolor{black}{We introduce guidance scale-invariant group fairness and show that total bias splits into model bias and guidance bias.}
\textcolor{black}{\method mitigates guidance bias in CG and CFG, preserving image quality and group ratios across scales, and complements existing model debiasing.}
\textcolor{black}{Our decomposition also explains how guidance contributes to asymmetric gender bias amplification in T2I diffusion models.}

\textbf{Limitations.}
\textcolor{black}{First, \method addresses guidance bias, but achieving guidance scale-invariant group fairness also requires reducing model bias.}  \textcolor{black}{Jointly debiasing the model and its guidance is a natural next step.}\ %
\textcolor{black}{Second, a performance gap remains between \method (Estimate) and \method (Search) on unseen prompts. Closing this gap requires more accurate estimation, and real-time use requires greater efficiency.}
\textcolor{black}{Finally, our evaluation focuses on binary gender with a supplementary analysis of race, and extending \method to multi-group and intersectional settings remains future work.}

\textcolor{black}{\textbf{Broader impact.} CFG is widely used in diffusion research and applications such as synthetic-data generation and model distillation. 
Our work highlights the role of guidance in the bias of diffusion models by isolating it within total bias.}
\textcolor{black}{We hope this finding helps practitioners use diffusion models more reliably and responsibly.}

\begin{ack}
\end{ack}

This work was partly supported by the National Research Foundation of Korea (NRF) grant funded by the Korean government (MSIT) (Nos. RS-2026-25608412, RS-2026-25607210) and the Institute of Information \& Communications Technology Planning \& Evaluation (IITP) grant funded by the Korea government (MSIT) (No. RS-2026-25524173, Ultra-Long-Term Hierarchical Memory and Reasoning Architecture for Next-Generation Omnimodal Agents).

\bibliographystyle{unsrtnat}
\bibliography{bib/main}

\begin{thebibliography}{70}
\providecommand{\natexlab}[1]{#1}
\providecommand{\url}[1]{\texttt{#1}}
\expandafter\ifx\csname urlstyle\endcsname\relax
  \providecommand{\doi}[1]{doi: #1}\else
  \providecommand{\doi}{doi: \begingroup \urlstyle{rm}\Url}\fi

\bibitem[Rombach et~al.(2022)Rombach, Blattmann, Lorenz, Esser, and
  Ommer]{rombach2022high}
Robin Rombach, Andreas Blattmann, Dominik Lorenz, Patrick Esser, and Bj{\"o}rn
  Ommer.
\newblock High-resolution image synthesis with latent diffusion models.
\newblock In \emph{CVPR}, pages 10684--10695, 2022.

\bibitem[Esser et~al.(2024)Esser, Kulal, Blattmann, Entezari, M{\"u}ller,
  Saini, Levi, Lorenz, Sauer, Boesel, et~al.]{esser2024scaling}
Patrick Esser, Sumith Kulal, Andreas Blattmann, Rahim Entezari, Jonas
  M{\"u}ller, Harry Saini, Yam Levi, Dominik Lorenz, Axel Sauer, Frederic
  Boesel, et~al.
\newblock Scaling rectified flow transformers for high-resolution image
  synthesis.
\newblock In \emph{ICML}, pages 12606--12633, 2024.

\bibitem[Dhariwal and Nichol(2021)]{dhariwal2021diffusion}
Prafulla Dhariwal and Alexander Nichol.
\newblock Diffusion models beat {GANs} on image synthesis.
\newblock In \emph{NeurIPS}, volume~34, pages 8780--8794, 2021.

\bibitem[Ho and Salimans(2021)]{ho2022classifier}
Jonathan Ho and Tim Salimans.
\newblock Classifier-free diffusion guidance.
\newblock In \emph{NeurIPS Workshop on Deep Generative Models and Downstream
  Applications}, 2021.

\bibitem[Kim et~al.(2025{\natexlab{a}})Kim, Kim, Park, Entezari, and
  Yoon]{kim2025rethinking}
Eunji Kim, Siwon Kim, Minjun Park, Rahim Entezari, and Sungroh Yoon.
\newblock Rethinking training for de-biasing text-to-image generation:
  Unlocking the potential of stable diffusion.
\newblock In \emph{CVPR}, pages 13361--13370, 2025{\natexlab{a}}.

\bibitem[Kim et~al.(2024{\natexlab{a}})Kim, Na, Park, Jang, Kim, Kang, and
  Moon]{kim2024unbiased}
Yeongmin Kim, Byeonghu Na, Minsang Park, JoonHo Jang, Dongjun Kim, Wanmo Kang,
  and Il-Chul Moon.
\newblock Training unbiased diffusion models from biased dataset.
\newblock In \emph{ICLR}, 2024{\natexlab{a}}.

\bibitem[Shen et~al.(2024)Shen, Du, Pang, Lin, Wong, and
  Kankanhalli]{shenfinetuning}
Xudong Shen, Chao Du, Tianyu Pang, Min Lin, Yongkang Wong, and Mohan
  Kankanhalli.
\newblock Finetuning text-to-image diffusion models for fairness.
\newblock In \emph{ICLR}, 2024.

\bibitem[Seshadri et~al.(2024)Seshadri, Singh, and Elazar]{seshadri2024bias}
Preethi Seshadri, Sameer Singh, and Yanai Elazar.
\newblock The bias amplification paradox in text-to-image generation.
\newblock In \emph{NAACL}, pages 6367--6384, 2024.

\bibitem[Jiang et~al.(2020)Jiang, Pacchiano, Stepleton, Jiang, and
  Chiappa]{jiang2020wasserstein}
Ray Jiang, Aldo Pacchiano, Tom Stepleton, Heinrich Jiang, and Silvia Chiappa.
\newblock Wasserstein fair classification.
\newblock In \emph{UAI}, pages 862--872, 2020.

\bibitem[{Black Forest Labs}(2025)]{flux-2-2025}
{Black Forest Labs}.
\newblock {FLUX.2: Frontier Visual Intelligence}.
\newblock \url{https://bfl.ai/blog/flux-2}, 2025.

\bibitem[Gandikota et~al.(2024)Gandikota, Orgad, Belinkov, Materzy{\'n}ska, and
  Bau]{gandikota2024unified}
Rohit Gandikota, Hadas Orgad, Yonatan Belinkov, Joanna Materzy{\'n}ska, and
  David Bau.
\newblock Unified concept editing in diffusion models.
\newblock In \emph{WACV}, pages 5111--5120, 2024.

\bibitem[Girrbach et~al.(2025)Girrbach, Alaniz, Smith, and
  Akata]{girrbach2025large}
Leander Girrbach, Stephan Alaniz, Genevieve Smith, and Zeynep Akata.
\newblock A large scale analysis of gender biases in text-to-image generative
  models.
\newblock \emph{arXiv preprint arXiv:2503.23398}, 2025.

\bibitem[Cho et~al.(2023)Cho, Zala, and Bansal]{cho2023dall}
Jaemin Cho, Abhay Zala, and Mohit Bansal.
\newblock {DALL-Eval}: Probing the reasoning skills and social biases of
  text-to-image generation models.
\newblock In \emph{ICCV}, pages 3043--3054, 2023.

\bibitem[Luccioni et~al.(2023)Luccioni, Akiki, Mitchell, and
  Jernite]{luccioni2023stable}
Sasha Luccioni, Christopher Akiki, Margaret Mitchell, and Yacine Jernite.
\newblock Stable bias: Evaluating societal representations in diffusion models.
\newblock In \emph{NeurIPS}, volume~36, pages 56338--56351, 2023.

\bibitem[Bird et~al.(2023)Bird, Ungless, and Kasirzadeh]{bird2023typology}
Charlotte Bird, Eddie Ungless, and Atoosa Kasirzadeh.
\newblock Typology of risks of generative text-to-image models.
\newblock In \emph{AAAI/ACM AIES}, pages 396--410, 2023.

\bibitem[Barroso~da Silveira and Alves~Lima(2024)]{barroso2024racial}
Julia Barroso~da Silveira and Ellen Alves~Lima.
\newblock Racial biases in {AIs} and {Gemini}'s inability to write narratives
  about {B}lack people.
\newblock \emph{Emerging Media}, 2\penalty0 (2):\penalty0 277--287, 2024.

\bibitem[Bianchi et~al.(2023)Bianchi, Kalluri, Durmus, Ladhak, Cheng, Nozza,
  Hashimoto, Jurafsky, Zou, and Caliskan]{bianchi2023easily}
Federico Bianchi, Pratyusha Kalluri, Esin Durmus, Faisal Ladhak, Myra Cheng,
  Debora Nozza, Tatsunori Hashimoto, Dan Jurafsky, James Zou, and Aylin
  Caliskan.
\newblock Easily accessible text-to-image generation amplifies demographic
  stereotypes at large scale.
\newblock In \emph{FAccT}, pages 1493--1504, 2023.

\bibitem[Wu et~al.(2024)Wu, Nakashima, and Garcia]{wu2024stable}
Yankun Wu, Yuta Nakashima, and Noa Garcia.
\newblock {Stable Diffusion} exposed: Gender bias from prompt to image.
\newblock In \emph{AAAI/ACM AIES}, 2024.

\bibitem[Friedrich et~al.(2023)Friedrich, Brack, Struppek, Hintersdorf,
  Schramowski, Luccioni, and Kersting]{friedrich2023fair}
Felix Friedrich, Manuel Brack, Lukas Struppek, Dominik Hintersdorf, Patrick
  Schramowski, Sasha Luccioni, and Kristian Kersting.
\newblock Fair diffusion: Instructing text-to-image generation models on
  fairness.
\newblock \emph{arXiv preprint arXiv:2302.10893}, 2023.

\bibitem[Parihar et~al.(2024)Parihar, Bhat, Basu, Mallick, Kundu, and
  Babu]{parihar2024balancing}
Rishubh Parihar, Abhijnya Bhat, Abhipsa Basu, Saswat Mallick, Jogendra~Nath
  Kundu, and R.~Venkatesh Babu.
\newblock Balancing act: Distribution-guided debiasing in diffusion models.
\newblock In \emph{CVPR}, pages 6668--6678, 2024.

\bibitem[Teo et~al.(2024)Teo, Abdollahzadeh, Ma, and Cheung]{teo2025fairqueue}
Christopher T~H Teo, Milad Abdollahzadeh, Xinda Ma, and Ngai-Man Cheung.
\newblock {FairQueue}: Rethinking prompt learning for fair text-to-image
  generation.
\newblock In \emph{NeurIPS}, volume~37, pages 22878--22926, 2024.

\bibitem[Li et~al.(2024)Li, Shen, Torr, Tresp, and Gu]{li2024self}
Hang Li, Chengzhi Shen, Philip Torr, Volker Tresp, and Jindong Gu.
\newblock Self-discovering interpretable diffusion latent directions for
  responsible text-to-image generation.
\newblock In \emph{CVPR}, pages 12006--12016, 2024.

\bibitem[Park et~al.(2025)Park, Lee, Chung, Lee, Choo, and Gu]{park2025fair}
Jeonghoon Park, Juyoung Lee, Chaeyeon Chung, Jaeseong Lee, Jaegul Choo, and
  Jindong Gu.
\newblock Fair generation without unfair distortions: Debiasing text-to-image
  generation with entanglement-free attention.
\newblock In \emph{ICCV}, pages 17567--17576, 2025.

\bibitem[Kim et~al.(2023)Kim, Kim, Shin, and Yoon]{kim2024destereotyping}
Eunji Kim, Siwon Kim, Chaehun Shin, and Sungroh Yoon.
\newblock De-stereotyping text-to-image models through prompt tuning.
\newblock In \emph{ICML Workshop on Challenges in Deployable Generative AI},
  2023.

\bibitem[Chuang et~al.(2023)Chuang, Jampani, Li, Torralba, and
  Jegelka]{chuang2023debiasing}
Ching-Yao Chuang, Varun Jampani, Yuanzhen Li, Antonio Torralba, and Stefanie
  Jegelka.
\newblock Debiasing vision-language models via biased prompts.
\newblock \emph{arXiv preprint arXiv:2302.00070}, 2023.

\bibitem[Han et~al.(2025)Han, Xu, Bao, Yang, Zi, and Huang]{hanlightfair}
Boyu Han, Qianqian Xu, Shilong Bao, Zhiyong Yang, Kangli Zi, and Qingming
  Huang.
\newblock {LightFair}: Towards an efficient alternative for fair {T2I}
  diffusion via debiasing pre-trained text encoders.
\newblock In \emph{NeurIPS}, volume~38, pages 22671--22724, 2025.

\bibitem[Mo et~al.(2019)Mo, Kim, Kim, Cho, and Shin]{mo2019mining}
Sangwoo Mo, Chiheon Kim, Sungwoong Kim, Minsu Cho, and Jinwoo Shin.
\newblock Mining {GOLD} samples for conditional {GAN}s.
\newblock In \emph{NeurIPS}, 2019.

\bibitem[Kim et~al.(2026)Kim, Mo, Rizve, Xu, Liu, Shin, and
  Hinz]{kim2026rethinking}
Subin Kim, Sangwoo Mo, Mamshad~Nayeem Rizve, Yiran Xu, Difan Liu, Jinwoo Shin,
  and Tobias Hinz.
\newblock Rethinking prompt design for inference-time scaling in text-to-visual
  generation.
\newblock In \emph{CVPR}, 2026.

\bibitem[Bansal et~al.(2023)Bansal, Chu, Schwarzschild, Sengupta, Goldblum,
  Geiping, and Goldstein]{bansal2023universal}
Arpit Bansal, Hong-Min Chu, Avi Schwarzschild, Soumyadip Sengupta, Micah
  Goldblum, Jonas Geiping, and Tom Goldstein.
\newblock Universal guidance for diffusion models.
\newblock In \emph{CVPRW}, pages 843--852, 2023.

\bibitem[Ye et~al.(2024)Ye, Lin, Han, Xu, Liu, Liang, Ma, Zou, and
  Ermon]{ye2024tfg}
Haotian Ye, Haowei Lin, Jiaqi Han, Minkai Xu, Sheng Liu, Yitao Liang, Jianzhu
  Ma, James Zou, and Stefano Ermon.
\newblock {TFG}: Unified training-free guidance for diffusion models.
\newblock In \emph{NeurIPS}, volume~37, pages 22370--22417, 2024.

\bibitem[Kynk\"a\"anniemi et~al.(2024)Kynk\"a\"anniemi, Aittala, Karras, Laine,
  Aila, and Lehtinen]{kynkaanniemi2024applying}
Tuomas Kynk\"a\"anniemi, Miika Aittala, Tero Karras, Samuli Laine, Timo Aila,
  and Jaakko Lehtinen.
\newblock Applying guidance in a limited interval improves sample and
  distribution quality in diffusion models.
\newblock In \emph{NeurIPS}, 2024.

\bibitem[Sadat et~al.(2025)Sadat, Hilliges, and Weber]{sadat2024eliminating}
Seyedmorteza Sadat, Otmar Hilliges, and Romann~M. Weber.
\newblock Eliminating oversaturation and artifacts of high guidance scales in
  diffusion models.
\newblock In \emph{ICLR}, 2025.

\bibitem[Karras et~al.(2024)Karras, Aittala, Kynk{\"a}{\"a}nniemi, Lehtinen,
  Aila, and Laine]{karras2024guiding}
Tero Karras, Miika Aittala, Tuomas Kynk{\"a}{\"a}nniemi, Jaakko Lehtinen, Timo
  Aila, and Samuli Laine.
\newblock Guiding a diffusion model with a bad version of itself.
\newblock In \emph{NeurIPS}, volume~37, pages 52996--53021, 2024.

\bibitem[Hong et~al.(2023)Hong, Lee, Jang, and Kim]{hong2023improving}
Susung Hong, Gyuseong Lee, Wooseok Jang, and Seungryong Kim.
\newblock Improving sample quality of diffusion models using self-attention
  guidance.
\newblock In \emph{ICCV}, pages 7462--7471, 2023.

\bibitem[Ahn et~al.(2024)Ahn, Cho, Min, Jang, Kim, Kim, Park, Jin, and
  Kim]{ahn2024self}
Donghoon Ahn, Hyoungwon Cho, Jaewon Min, Wooseok Jang, Jungwoo Kim, SeonHwa
  Kim, Hyun~Hee Park, Kyong~Hwan Jin, and Seungryong Kim.
\newblock Self-rectifying diffusion sampling with perturbed-attention guidance.
\newblock In \emph{ECCV}, pages 1--17, 2024.

\bibitem[Roos et~al.(2026)Roos, Iakovleva, Gjergji, Pastore, and
  Tartaglione]{roos2026met}
Nathan Roos, Ekaterina Iakovleva, Ani Gjergji, Vito~Paolo Pastore, and Enzo
  Tartaglione.
\newblock How {I} met your bias: Investigating bias amplification in diffusion
  models.
\newblock In \emph{WACV}, pages 5374--5383, 2026.

\bibitem[Pessach and Shmueli(2022)]{pessach2022review}
Dana Pessach and Erez Shmueli.
\newblock A review on fairness in machine learning.
\newblock \emph{ACM computing surveys (CSUR)}, 55\penalty0 (3):\penalty0 1--44,
  2022.

\bibitem[Kim et~al.(2024{\natexlab{b}})Kim, Mo, Kim, Lee, Lee, and
  Shin]{kim2024discovering}
Younghyun Kim, Sangwoo Mo, Minkyu Kim, Kyungmin Lee, Jaeho Lee, and Jinwoo
  Shin.
\newblock Discovering and mitigating visual biases through keyword explanation.
\newblock In \emph{CVPR}, 2024{\natexlab{b}}.

\bibitem[Bae et~al.(2025)Bae, Ok, Mo, and Lee]{bae2025reasoning}
Jiyun Bae, Hyunjong Ok, Sangwoo Mo, and Jaeho Lee.
\newblock Understanding the effects of distractors on reasoning vision-language
  models.
\newblock \emph{arXiv preprint arXiv:2511.21397}, 2025.

\bibitem[Dwork et~al.(2012)Dwork, Hardt, Pitassi, Reingold, and
  Zemel]{dwork2012fairness}
Cynthia Dwork, Moritz Hardt, Toniann Pitassi, Omer Reingold, and Richard Zemel.
\newblock Fairness through awareness.
\newblock In \emph{ITCS}, 2012.

\bibitem[Hardt et~al.(2016)Hardt, Price, and Srebro]{hardt2016equality}
Moritz Hardt, Eric Price, and Nathan Srebro.
\newblock Equality of opportunity in supervised learning.
\newblock In \emph{NeurIPS}, 2016.

\bibitem[Zafar et~al.(2017)Zafar, Valera, Rodriguez, and
  Gummadi]{zafar2017fairness}
Muhammad~Bilal Zafar, Isabel Valera, Manuel~Gomez Rodriguez, and Krishna~P.
  Gummadi.
\newblock Fairness constraints: Mechanisms for fair classification.
\newblock In \emph{AISTATS}, 2017.

\bibitem[Agarwal et~al.(2018)Agarwal, Beygelzimer, Dud{\'\i}k, Langford, and
  Wallach]{agarwal2018reductions}
Alekh Agarwal, Alina Beygelzimer, Miroslav Dud{\'\i}k, John Langford, and Hanna
  Wallach.
\newblock A reductions approach to fair classification.
\newblock In \emph{ICML}, 2018.

\bibitem[Sagawa et~al.(2020)Sagawa, Koh, Hashimoto, and
  Liang]{sagawa2020distributionally}
Shiori Sagawa, Pang~Wei Koh, Tatsunori~B. Hashimoto, and Percy Liang.
\newblock Distributionally robust neural networks for group shifts: On the
  importance of regularization for worst-case generalization.
\newblock In \emph{ICLR}, 2020.

\bibitem[Idrissi et~al.(2022)Idrissi, Arjovsky, Pezeshki, and
  Lopez-Paz]{idrissi2022simple}
Badr~Youbi Idrissi, Martin Arjovsky, Mohammad Pezeshki, and David Lopez-Paz.
\newblock Simple data balancing achieves competitive worst-group-accuracy.
\newblock In \emph{CLeaR}, 2022.

\bibitem[Chzhen et~al.(2020)Chzhen, Denis, Hebiri, Oneto, and
  Pontil]{chzhen2020fair}
Evgenii Chzhen, Christophe Denis, Mohamed Hebiri, Luca Oneto, and Massimiliano
  Pontil.
\newblock Fair regression with {Wasserstein} barycenters.
\newblock In \emph{NeurIPS}, volume~33, pages 7321--7331, 2020.

\bibitem[Farokhi(2021)]{farokhi2021optimal}
Farhad Farokhi.
\newblock Optimal pre-processing to achieve fairness and its relationship with
  total variation barycenter.
\newblock \emph{arXiv preprint arXiv:2101.06811}, 2021.

\bibitem[Kim et~al.(2025{\natexlab{b}})Kim, Kong, Lee, Chae, Park, and
  Kim]{kim2025fairness}
Kunwoong Kim, Insung Kong, Jongjin Lee, Minwoo Chae, Sangchul Park, and Yongdai
  Kim.
\newblock Fairness through matching.
\newblock \emph{Transactions on Machine Learning Research}, 2025{\natexlab{b}}.

\bibitem[Ho et~al.(2020)Ho, Jain, and Abbeel]{ho2020denoising}
Jonathan Ho, Ajay Jain, and Pieter Abbeel.
\newblock Denoising diffusion probabilistic models.
\newblock In \emph{NeurIPS}, volume~33, pages 6840--6851, 2020.

\bibitem[Song et~al.(2021)Song, Sohl-Dickstein, Kingma, Kumar, Ermon, and
  Poole]{songscore}
Yang Song, Jascha Sohl-Dickstein, Diederik~P Kingma, Abhishek Kumar, Stefano
  Ermon, and Ben Poole.
\newblock Score-based generative modeling through stochastic differential
  equations.
\newblock In \emph{ICLR}, 2021.

\bibitem[Karras et~al.(2022)Karras, Aittala, Aila, and
  Laine]{karras2022elucidating}
Tero Karras, Miika Aittala, Timo Aila, and Samuli Laine.
\newblock Elucidating the design space of diffusion-based generative models.
\newblock In \emph{NeurIPS}, volume~35, pages 26565--26577, 2022.

\bibitem[Zheng and Lan(2024)]{zheng2024characteristic}
Candi Zheng and Yuan Lan.
\newblock Characteristic guidance: Non-linear correction for diffusion model at
  large guidance scale.
\newblock In \emph{ICML}, pages 61386--61412, 2024.

\bibitem[Li and Jiao(2025)]{li2025provable}
Gen Li and Yuchen Jiao.
\newblock Provable efficiency of guidance in diffusion models for general data
  distribution.
\newblock In \emph{ICML}, pages 35034--35046, 2025.

\bibitem[Azangulov et~al.(2026)Azangulov, Potaptchik, Li, Aamari,
  Deligiannidis, and Rousseau]{azangulovadaptive}
Iskander Azangulov, Peter Potaptchik, Qinyu Li, Eddie Aamari, George
  Deligiannidis, and Judith Rousseau.
\newblock Adaptive diffusion guidance via stochastic optimal control.
\newblock In \emph{AISTATS}, 2026.

\bibitem[Couairon et~al.(2022)Couairon, Douze, Cord, and
  Schwenk]{couairon2022embedding}
Guillaume Couairon, Matthijs Douze, Matthieu Cord, and Holger Schwenk.
\newblock Embedding arithmetic of multimodal queries for image retrieval.
\newblock In \emph{Proceedings of the IEEE/CVF Conference on Computer Vision
  and Pattern Recognition (CVPR) Workshops}, pages 4950--4958, June 2022.

\bibitem[Liu et~al.(2015)Liu, Luo, Wang, and Tang]{liu2015deep}
Ziwei Liu, Ping Luo, Xiaogang Wang, and Xiaoou Tang.
\newblock Deep learning face attributes in the wild.
\newblock In \emph{ICCV}, pages 3730--3738, 2015.

\bibitem[Ning et~al.(2023)Ning, Sangineto, Porrello, Calderara, and
  Cucchiara]{ning2023input}
Mang Ning, Enver Sangineto, Angelo Porrello, Simone Calderara, and Rita
  Cucchiara.
\newblock Input perturbation reduces exposure bias in diffusion models.
\newblock In \emph{ICML}, pages 26245--26265, 2023.

\bibitem[Karkkainen and Joo(2021)]{karkkainenfairface}
Kimmo Karkkainen and Jungseock Joo.
\newblock {FairFace}: Face attribute dataset for balanced race, gender, and age
  for bias measurement and mitigation.
\newblock In \emph{WACV}, pages 1548--1558, 2021.

\bibitem[Heusel et~al.(2017)Heusel, Ramsauer, Unterthiner, Nessler, and
  Hochreiter]{heusel2017gans}
Martin Heusel, Hubert Ramsauer, Thomas Unterthiner, Bernhard Nessler, and Sepp
  Hochreiter.
\newblock {GAN}s trained by a two time-scale update rule converge to a local
  {N}ash equilibrium.
\newblock In \emph{NeurIPS}, pages 6626--6637, 2017.

\bibitem[Nam et~al.(2023)Nam, Mo, Lee, and Shin]{nam2023breaking}
Junhyun Nam, Sangwoo Mo, Jaeho Lee, and Jinwoo Shin.
\newblock Breaking the spurious causality of conditional generation via
  fairness intervention with corrective sampling.
\newblock \emph{Transactions on Machine Learning Research}, 2023.

\bibitem[Stein et~al.(2023)Stein, Cresswell, Hosseinzadeh, Sui, Ross,
  Villecroze, Liu, Caterini, Taylor, and Loaiza-Ganem]{stein2023exposing}
George Stein, Jesse~C. Cresswell, Rasa Hosseinzadeh, Yi~Sui, Brendan~Leigh
  Ross, Valentin Villecroze, Zhaoyan Liu, Anthony~L. Caterini, J.~Eric~T.
  Taylor, and Gabriel Loaiza-Ganem.
\newblock Exposing flaws of generative model evaluation metrics and their
  unfair treatment of diffusion models.
\newblock In \emph{NeurIPS}, 2023.

\bibitem[Radford et~al.(2021)Radford, Kim, Hallacy, Ramesh, Goh, Agarwal,
  Sastry, Askell, Mishkin, Clark, et~al.]{radford2021learning}
Alec Radford, Jong~Wook Kim, Chris Hallacy, Aditya Ramesh, Gabriel Goh,
  Sandhini Agarwal, Girish Sastry, Amanda Askell, Pamela Mishkin, Jack Clark,
  et~al.
\newblock Learning transferable visual models from natural language
  supervision.
\newblock In \emph{ICML}, pages 8748--8763, 2021.

\bibitem[{Ultralytics}()]{ultralyticsYOLO11}
{Ultralytics}.
\newblock Ultralytics {YOLO11}.
\newblock URL \url{https://docs.ultralytics.com/models/yolo11/}.

\bibitem[Hessel et~al.(2021)Hessel, Holtzman, Forbes, Le~Bras, and
  Choi]{hessel2021clipscore}
Jack Hessel, Ari Holtzman, Maxwell Forbes, Ronan Le~Bras, and Yejin Choi.
\newblock {CLIPScore}: A reference-free evaluation metric for image captioning.
\newblock In \emph{EMNLP}, pages 7514--7528, 2021.

\bibitem[{rsinema}()]{rsinemaAestheticScorer}
{rsinema}.
\newblock Aesthetic scorer.
\newblock URL \url{https://huggingface.co/rsinema/aesthetic-scorer}.

\bibitem[Kirstain et~al.(2023)Kirstain, Polyak, Singer, Matiana, Penna, and
  Levy]{kirstain2023pick}
Yuval Kirstain, Adam Polyak, Uriel Singer, Shahbuland Matiana, Joe Penna, and
  Omer Levy.
\newblock {Pick-a-Pic}: An open dataset of user preferences for text-to-image
  generation.
\newblock In \emph{NeurIPS}, volume~36, pages 36652--36663, 2023.

\bibitem[Gupta and Brown(2026)]{gupta2026newer}
Shesh~Narayan Gupta and Nik~Bear Brown.
\newblock {Newer Is Not Fairer}: Gender stereotyping in text-to-image ai across
  model generations.
\newblock \emph{arXiv preprint arXiv:2609.18007}, 2026.

\bibitem[{U.S. Bureau of Labor Statistics}(2025)]{bls2025occupation}
{U.S. Bureau of Labor Statistics}.
\newblock Employed people by detailed occupation, sex, race, and {Hispanic} or
  {Latino} ethnicity.
\newblock Current Population Survey, 2025 annual averages, Table 11, 2025.
\newblock URL \url{https://www.bls.gov/cps/data/aa2025/cpsa2025.pdf}.

\bibitem[Choi et~al.(2022)Choi, Lee, Shin, Kim, Kim, and
  Yoon]{choi2022perception}
Jooyoung Choi, Jungbeom Lee, Chaehun Shin, Sungwon Kim, Hyunwoo Kim, and
  Sungroh Yoon.
\newblock Perception prioritized training of diffusion models.
\newblock In \emph{CVPR}, pages 11472--11481, 2022.

\bibitem[Oquab et~al.(2024)Oquab, Darcet, Moutakanni, Vo, Szafraniec, Khalidov,
  Fernandez, Haziza, Massa, El-Nouby, et~al.]{oquab2024dinov2}
Maxime Oquab, Timoth{\'e}e Darcet, Th{\'e}o Moutakanni, Huy Vo, Marc
  Szafraniec, Vasil Khalidov, Pierre Fernandez, Daniel Haziza, Francisco Massa,
  Alaaeldin El-Nouby, et~al.
\newblock {DINOv2}: Learning robust visual features without supervision.
\newblock \emph{Transactions on Machine Learning Research}, 2024.

\end{thebibliography}

\etocdepthtag.toc{mtchapter}
\clearpage
\appendix
\crefalias{section}{appendix}
\crefalias{subsection}{appendix}      %
\crefalias{subsubsection}{appendix}   %
\etocdepthtag.toc{mtappendix}
\hypersetup{linkcolor=black}
\etocsettagdepth{mtchapter}{none}
\etocsettagdepth{mtappendix}{subsection}
\tableofcontents
\hypersetup{linkcolor=magenta}
\clearpage
\section{Guidance Bias and Bias Amplification}
We extend the amplification analysis of \cref{sec:discussion}.
A Colored MNIST toy isolates the effect; we then show how it appears in text-to-image generation across models, occupations, and prompts.

\subsection{Toy example: Colored MNIST}
\label{app:toy}
We use the toy example to observe the whole scale range: the text-to-image sweeps drop their lowest scales, where image quality degrades, and that is where amplification concentrates.

\textbf{Setup.}
We construct a Colored MNIST dataset ($\mathcal{D}_{\mathrm{org}}$) where digits represent occupations and colors (red or green) represent gender, with the red ratio rising linearly from $20\%$ to $80\%$ across digits so that every class has a known color bias.
Using ADM~\citep{dhariwal2021diffusion}, we train a class-conditional diffusion model and time-dependent classifiers on $\mathcal{D}_{\mathrm{org}}$ and its reverse $\mathcal{D}_{\mathrm{opp}}$, which flips the red ratio of every digit, and measure the red ratio across guidance scales.

\textbf{Amplification concentrates at low guidance scales.}
\cref{fig:app_toy_colored_mnist}(a) tracks the group ratio across the full scale range.
First, guidance accounts for most of the amplification.
At $w{=}1$, where guidance is off, the generated ratios sit $2$ pp from the training ratios ($20$--$80\%$); the gap widens with the scale.
Second, about two thirds of that movement occurs between the two lowest scales.
This concentration is one reason why \cref{tab:table1_2_results} shows larger bias improvements on SD3 and FLUX.2, whose ranges reach $1.5$ and $1.0$, than on SD1.5, whose range starts at $2.5$.

\textbf{Amplification follows the model distribution relative to the marginal.}
\cref{fig:app_toy_colored_mnist}(b) plots $Q^{w}(\mathrm{red}\mid y)$ against its value without guidance $Q^{1}(\mathrm{red}\mid y)$ for each digit, where $Q^{w}(\mathrm{red}\mid y)$ is the red ratio generated for digit $y$ at scale $w$.
A logit fit with a high $R^2$ ($0.95$) shows that the amplification of a class is set by $Q^{1}(\mathrm{red}\mid y)$ alone, as it is on the text-to-image models ($R^2{=}\textcolor{black}{0.88}$--$\textcolor{black}{0.96}$, \cref{fig:app_asym_bias_amp_scatter}).
Its size follows from how far that value sits from the marginal $Q^{1}(\mathrm{red})$.
\textcolor{black}{The dataset marginal is $50\%$ by construction, whereas the fitted curve crosses $y{=}x$ at $45\%$.}

\textbf{Amplification takes its direction from the guidance function.}
\cref{fig:app_toy_colored_mnist}(c) and (d) show the same diffusion model amplifying in opposite directions under classifiers of opposite bias.
This is the more general statement: growth of the majority group is the special case in which the guidance bias and the model bias share the sign.
With a binary attribute we already find many classes that violate it (\cref{fig:app_sd15_mitigating_occ}), and with race, an attribute with four groups, it explains still less (\cref{app:race}).

\begin{figure}[H]
\definecolor{toyd0}{HTML}{1A9850}
\definecolor{toyd1}{HTML}{369860}
\definecolor{toyd2}{HTML}{529971}
\definecolor{toyd3}{HTML}{6F9981}
\definecolor{toyd4}{HTML}{8B9A92}
\definecolor{toyd5}{HTML}{A18E8D}
\definecolor{toyd6}{HTML}{AE7774}
\definecolor{toyd7}{HTML}{BC5F5A}
\definecolor{toyd8}{HTML}{CA4740}
\definecolor{toyd9}{HTML}{D73027}
  \centering
  {\footnotesize \raisebox{1.6pt}{\textcolor{toyd0}{\rule{6pt}{1.2pt}}}\,0 \raisebox{1.6pt}{\textcolor{toyd1}{\rule{6pt}{1.2pt}}}\,1 \raisebox{1.6pt}{\textcolor{toyd2}{\rule{6pt}{1.2pt}}}\,2 \raisebox{1.6pt}{\textcolor{toyd3}{\rule{6pt}{1.2pt}}}\,3 \raisebox{1.6pt}{\textcolor{toyd4}{\rule{6pt}{1.2pt}}}\,4 \raisebox{1.6pt}{\textcolor{toyd5}{\rule{6pt}{1.2pt}}}\,5 \raisebox{1.6pt}{\textcolor{toyd6}{\rule{6pt}{1.2pt}}}\,6 \raisebox{1.6pt}{\textcolor{toyd7}{\rule{6pt}{1.2pt}}}\,7 \raisebox{1.6pt}{\textcolor{toyd8}{\rule{6pt}{1.2pt}}}\,8 \raisebox{1.6pt}{\textcolor{toyd9}{\rule{6pt}{1.2pt}}}\,9}\\[2pt]
  \begin{minipage}[b]{0.24\linewidth}
    \centering
    \begin{overpic}[width=\linewidth,percent]{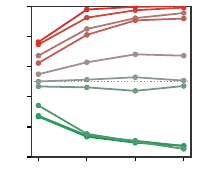}
      \put(19.47,7.30){\makebox[0pt][c]{\tiny 1.0}} \put(44.03,7.30){\makebox[0pt][c]{\tiny 3.0}} \put(68.60,7.30){\makebox[0pt][c]{\tiny 5.0}} \put(93.16,7.30){\makebox[0pt][c]{\tiny 7.0}}
      \put(13.18,13.53){\makebox[0pt][r]{\tiny 0}} \put(13.18,28.89){\makebox[0pt][r]{\tiny 20}} \put(13.18,44.24){\makebox[0pt][r]{\tiny 40}} \put(13.18,59.60){\makebox[0pt][r]{\tiny 60}} \put(13.18,74.95){\makebox[0pt][r]{\tiny 80}} \put(13.18,90.31){\makebox[0pt][r]{\tiny 100}}
      \put(56.31,0.4){\makebox[0pt][c]{\tiny Guidance scale $w$}}
      \put(-0.9,53.12){\rotatebox{90}{\makebox[0pt][c]{\tiny Red ratio (\%)}}}
    \end{overpic}\\[-1pt]
    {\footnotesize (a) CFG}
  \end{minipage}\hfill
  \begin{minipage}[b]{0.24\linewidth}
    \centering
    \begin{overpic}[width=\linewidth,percent]{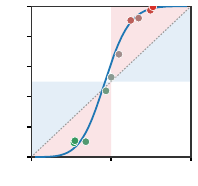}
      \put(15.78,7.30){\makebox[0pt][c]{\tiny 0}} \put(56.31,7.30){\makebox[0pt][c]{\tiny 50}} \put(96.84,7.30){\makebox[0pt][c]{\tiny 100}}
      \put(13.18,13.53){\makebox[0pt][r]{\tiny 0}} \put(13.18,28.89){\makebox[0pt][r]{\tiny 20}} \put(13.18,44.24){\makebox[0pt][r]{\tiny 40}} \put(13.18,59.60){\makebox[0pt][r]{\tiny 60}} \put(13.18,74.95){\makebox[0pt][r]{\tiny 80}} \put(13.18,90.31){\makebox[0pt][r]{\tiny 100}}
      \put(56.31,0.4){\makebox[0pt][c]{\tiny Red ratio (\%, $w{=}1.0$)}}
      \put(-0.9,53.12){\rotatebox{90}{\makebox[0pt][c]{\tiny Red ratio (\%, $w{=}5.0$)}}}
      \put(17.78,85.01){\tiny $R^2{=}0.950$}
    \end{overpic}\\[-1pt]
    {\footnotesize (b) CFG: $1.0\!\to\!5.0$}
  \end{minipage}\hfill
  \begin{minipage}[b]{0.24\linewidth}
    \centering
    \begin{overpic}[width=\linewidth,percent]{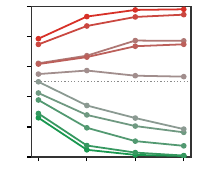}
      \put(19.47,7.30){\makebox[0pt][c]{\tiny 1.0}} \put(44.03,7.30){\makebox[0pt][c]{\tiny 3.0}} \put(68.60,7.30){\makebox[0pt][c]{\tiny 5.0}} \put(93.16,7.30){\makebox[0pt][c]{\tiny 7.0}}
      \put(13.18,13.53){\makebox[0pt][r]{\tiny 0}} \put(13.18,28.89){\makebox[0pt][r]{\tiny 20}} \put(13.18,44.24){\makebox[0pt][r]{\tiny 40}} \put(13.18,59.60){\makebox[0pt][r]{\tiny 60}} \put(13.18,74.95){\makebox[0pt][r]{\tiny 80}} \put(13.18,90.31){\makebox[0pt][r]{\tiny 100}}
      \put(56.31,0.4){\makebox[0pt][c]{\tiny Guidance scale $w$}}
      \put(-0.9,53.12){\rotatebox{90}{\makebox[0pt][c]{\tiny Red ratio (\%)}}}
    \end{overpic}\\[-1pt]
    {\footnotesize (c) CG aligned}
  \end{minipage}\hfill
  \begin{minipage}[b]{0.24\linewidth}
    \centering
    \begin{overpic}[width=\linewidth,percent]{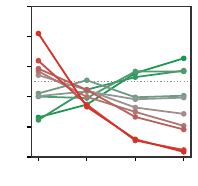}
      \put(19.47,7.30){\makebox[0pt][c]{\tiny 1.0}} \put(44.03,7.30){\makebox[0pt][c]{\tiny 3.0}} \put(68.60,7.30){\makebox[0pt][c]{\tiny 5.0}} \put(93.16,7.30){\makebox[0pt][c]{\tiny 7.0}}
      \put(13.18,13.53){\makebox[0pt][r]{\tiny 0}} \put(13.18,28.89){\makebox[0pt][r]{\tiny 20}} \put(13.18,44.24){\makebox[0pt][r]{\tiny 40}} \put(13.18,59.60){\makebox[0pt][r]{\tiny 60}} \put(13.18,74.95){\makebox[0pt][r]{\tiny 80}} \put(13.18,90.31){\makebox[0pt][r]{\tiny 100}}
      \put(56.31,0.4){\makebox[0pt][c]{\tiny Guidance scale $w$}}
      \put(-0.9,53.12){\rotatebox{90}{\makebox[0pt][c]{\tiny Red ratio (\%)}}}
    \end{overpic}\\[-1pt]
    {\footnotesize (d) CG opposed}
  \end{minipage}
  \caption[Colored MNIST toy example.]{\textbf{Colored MNIST toy example.}
  (a) Red ratio of each digit across guidance scales under CFG.
  (b) The same ratios at $w{=}1.0$ against $w{=}5.0$, with a logit fit and the diagonal $y{=}x$.
  (c, d) The sweep of (a) under CG, with the classifier trained on $\mathcal{D}_{\mathrm{org}}$ and on $\mathcal{D}_{\mathrm{opp}}$.}
  \label{fig:app_toy_colored_mnist}
\end{figure}

\newpage
\subsection{Gender}
\label{app:bias-amplification}
\label{app:minority-amplification}
\textcolor{black}{We give analysis details for \cref{sec:discussion} and more comparisons across models and prompt templates.}

\textbf{Setup.}
\textcolor{black}{We compare female ratios at low and high guidance scales across SD1.5, SD3, and FLUX.2, using the occupations and prompt templates in \cref{tab:occupations,tab:prompts}. We follow the evaluation protocol in \cref{app:experimental-details}, generating 100 images per occupation--template pair at each scale for each of three seeds. Each point represents an occupation--template pair, with female ratios averaged over the seeds. In \cref{fig:app_asym_bias_amp_scatter}, we use the 25 occupations with BLS statistics (\cref{tab:occupations}).}

\textbf{Amplification is not symmetric in T2I diffusion models.}
\cref{fig:app_asym_bias_amp_scatter} shows male- and female-associated occupations amplifying unequally, with the split forming left of $50\%$ on SD1.5 but right of it on SD3 and FLUX.2.
Earlier accounts attribute this asymmetry to the model's own gender bias~\citep{girrbach2025large, wu2024stable}; we explain it more clearly through the guidance bias.

\textcolor{black}{\textbf{Against expectation, SD3 and FLUX.2 are unconditionally female-biased.}
To confirm that the female bias of $\emptyset$ is the cause, we debias the null prompt by shifting its embedding toward male: the asymmetry recedes and the split moves back toward $50\%$ (\cref{fig:ablation_alpha-b,fig:ablation_alpha-d}).
Web corpora are reported to be male-biased, yet the marginal that the split identifies means the null branch of both models is female-biased.
Neither model's training data nor recipe is public, so we cannot take this further.
This marginal may result from reweighting gender composition in the data or objective.}

\textcolor{black}{On SD1.5, the split instead lies on the male side, possibly reflecting its male-biased LAION training data. Its noisier pattern likely stems from multiple fine-tuning stages and from our scales ($w\geq2.5$) missing the low range where amplification concentrates.}

\textcolor{black}{\textbf{Guidance turns this female marginal into stronger male bias, at the scales users choose.}
In \cref{fig:app_asym_scatter-b,fig:app_asym_scatter-c}, occupations with a similar female ratio at the lowest scale reach a lower ratio at the highest scale on FLUX.2 than on SD3.
A female-leaning null branch, possibly a by-product of rebalancing during training, thus backfires at the scales users prefer.
\method controls guidance bias, holding the group ratio steady across scales.}

\begingroup
\setlength{\intextsep}{4pt}
\begin{figure}[H]
  \centering
    \begin{overpic}[width=\linewidth,percent]{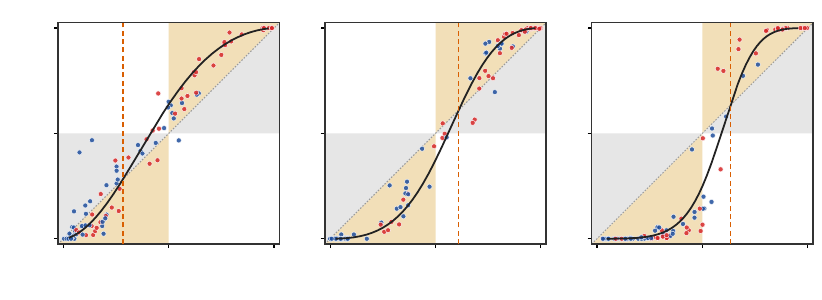}
      \put(7.71,5.18){\makebox[0pt][c]{\tiny 0}} \put(15.00,5.18){\makebox[0pt][c]{\tiny \textcolor{black}{28}}} \put(20.54,5.18){\makebox[0pt][c]{\tiny 50}} \put(33.37,5.18){\makebox[0pt][c]{\tiny 100}}
      \put(20.54,3.03){\makebox[0pt][c]{\scriptsize Female ratio (\%, $w{=}2.5$)}}
      \put(6.44,7.06){\makebox[0pt][r]{\tiny 0}} \put(6.44,19.88){\makebox[0pt][r]{\tiny 50}} \put(6.44,32.71){\makebox[0pt][r]{\tiny 100}}
      \put(3.03,20.29){\rotatebox{90}{\makebox[0pt][c]{\scriptsize Female ratio (\%, $w{=}12.5$)}}}
      \put(7.83,31.73){\tiny \textcolor{black}{$R^2{=}0.884$}}
      \put(20.54,0.63){\makebox[0pt][c]{\small (a) SD1.5}}
      \put(40.20,5.18){\makebox[0pt][c]{\tiny 0}} \put(53.03,5.18){\makebox[0pt][c]{\tiny 50}} \put(55.85,5.18){\makebox[0pt][c]{\tiny \textcolor{black}{61}}} \put(65.86,5.18){\makebox[0pt][c]{\tiny 100}}
      \put(53.03,3.03){\makebox[0pt][c]{\scriptsize Female ratio (\%, $w{=}1.5$)}}
      \put(36.53,20.29){\rotatebox{90}{\makebox[0pt][c]{\scriptsize Female ratio (\%, $w{=}7.5$)}}}
      \put(40.32,31.73){\tiny \textcolor{black}{$R^2{=}0.963$}}
      \put(53.03,0.63){\makebox[0pt][c]{\small (b) SD3}}
      \put(72.70,5.18){\makebox[0pt][c]{\tiny 0}} \put(85.52,5.18){\makebox[0pt][c]{\tiny 50}} \put(88.93,5.18){\makebox[0pt][c]{\tiny \textcolor{black}{63}}} \put(98.35,5.18){\makebox[0pt][c]{\tiny 100}}
      \put(85.52,3.03){\makebox[0pt][c]{\scriptsize Female ratio (\%, $w{=}1.0$)}}
      \put(69.02,20.29){\rotatebox{90}{\makebox[0pt][c]{\scriptsize Female ratio (\%, $w{=}5.0$)}}}
      \put(72.81,31.73){\tiny \textcolor{black}{$R^2{=}0.907$}}
      \put(85.52,0.63){\makebox[0pt][c]{\small (c) FLUX.2}}
      \put(7.07,34.76){\tiny Real-world majority:\,\textcolor[HTML]{D62728}{\raisebox{.2ex}{\scalebox{.7}{$\bullet$}}}\,Female\ \textcolor[HTML]{1F4E9C}{\raisebox{.2ex}{\scalebox{.7}{$\bullet$}}}\,Male}
      \put(98.99,34.76){\makebox[0pt][r]{\tiny Guidance bias:\,{\setlength{\fboxsep}{0pt}\setlength{\fboxrule}{0.2pt}\fcolorbox[HTML]{BBBBBB}{F2DFB8}{\rule{0pt}{3pt}\rule{3.6pt}{0pt}}}\,Amplify\ {\setlength{\fboxsep}{0pt}\setlength{\fboxrule}{0.2pt}\fcolorbox[HTML]{BBBBBB}{E6E6E6}{\rule{0pt}{3pt}\rule{3.6pt}{0pt}}}\,Mitigate}}
    \end{overpic}
  \caption{\textbf{Asymmetric bias amplification under standard CFG.}
  Each dot is an occupation--template pair, plotted by its female ratio at low ($x$) and high ($y$) guidance scales.
  \textcolor{black}{A dot in a sand-colored region was amplified as the guidance scale rose; one in a gray region was mitigated.}
  }
  \label{fig:app_asym_bias_amp_scatter}%
  \renewcommand{\thesubfigure}{(\alph{subfigure})}%
  \phantomsubcaption\label{fig:app_asym_scatter-a}%
  \phantomsubcaption\label{fig:app_asym_scatter-b}%
  \phantomsubcaption\label{fig:app_asym_scatter-c}%
\end{figure}

\endgroup

\newpage

\textbf{Amplification differs across the neutral templates of one occupation.}
In \cref{fig:app_prompt_ablation_2x4}, the four neutral templates of \cref{tab:prompts} give sharply different female ratios for the same occupation.
Where the low-scale ratio lies near the split of \cref{fig:app_asym_bias_amp_scatter}, a template can move an occupation across it, and the amplification follows the side it lands on.
The guidance bias therefore gives one account of the amplification seen on T2I models, whether it varies by model, by occupation, or by prompt.

\begin{figure}[H]
  \centering
    \begin{overpic}[width=\linewidth,percent]{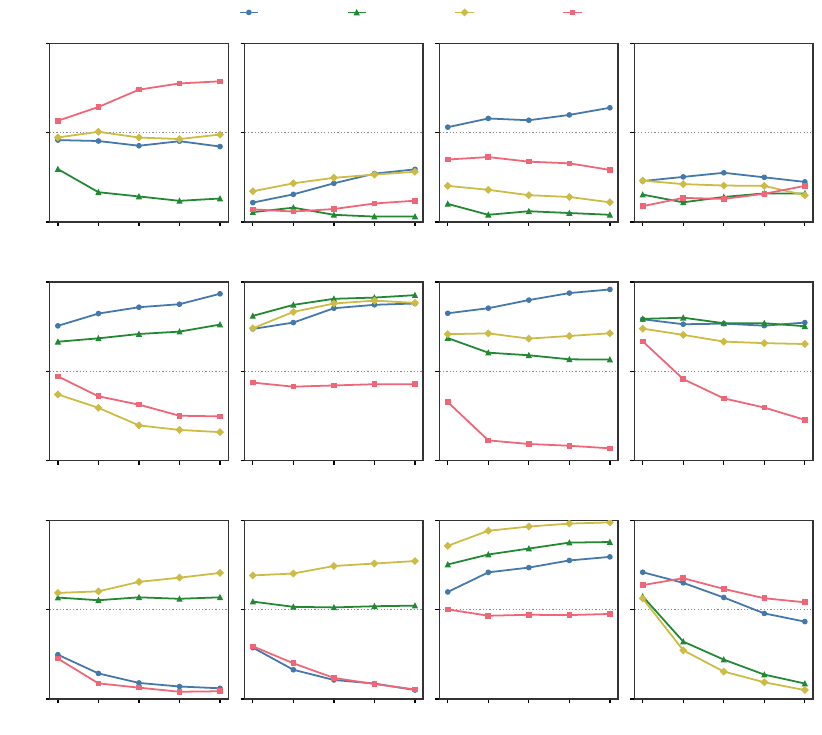}
      \put(16.92,84.72){\makebox[0pt][c]{\tiny cook}}
      \put(7.05,60.73){\makebox[0pt][c]{\tiny 2.5}} \put(11.98,60.73){\makebox[0pt][c]{\tiny 5.0}} \put(16.92,60.73){\makebox[0pt][c]{\tiny 7.5}} \put(21.85,60.73){\makebox[0pt][c]{\tiny 10.0}} \put(26.79,60.73){\makebox[0pt][c]{\tiny 12.5}}
      \put(16.92,58.59){\makebox[0pt][c]{\scriptsize Guidance scale $w$}}
      \put(40.66,84.72){\makebox[0pt][c]{\tiny supervisor}}
      \put(30.79,60.73){\makebox[0pt][c]{\tiny 2.5}} \put(35.72,60.73){\makebox[0pt][c]{\tiny 5.0}} \put(40.66,60.73){\makebox[0pt][c]{\tiny 7.5}} \put(45.59,60.73){\makebox[0pt][c]{\tiny 10.0}} \put(50.53,60.73){\makebox[0pt][c]{\tiny 12.5}}
      \put(40.66,58.59){\makebox[0pt][c]{\scriptsize Guidance scale $w$}}
      \put(64.39,84.72){\makebox[0pt][c]{\tiny entrepreneur}}
      \put(54.52,60.73){\makebox[0pt][c]{\tiny 2.5}} \put(59.46,60.73){\makebox[0pt][c]{\tiny 5.0}} \put(64.39,60.73){\makebox[0pt][c]{\tiny 7.5}} \put(69.33,60.73){\makebox[0pt][c]{\tiny 10.0}} \put(74.27,60.73){\makebox[0pt][c]{\tiny 12.5}}
      \put(64.39,58.59){\makebox[0pt][c]{\scriptsize Guidance scale $w$}}
      \put(88.13,84.72){\makebox[0pt][c]{\tiny politician}}
      \put(78.26,60.73){\makebox[0pt][c]{\tiny 2.5}} \put(83.20,60.73){\makebox[0pt][c]{\tiny 5.0}} \put(88.13,60.73){\makebox[0pt][c]{\tiny 7.5}} \put(93.07,60.73){\makebox[0pt][c]{\tiny 10.0}} \put(98.00,60.73){\makebox[0pt][c]{\tiny 12.5}}
      \put(88.13,58.59){\makebox[0pt][c]{\scriptsize Guidance scale $w$}}
      \put(5.43,61.97){\makebox[0pt][r]{\tiny 0\%}} \put(5.43,72.83){\makebox[0pt][r]{\tiny 50\%}} \put(5.43,83.69){\makebox[0pt][r]{\tiny 100\%}}
      \put(0.76,73.23){\rotatebox{90}{\makebox[0pt][c]{\scriptsize SD1.5}}}
      \put(16.92,55.68){\makebox[0pt][c]{\tiny doctor}}
      \put(7.05,31.69){\makebox[0pt][c]{\tiny 1.5}} \put(11.98,31.69){\makebox[0pt][c]{\tiny 3.0}} \put(16.92,31.69){\makebox[0pt][c]{\tiny 4.5}} \put(21.85,31.69){\makebox[0pt][c]{\tiny 6.0}} \put(26.79,31.69){\makebox[0pt][c]{\tiny 7.5}}
      \put(16.92,29.55){\makebox[0pt][c]{\scriptsize Guidance scale $w$}}
      \put(40.66,55.68){\makebox[0pt][c]{\tiny journalist}}
      \put(30.79,31.69){\makebox[0pt][c]{\tiny 1.5}} \put(35.72,31.69){\makebox[0pt][c]{\tiny 3.0}} \put(40.66,31.69){\makebox[0pt][c]{\tiny 4.5}} \put(45.59,31.69){\makebox[0pt][c]{\tiny 6.0}} \put(50.53,31.69){\makebox[0pt][c]{\tiny 7.5}}
      \put(40.66,29.55){\makebox[0pt][c]{\scriptsize Guidance scale $w$}}
      \put(64.39,55.68){\makebox[0pt][c]{\tiny painter}}
      \put(54.52,31.69){\makebox[0pt][c]{\tiny 1.5}} \put(59.46,31.69){\makebox[0pt][c]{\tiny 3.0}} \put(64.39,31.69){\makebox[0pt][c]{\tiny 4.5}} \put(69.33,31.69){\makebox[0pt][c]{\tiny 6.0}} \put(74.27,31.69){\makebox[0pt][c]{\tiny 7.5}}
      \put(64.39,29.55){\makebox[0pt][c]{\scriptsize Guidance scale $w$}}
      \put(88.13,55.68){\makebox[0pt][c]{\tiny student}}
      \put(78.26,31.69){\makebox[0pt][c]{\tiny 1.5}} \put(83.20,31.69){\makebox[0pt][c]{\tiny 3.0}} \put(88.13,31.69){\makebox[0pt][c]{\tiny 4.5}} \put(93.07,31.69){\makebox[0pt][c]{\tiny 6.0}} \put(98.00,31.69){\makebox[0pt][c]{\tiny 7.5}}
      \put(88.13,29.55){\makebox[0pt][c]{\scriptsize Guidance scale $w$}}
      \put(5.43,32.93){\makebox[0pt][r]{\tiny 0\%}} \put(5.43,43.79){\makebox[0pt][r]{\tiny 50\%}} \put(5.43,54.65){\makebox[0pt][r]{\tiny 100\%}}
      \put(0.76,44.19){\rotatebox{90}{\makebox[0pt][c]{\scriptsize SD3}}}
      \put(16.92,26.64){\makebox[0pt][c]{\tiny supervisor}}
      \put(7.05,2.65){\makebox[0pt][c]{\tiny 1.0}} \put(11.98,2.65){\makebox[0pt][c]{\tiny 2.0}} \put(16.92,2.65){\makebox[0pt][c]{\tiny 3.0}} \put(21.85,2.65){\makebox[0pt][c]{\tiny 4.0}} \put(26.79,2.65){\makebox[0pt][c]{\tiny 5.0}}
      \put(16.92,0.51){\makebox[0pt][c]{\scriptsize Guidance scale $w$}}
      \put(40.66,26.64){\makebox[0pt][c]{\tiny journalist}}
      \put(30.79,2.65){\makebox[0pt][c]{\tiny 1.0}} \put(35.72,2.65){\makebox[0pt][c]{\tiny 2.0}} \put(40.66,2.65){\makebox[0pt][c]{\tiny 3.0}} \put(45.59,2.65){\makebox[0pt][c]{\tiny 4.0}} \put(50.53,2.65){\makebox[0pt][c]{\tiny 5.0}}
      \put(40.66,0.51){\makebox[0pt][c]{\scriptsize Guidance scale $w$}}
      \put(64.39,26.64){\makebox[0pt][c]{\tiny assistant}}
      \put(54.52,2.65){\makebox[0pt][c]{\tiny 1.0}} \put(59.46,2.65){\makebox[0pt][c]{\tiny 2.0}} \put(64.39,2.65){\makebox[0pt][c]{\tiny 3.0}} \put(69.33,2.65){\makebox[0pt][c]{\tiny 4.0}} \put(74.27,2.65){\makebox[0pt][c]{\tiny 5.0}}
      \put(64.39,0.51){\makebox[0pt][c]{\scriptsize Guidance scale $w$}}
      \put(88.13,26.64){\makebox[0pt][c]{\tiny fashion designer}}
      \put(78.26,2.65){\makebox[0pt][c]{\tiny 1.0}} \put(83.20,2.65){\makebox[0pt][c]{\tiny 2.0}} \put(88.13,2.65){\makebox[0pt][c]{\tiny 3.0}} \put(93.07,2.65){\makebox[0pt][c]{\tiny 4.0}} \put(98.00,2.65){\makebox[0pt][c]{\tiny 5.0}}
      \put(88.13,0.51){\makebox[0pt][c]{\scriptsize Guidance scale $w$}}
      \put(5.43,3.89){\makebox[0pt][r]{\tiny 0\%}} \put(5.43,14.75){\makebox[0pt][r]{\tiny 50\%}} \put(5.43,25.61){\makebox[0pt][r]{\tiny 100\%}}
      \put(0.76,15.15){\rotatebox{90}{\makebox[0pt][c]{\scriptsize FLUX.2}}}
      \put(31.94,87.47){\tiny face of}
      \put(45.08,87.47){\tiny portrait}
      \put(58.21,87.47){\tiny smiling}
      \put(71.34,87.47){\tiny at work}
    \end{overpic}
  \caption{\textcolor{black}{\textbf{Prompt-template ablation for four occupations per model where templates disagree.}}
  }
  \label{fig:app_prompt_ablation_2x4}
\end{figure}

\newpage
\subsection{Race}
\label{app:race}
We present additional analyses of the race--occupation bias in \cref{sec:exp_analysis}, per occupation on SD3.

\textbf{Setup.}
We keep the setting of \cref{sec:exp_t2i} with four race groups, White, Black, Asian, and \mbox{Indian}~\citep{karkkainenfairface}.
The bias range, our metric of guidance bias, is measured as $\max_{w,w'}\lVert Q^{w}(a\mid y)-Q^{w'}(a\mid y)\rVert_{\mathrm{TV}}$.

\textbf{\method on race.}
We apply \method in the same form as \cref{eq:ours} and change only the attribute direction, set per occupation from the groups with the largest rise and fall in group ratio under CFG.
This tests whether the null-prompt adjustment mitigates the guidance bias over four groups.

\textbf{Amplification takes diverse forms with multi-group attributes.}
\cref{fig:app_race_types}'s CFG row shows the majority group gaining at a minority group's expense (author), the reverse (singer, scientist), and two groups gaining at once (photographer).
\textcolor{black}{Whether a group is the majority or a minority under a given condition does not determine whether its group ratio increases or decreases: the group whose ratio falls is not always the smallest.}

\textbf{\method's null-prompt adjustment also works with multi-group attributes.}
\textcolor{black}{\cref{fig:app_race_types}'s \method row shows these shifts reduced, although residual variation remains, including the fall of White for singer.}
For photographer, the direction is set on Indian and White only, yet the rise of Asian flattens as well, which hints that one direction could capture the changes of several groups at once.

\begin{figure}[H]
  \centering
    \begin{overpic}[width=\linewidth,percent]{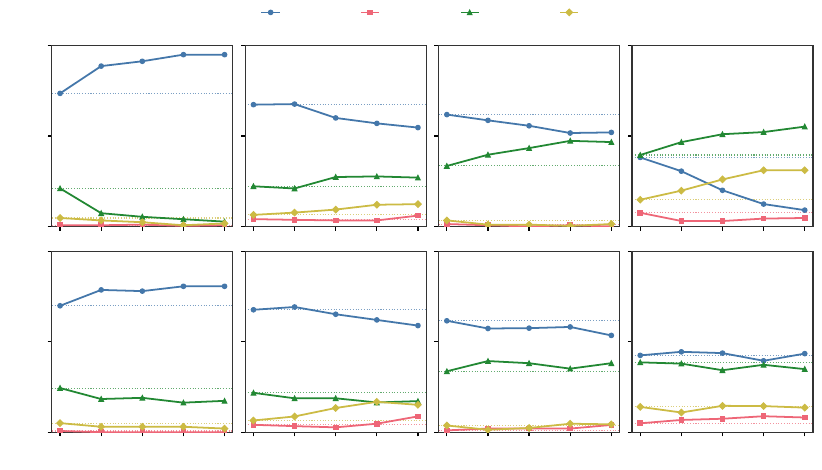}
      \put(17.33,52.27){\makebox[0pt][c]{\tiny author}}
      \put(40.88,52.27){\makebox[0pt][c]{\tiny singer}}
      \put(64.43,52.27){\makebox[0pt][c]{\tiny scientist}}
      \put(87.97,52.27){\makebox[0pt][c]{\tiny photographer}}
      \put(5.68,29.20){\makebox[0pt][r]{\tiny 0\%}} \put(5.68,40.22){\makebox[0pt][r]{\tiny 50\%}} \put(5.68,51.24){\makebox[0pt][r]{\tiny 100\%}}
      \put(0.88,40.62){\rotatebox{90}{\makebox[0pt][c]{\scriptsize CFG}}}
      \put(7.31,2.90){\makebox[0pt][c]{\tiny 1.5}} \put(12.32,2.90){\makebox[0pt][c]{\tiny 3.0}} \put(17.33,2.90){\makebox[0pt][c]{\tiny 4.5}} \put(22.34,2.90){\makebox[0pt][c]{\tiny 6.0}} \put(27.34,2.90){\makebox[0pt][c]{\tiny 7.5}}
      \put(30.86,2.90){\makebox[0pt][c]{\tiny 1.5}} \put(35.87,2.90){\makebox[0pt][c]{\tiny 3.0}} \put(40.88,2.90){\makebox[0pt][c]{\tiny 4.5}} \put(45.88,2.90){\makebox[0pt][c]{\tiny 6.0}} \put(50.89,2.90){\makebox[0pt][c]{\tiny 7.5}}
      \put(54.41,2.90){\makebox[0pt][c]{\tiny 1.5}} \put(59.42,2.90){\makebox[0pt][c]{\tiny 3.0}} \put(64.43,2.90){\makebox[0pt][c]{\tiny 4.5}} \put(69.43,2.90){\makebox[0pt][c]{\tiny 6.0}} \put(74.44,2.90){\makebox[0pt][c]{\tiny 7.5}}
      \put(77.96,2.90){\makebox[0pt][c]{\tiny 1.5}} \put(82.97,2.90){\makebox[0pt][c]{\tiny 3.0}} \put(87.97,2.90){\makebox[0pt][c]{\tiny 4.5}} \put(92.98,2.90){\makebox[0pt][c]{\tiny 6.0}} \put(97.99,2.90){\makebox[0pt][c]{\tiny 7.5}}
      \put(5.68,4.14){\makebox[0pt][r]{\tiny 0\%}} \put(5.68,15.16){\makebox[0pt][r]{\tiny 50\%}} \put(5.68,26.17){\makebox[0pt][r]{\tiny 100\%}}
      \put(0.88,15.56){\rotatebox{90}{\makebox[0pt][c]{\scriptsize \method}}}
      \put(34.60,55.28){\tiny White}
      \put(46.72,55.28){\tiny Black}
      \put(58.84,55.28){\tiny Asian}
      \put(70.96,55.28){\tiny Indian}
      \put(17.33,0.20){\makebox[0pt][c]{\scriptsize Guidance scale $w$}}
      \put(40.88,0.20){\makebox[0pt][c]{\scriptsize Guidance scale $w$}}
      \put(64.43,0.20){\makebox[0pt][c]{\scriptsize Guidance scale $w$}}
      \put(87.97,0.20){\makebox[0pt][c]{\scriptsize Guidance scale $w$}}
    \end{overpic}\par\vspace{5pt}
  \caption{\textbf{Race group ratios across guidance scales, CFG vs.\ \method.} \textcolor{black}{Curves average the three evaluation seeds in \cref{tab:race_tv}.}}
  \label{fig:app_race_types}
\end{figure}

\newpage
\begingroup
\setlength{\intextsep}{4pt}
\setlength{\textfloatsep}{6pt}
\section{Extending Debiasing Methods Across the Guidance Scale}
\label{app:debias-extension-failures}

We extend UCE~\citep{gandikota2024unified} and FT~\citep{shenfinetuning} to generation across guidance scales: one debiases each scale with its own model (per-scale)\textcolor{black}{;} the other debiases the whole range with one model, drawing the scale uniformly over the sweep (random).

\textbf{UCE across scales.}
In UCE the guidance scale enters where the edit is decided: the loop generates validation images, measures their group ratio, and moves the cross-attention projections by its gap to the target.
Multi (random) redraws that scale for every concept at every iteration, leaving everything else unchanged.
We choose each edit's iteration count by the lowest bias at its target scale.

\textbf{FT across scales.}
In FT the guidance scale enters where the loss is computed: each step generates images with it, classifies their attribute, and backpropagates the gap to the target through the sampler.
Multi (random) randomizes the scale at the batch level, drawing it once for every image in the step, and leaves everything else unchanged.
\textcolor{black}{We} choose each checkpoint by the lowest average bias at its scale \textcolor{black}{($6$k steps; $w{=}7.5$: the published $9.8$k model, $w{=}12.5$: $10$k)}.

\textbf{Multi (per-scale) fails because its models disagree across scales.}
\cref{fig:app_perscale_zigzag} shows that fitting a model per scale lowers the average bias but does not control the guidance bias.
Each checkpoint has its own direction and degree of bias, so the ratio changes direction from one scale to the next.

\textbf{Multi (random) succeeds in part, depending on the level of weight tuning.}
UCE moves the cross-attention projections along a fixed axis, so randomizing the scale changes only its direction.
FT tunes LoRA weights, which can move in many directions.
FT flattens the curve, cutting the range from $19.3$ to $12.0$, at the cost of a higher average bias ($14.1$ to $16.5$).
UCE improves neither ($18.0$ to $19.1$ average, $13.2$ to $13.7$ range).
Neither reaches the range \method holds ($9.4$ for FT, $7.6$ for UCE).

\begin{table}[H]
  \centering
  \scriptsize
  \setlength{\tabcolsep}{2pt}
  \caption{\textbf{\cref{tab:debiased_results}, with the average bias at each guidance scale.}}
  \label{tab:app_debias_perw}
  \begin{tabular}{l l l ccccc cc>{\columncolor{targetshade}}c}
    \toprule
    & & & \multicolumn{5}{c}{Guidance scale $w$} & \multicolumn{3}{c}{Bias (\%) $\downarrow$} \\
    \cmidrule(lr){4-8} \cmidrule(lr){9-11}
    Method & Scales & Guidance & $2.5$ & $5.0$ & $7.5$ & $10.0$ & $12.5$ & Avg. & Worst & Range \\
    \midrule
    \multirow{4}{*}{UCE}
      & Single            & CFG            & \pmvb{15.0}{0.6}\phantom{$^{*}$} & \pmv{16.7}{0.4} & \pmvb{18.1}{0.6} & \pmv{19.6}{0.8}\phantom{$^{*}$} & \pmv{20.6}{1.3} & \pmv{18.0}{0.7} & \pmv{23.4}{1.4} & \pmv{13.2}{0.7} \\
      & Multi (random)    & CFG            & \pmv{16.2}{0.3}\phantom{$^{*}$} & \pmv{18.6}{0.6} & \pmv{19.5}{0.9} & \pmv{20.3}{1.5}\phantom{$^{*}$} & \pmv{21.2}{0.9} & \pmv{19.1}{0.8} & \pmv{24.2}{1.2} & \pmv{13.7}{1.1} \\
      & Multi (per-scale) & CFG            & \pmvb{15.0}{0.6}$^{*}$ & \pmvb{15.3}{0.9} & \pmvb{18.1}{0.6} & \pmv{19.6}{0.8}$^{*}$ & \pmv{20.7}{0.9} & \pmvb{17.7}{0.6} & \pmv{29.4}{0.7} & \pmv{30.6}{0.4} \\
      & Single            & \textbf{\method (ours)} & \pmv{19.3}{0.6}\phantom{$^{*}$} & \pmv{18.2}{0.3} & \pmv{18.1}{0.6} & \pmvb{18.4}{0.8}\phantom{$^{*}$} & \pmvb{18.7}{1.2} & \pmv{18.5}{0.6} & \pmvb{22.2}{1.1} & \pmvb{\phantom{0}7.6}{0.7} \\
    \arrayrulecolor{black!35}\midrule[0.4pt]\arrayrulecolor{black}
    \multirow{4}{*}{FT}
      & Single            & CFG            & \pmv{13.3}{0.4}\phantom{$^{*}$} & \pmv{12.8}{0.4} & \pmvb{13.7}{0.5} & \pmv{14.7}{0.4}\phantom{$^{*}$} & \pmv{16.2}{0.3} & \pmv{14.1}{0.3} & \pmv{21.3}{0.6} & \pmv{19.3}{0.8} \\
      & Multi (random)    & CFG            & \pmv{13.1}{0.7}\phantom{$^{*}$} & \pmv{16.8}{0.5} & \pmv{17.7}{0.9} & \pmv{17.6}{1.1}\phantom{$^{*}$} & \pmv{17.4}{1.2} & \pmv{16.5}{0.9} & \pmv{21.3}{0.9} & \pmv{12.0}{0.6} \\
      & Multi (per-scale) & CFG            & \pmvb{10.9}{1.3}\phantom{$^{*}$} & \pmvb{10.8}{0.7} & \pmvb{13.7}{0.5} & \pmv{15.0}{0.3}\phantom{$^{*}$} & \pmv{17.7}{0.6} & \pmvb{13.6}{0.4} & \pmv{22.0}{0.6} & \pmv{17.9}{0.7} \\
      & Single            & \textbf{\method (ours)} & \pmv{13.7}{0.6}\phantom{$^{*}$} & \pmv{13.8}{0.2} & \pmvb{13.7}{0.5} & \pmvb{13.8}{0.4}\phantom{$^{*}$} & \pmvb{14.1}{0.4} & \pmv{13.8}{0.2} & \pmvb{18.2}{0.4} & \pmvb{\phantom{0}9.4}{0.4} \\
    \bottomrule
  \end{tabular}
  \par\smallskip
  \parbox{\linewidth}{\footnotesize\raggedright
    \textcolor{black}{$^{*}$UCE edit fails at that scale (avg.\ bias $23.6$ at $w{=}2.5$, $28.2$ at $w{=}10.0$); the $w{=}7.5$ edit is used.}}
\end{table}

\begin{figure}[H]
  \vspace{0pt}
  \centering
    \begin{overpic}[width=\linewidth,percent]{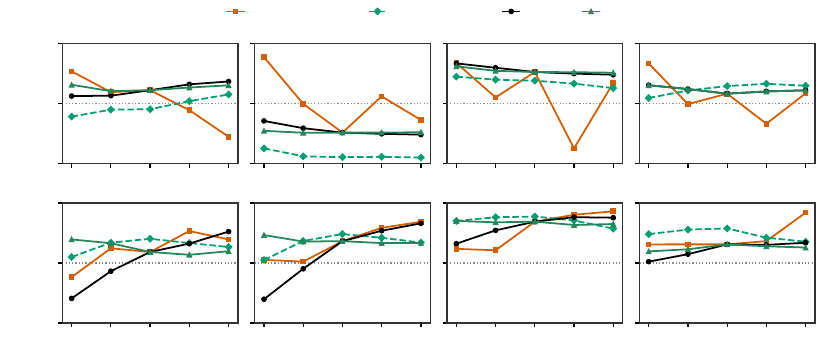}
      \put(18.28,38.76){\makebox[0pt][c]{\tiny technician}}
      \put(8.72,21.84){\makebox[0pt][c]{\tiny 2.5}} \put(13.50,21.84){\makebox[0pt][c]{\tiny 5}} \put(18.28,21.84){\makebox[0pt][c]{\tiny 7.5}} \put(23.05,21.84){\makebox[0pt][c]{\tiny 10}} \put(27.83,21.84){\makebox[0pt][c]{\tiny 12.5}}
      \put(41.70,38.76){\makebox[0pt][c]{\tiny intern}}
      \put(32.14,21.84){\makebox[0pt][c]{\tiny 2.5}} \put(36.92,21.84){\makebox[0pt][c]{\tiny 5}} \put(41.70,21.84){\makebox[0pt][c]{\tiny 7.5}} \put(46.48,21.84){\makebox[0pt][c]{\tiny 10}} \put(51.25,21.84){\makebox[0pt][c]{\tiny 12.5}}
      \put(65.12,38.76){\makebox[0pt][c]{\tiny fashion designer}}
      \put(55.57,21.84){\makebox[0pt][c]{\tiny 2.5}} \put(60.34,21.84){\makebox[0pt][c]{\tiny 5}} \put(65.12,21.84){\makebox[0pt][c]{\tiny 7.5}} \put(69.90,21.84){\makebox[0pt][c]{\tiny 10}} \put(74.67,21.84){\makebox[0pt][c]{\tiny 12.5}}
      \put(88.54,38.76){\makebox[0pt][c]{\tiny athlete}}
      \put(78.99,21.84){\makebox[0pt][c]{\tiny 2.5}} \put(83.76,21.84){\makebox[0pt][c]{\tiny 5}} \put(88.54,21.84){\makebox[0pt][c]{\tiny 7.5}} \put(93.32,21.84){\makebox[0pt][c]{\tiny 10}} \put(98.10,21.84){\makebox[0pt][c]{\tiny 12.5}}
      \put(6.94,23.08){\makebox[0pt][r]{\tiny 0}} \put(6.94,30.40){\makebox[0pt][r]{\tiny 50}} \put(6.94,37.73){\makebox[0pt][r]{\tiny 100}}
      \put(1.14,30.81){\rotatebox{90}{\makebox[0pt][c]{\scriptsize UCE}}}
      \put(18.28,19.32){\makebox[0pt][c]{\tiny pilot}}
      \put(8.72,2.40){\makebox[0pt][c]{\tiny 2.5}} \put(13.50,2.40){\makebox[0pt][c]{\tiny 5}} \put(18.28,2.40){\makebox[0pt][c]{\tiny 7.5}} \put(23.05,2.40){\makebox[0pt][c]{\tiny 10}} \put(27.83,2.40){\makebox[0pt][c]{\tiny 12.5}}
      \put(18.28,0.20){\makebox[0pt][c]{\scriptsize Guidance scale $w$}}
      \put(41.70,19.32){\makebox[0pt][c]{\tiny firefighter}}
      \put(32.14,2.40){\makebox[0pt][c]{\tiny 2.5}} \put(36.92,2.40){\makebox[0pt][c]{\tiny 5}} \put(41.70,2.40){\makebox[0pt][c]{\tiny 7.5}} \put(46.48,2.40){\makebox[0pt][c]{\tiny 10}} \put(51.25,2.40){\makebox[0pt][c]{\tiny 12.5}}
      \put(41.70,0.20){\makebox[0pt][c]{\scriptsize Guidance scale $w$}}
      \put(65.12,19.32){\makebox[0pt][c]{\tiny nurse}}
      \put(55.57,2.40){\makebox[0pt][c]{\tiny 2.5}} \put(60.34,2.40){\makebox[0pt][c]{\tiny 5}} \put(65.12,2.40){\makebox[0pt][c]{\tiny 7.5}} \put(69.90,2.40){\makebox[0pt][c]{\tiny 10}} \put(74.67,2.40){\makebox[0pt][c]{\tiny 12.5}}
      \put(65.12,0.20){\makebox[0pt][c]{\scriptsize Guidance scale $w$}}
      \put(88.54,19.32){\makebox[0pt][c]{\tiny cook}}
      \put(78.99,2.40){\makebox[0pt][c]{\tiny 2.5}} \put(83.76,2.40){\makebox[0pt][c]{\tiny 5}} \put(88.54,2.40){\makebox[0pt][c]{\tiny 7.5}} \put(93.32,2.40){\makebox[0pt][c]{\tiny 10}} \put(98.10,2.40){\makebox[0pt][c]{\tiny 12.5}}
      \put(88.54,0.20){\makebox[0pt][c]{\scriptsize Guidance scale $w$}}
      \put(6.94,3.64){\makebox[0pt][r]{\tiny 0}} \put(6.94,10.96){\makebox[0pt][r]{\tiny 50}} \put(6.94,18.28){\makebox[0pt][r]{\tiny 100}}
      \put(1.14,11.36){\rotatebox{90}{\makebox[0pt][c]{\scriptsize FT}}}
      \put(3.28,21.09){\rotatebox{90}{\makebox[0pt][c]{\scriptsize Female ratio (\%)}}}
      \put(30.30,41.64){\tiny Multi (per-scale)}
      \put(47.60,41.64){\tiny Multi (random)}
      \put(63.89,41.64){\tiny Single}
      \put(73.61,41.64){\tiny \method}
    \end{overpic}
  \par\vspace{2pt}
  \caption{\textbf{Extended debiasing still leaves large cross-scale variation.}
  \textbf{Multi (per-scale)} varies more across \(w\); \textbf{Multi (random)} has a similar or smaller bias range than \textbf{Single}.}
  \label{fig:app_perscale_zigzag}
\end{figure}

\endgroup

\newpage
\section{\method\ (Estimate) on Unseen Prompts}
\label{app:alpha-estimate}
\label{app:alpha-analysis}
\textcolor{black}{We evaluate \method (Estimate) on unseen prompts and report the estimator's success and failure cases; the procedures themselves are in \cref{app:alpha-procedure}.}

\textcolor{black}{\textbf{Setup.}
We evaluate \method (Estimate) on SD3 across 33 occupations and four prompt templates.
For each test occupation--template pair, we construct the lookup table after excluding its reference records.
The table maps female ratios at the lowest and highest guidance scales to $\hat{\alpha}$.}
\textcolor{black}{For each evaluation seed, we use 100 CFG images at each endpoint to compute expected bias metrics under five-image bootstrap sampling, evaluated exactly using binomial probabilities. We report the mean and standard deviation across three seeds.}

\textcolor{black}{We examine success and failure cases to explain the gap with \method (Search).}

{\color{black}

\textbf{Few-sample variability.} Sampling variability in the five-image probes causes the selected $\hat{\alpha}$ to fluctuate, limiting how consistently \method (Estimate) reproduces \method (Search).
Increasing the probe count could reduce this variability, at the cost of additional image generation.

\textbf{Transfer across seeds.} The effectiveness of an $\alpha$ also varies across seeds, so values inferred from reference prompts with similar endpoint ratios may not provide the same stabilization on evaluation seeds.
Constructing the lookup table from multiple reference seeds could make its selections less sensitive to individual realizations.

\textbf{Candidate coverage.} The lookup can only select among its predefined $\alpha$ candidates.
For example, the current positive candidates stop at $+50$, whereas \method (Search) selects $+75$ for writer.
Expanding the candidate range and refining its spacing could better capture the required correction strengths.

}
\begin{figure}[H]
  \centering
    \begin{overpic}[width=\linewidth,percent]{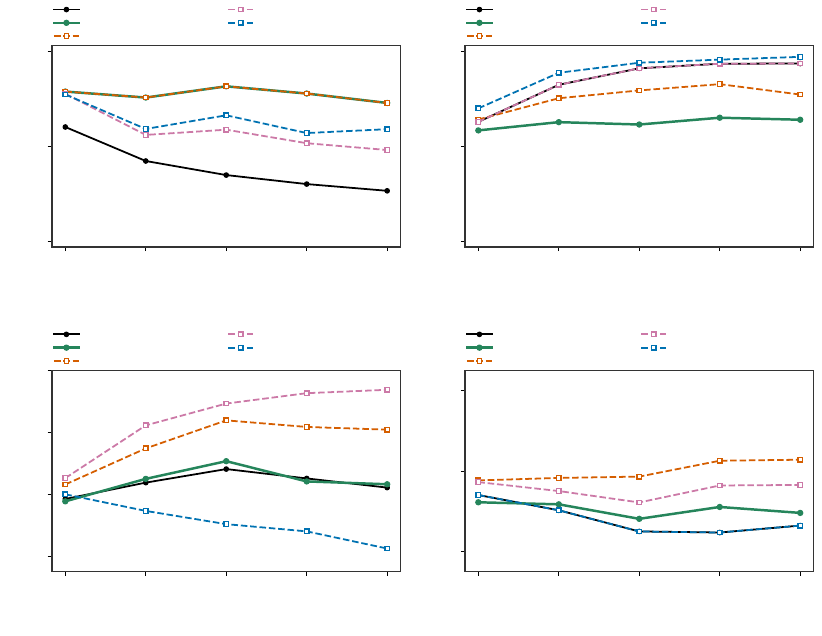}
      \put(7.94,45.58){\makebox[0pt][c]{\tiny 1.5}} \put(17.74,45.58){\makebox[0pt][c]{\tiny 3.0}} \put(27.53,45.58){\makebox[0pt][c]{\tiny 4.5}} \put(37.32,45.58){\makebox[0pt][c]{\tiny 6.0}} \put(47.11,45.58){\makebox[0pt][c]{\tiny 7.5}}
      \put(27.53,43.38){\makebox[0pt][c]{\scriptsize Guidance scale $w$}}
      \put(5.68,47.51){\makebox[0pt][r]{\tiny 0}} \put(5.68,59.07){\makebox[0pt][r]{\tiny 25}} \put(5.68,70.62){\makebox[0pt][r]{\tiny 50}}
      \put(2.78,59.47){\rotatebox{90}{\makebox[0pt][c]{\scriptsize Female ratio (\%)}}}
      \put(10.35,75.71){\tiny CFG: $\alpha=0$}
      \put(10.35,74.09){\tiny Search: $\alpha=-200$}
      \put(10.35,72.47){\tiny Est. 1: $\alpha=-200$ (41.8\%)}
      \put(31.57,75.71){\tiny Est. 2: $\alpha=-100$ (24.7\%)}
      \put(31.57,74.09){\tiny Est. 3: $\alpha=-150$ (17.5\%)}
      \put(27.53,40.03){\makebox[0pt][c]{\scriptsize (a) salesperson ($v_{1}$): Matching sign and magnitude}}
      \put(58.20,45.58){\makebox[0pt][c]{\tiny 1.5}} \put(67.99,45.58){\makebox[0pt][c]{\tiny 3.0}} \put(77.78,45.58){\makebox[0pt][c]{\tiny 4.5}} \put(87.57,45.58){\makebox[0pt][c]{\tiny 6.0}} \put(97.36,45.58){\makebox[0pt][c]{\tiny 7.5}}
      \put(77.78,43.38){\makebox[0pt][c]{\scriptsize Guidance scale $w$}}
      \put(55.93,47.51){\makebox[0pt][r]{\tiny 50}} \put(55.93,59.07){\makebox[0pt][r]{\tiny 75}} \put(55.93,70.62){\makebox[0pt][r]{\tiny 100}}
      \put(53.03,59.47){\rotatebox{90}{\makebox[0pt][c]{\scriptsize Female ratio (\%)}}}
      \put(60.61,75.71){\tiny CFG: $\alpha=0$}
      \put(60.61,74.09){\tiny Search: $\alpha=+75$}
      \put(60.61,72.47){\tiny Est. 1: $\alpha=+50$ (62.7\%)}
      \put(81.82,75.71){\tiny Est. 2: $\alpha=0$ (36.8\%)}
      \put(81.82,74.09){\tiny Est. 3: $\alpha=-50$ (0.5\%)}
      \put(77.78,40.03){\makebox[0pt][c]{\scriptsize (b) writer ($v_{1}$): Smaller magnitude}}
      \put(7.94,6.06){\makebox[0pt][c]{\tiny 1.5}} \put(17.74,6.06){\makebox[0pt][c]{\tiny 3.0}} \put(27.53,6.06){\makebox[0pt][c]{\tiny 4.5}} \put(37.32,6.06){\makebox[0pt][c]{\tiny 6.0}} \put(47.11,6.06){\makebox[0pt][c]{\tiny 7.5}}
      \put(27.53,3.86){\makebox[0pt][c]{\scriptsize Guidance scale $w$}}
      \put(5.68,9.18){\makebox[0pt][r]{\tiny 30}} \put(5.68,16.72){\makebox[0pt][r]{\tiny 50}} \put(5.68,24.26){\makebox[0pt][r]{\tiny 70}} \put(5.68,31.79){\makebox[0pt][r]{\tiny 90}}
      \put(2.78,19.95){\rotatebox{90}{\makebox[0pt][c]{\scriptsize Female ratio (\%)}}}
      \put(10.35,36.19){\tiny CFG: $\alpha=0$}
      \put(10.35,34.57){\tiny Search: $\alpha=+6.25$}
      \put(10.35,32.95){\tiny Est. 1: $\alpha=-50$ (37.1\%)}
      \put(31.57,36.19){\tiny Est. 2: $\alpha=-100$ (28.5\%)}
      \put(31.57,34.57){\tiny Est. 3: $\alpha=+50$ (18.1\%)}
      \put(27.53,0.51){\makebox[0pt][c]{\scriptsize (c) poet ($v_{3}$): Opposite sign}}
      \put(58.20,6.06){\makebox[0pt][c]{\tiny 1.5}} \put(67.99,6.06){\makebox[0pt][c]{\tiny 3.0}} \put(77.78,6.06){\makebox[0pt][c]{\tiny 4.5}} \put(87.57,6.06){\makebox[0pt][c]{\tiny 6.0}} \put(97.36,6.06){\makebox[0pt][c]{\tiny 7.5}}
      \put(77.78,3.86){\makebox[0pt][c]{\scriptsize Guidance scale $w$}}
      \put(55.93,9.75){\makebox[0pt][r]{\tiny 40}} \put(55.93,19.55){\makebox[0pt][r]{\tiny 60}} \put(55.93,29.34){\makebox[0pt][r]{\tiny 80}}
      \put(53.03,19.95){\rotatebox{90}{\makebox[0pt][c]{\scriptsize Female ratio (\%)}}}
      \put(60.61,36.19){\tiny CFG: $\alpha=0$}
      \put(60.61,34.57){\tiny Search: $\alpha=-20$}
      \put(60.61,32.95){\tiny Est. 1: $\alpha=-100$ (30.8\%)}
      \put(81.82,36.19){\tiny Est. 2: $\alpha=-50$ (28.0\%)}
      \put(81.82,34.57){\tiny Est. 3: $\alpha=0$ (14.2\%)}
      \put(77.78,0.51){\makebox[0pt][c]{\scriptsize (d) accountant ($v_{1}$): Larger magnitude}}
    \end{overpic}
  \caption{\textcolor{black}{\textbf{Success and failure cases of \method (Estimate).} Est. 1--3 show the three most probable $\hat{\alpha}$ choices under five-image sampling at each endpoint; percentages indicate selection probabilities. The top row shows successful or partially successful estimation, while the bottom row shows failures due to sign or magnitude errors.}}
  \label{fig:app_alpha_cases}
\end{figure}

\newpage
\begingroup
\setlength{\intextsep}{4pt}
\setlength{\textfloatsep}{6pt}
\section{Combining \method with Attribute-Based Methods}
\label{app:inference-based-methods}

\textcolor{black}{We test whether guidance bias also affects attribute-based methods, which steer generation toward an assigned attribute at inference time. We consider Fair Diffusion (FD)~\citep{friedrich2023fair}, which adds attribute guidance outside the model, and Weak Guidance (WG)~\citep{kim2025rethinking}, which injects the attribute into the model's conditioning; implementation details are in \cref{app:baselines-impl}.}

\textcolor{black}{\textbf{Setup.} We use SD1.5 with five guidance scales and three seeds, and assign a gender attribute to each sample. To apply \method, we shift $\alpha$ in \cref{eq:ours} toward the direction opposite to the assigned attribute, with $|\alpha|$ fixed to $10$.}
\begin{figure}[!htbp]
  \centering
  \begin{subfigure}[t]{0.40\linewidth}
    \centering
    \begin{overpic}[width=\linewidth,percent]{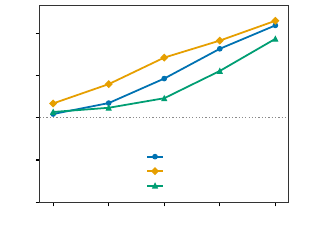}
      \put(16.18,6.12){\makebox[0pt][c]{\tiny 2.5}} \put(33.09,6.12){\makebox[0pt][c]{\tiny 5}} \put(50.00,6.12){\makebox[0pt][c]{\tiny 7.5}} \put(66.91,6.12){\makebox[0pt][c]{\tiny 10}} \put(83.82,6.12){\makebox[0pt][c]{\tiny 12.5}}
      \put(10.73,9.03){\makebox[0pt][r]{\tiny 20}} \put(10.73,21.88){\makebox[0pt][r]{\tiny 35}} \put(10.73,34.73){\makebox[0pt][r]{\tiny 50}} \put(10.73,47.58){\makebox[0pt][r]{\tiny 65}} \put(10.73,60.43){\makebox[0pt][r]{\tiny 80}}
      \put(50.00,0.57){\makebox[0pt][c]{\scriptsize Guidance scale $w$}}
      \put(2.97,40.09){\rotatebox{90}{\makebox[0pt][c]{\scriptsize Female ratio (\%)}}}
      \put(50.51,22.92){\tiny Nurse ($\Delta$=31.5\%)}
      \put(50.51,18.50){\tiny Intern ($\Delta$=29.3\%)}
      \put(50.51,14.08){\tiny Librarian ($\Delta$=26.0\%)}
    \end{overpic}
    \caption{Fair Diffusion \citep{friedrich2023fair}}
    \label{fig:app_inference_fd}
  \end{subfigure}
  \hspace{0.06\linewidth}
  \begin{subfigure}[t]{0.40\linewidth}
    \centering
    \begin{overpic}[width=\linewidth,percent]{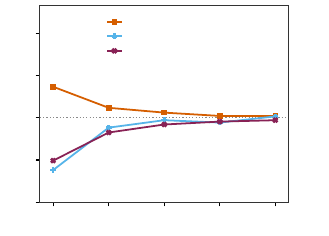}
      \put(16.18,6.12){\makebox[0pt][c]{\tiny 2.5}} \put(33.09,6.12){\makebox[0pt][c]{\tiny 5}} \put(50.00,6.12){\makebox[0pt][c]{\tiny 7.5}} \put(66.91,6.12){\makebox[0pt][c]{\tiny 10}} \put(83.82,6.12){\makebox[0pt][c]{\tiny 12.5}}
      \put(10.73,9.03){\makebox[0pt][r]{\tiny 20}} \put(10.73,21.88){\makebox[0pt][r]{\tiny 35}} \put(10.73,34.73){\makebox[0pt][r]{\tiny 50}} \put(10.73,47.58){\makebox[0pt][r]{\tiny 65}} \put(10.73,60.43){\makebox[0pt][r]{\tiny 80}}
      \put(50.00,0.57){\makebox[0pt][c]{\scriptsize Guidance scale $w$}}
      \put(2.97,40.09){\rotatebox{90}{\makebox[0pt][c]{\scriptsize Female ratio (\%)}}}
      \put(38.13,63.95){\tiny Fashion Designer ($\Delta$=10.4\%)}
      \put(38.13,59.53){\tiny Pilot ($\Delta$=19.1\%)}
      \put(38.13,55.11){\tiny Professor ($\Delta$=14.4\%)}
    \end{overpic}
    \caption{Weak Guidance \citep{kim2025rethinking}}
    \label{fig:app_inference_wg}
  \end{subfigure}
  \caption{\textbf{Bias of FD and WG across the guidance scale.}
  Female ratio curves across guidance scales and their deviation $\Delta$ from the target ratio for (a) FD and (b) WG.}
  \label{fig:app_inference_curves}
\end{figure}

\textcolor{black}{\textbf{FD weakens at high guidance scales, and WG at low scales.} \cref{fig:app_inference_curves} shows that FD fails to generate fair samples at high scales, while WG fails at low scales. As the guidance scale grows, text guidance dominates FD's attribute guidance, whereas WG's attribute alignment grows with the guidance scale. Neither method keeps the attribute aligned across scales.}

\begin{figure}[!htbp]
  \centering
  \begin{subfigure}{0.24\linewidth}
    \centering
    \begin{overpic}[width=\linewidth,percent]{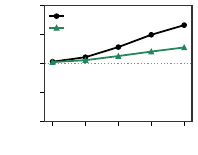}
      \put(26.63,10.21){\makebox[0pt][c]{\tiny 2.5}} \put(43.31,10.21){\makebox[0pt][c]{\tiny 5}} \put(60.00,10.21){\makebox[0pt][c]{\tiny 7.5}} \put(76.68,10.21){\makebox[0pt][c]{\tiny 10}} \put(93.37,10.21){\makebox[0pt][c]{\tiny 12.5}}
      \put(20.31,15.05){\makebox[0pt][r]{\tiny 0}} \put(20.31,29.83){\makebox[0pt][r]{\tiny 25}} \put(20.31,44.61){\makebox[0pt][r]{\tiny 50}} \put(20.31,59.40){\makebox[0pt][r]{\tiny 75}} \put(20.31,74.18){\makebox[0pt][r]{\tiny 100}}
      \put(60.00,0.95){\makebox[0pt][c]{\scriptsize Guidance scale $w$}}
      \put(6.31,46.40){\rotatebox{90}{\makebox[0pt][c]{\scriptsize Female ratio (\%)}}}
      \put(34.20,68.60){\tiny FD}
      \put(34.20,62.71){\tiny FD+\method}
    \end{overpic}
    \caption{\textcolor{black}{FD + nurse}}
  \end{subfigure}
  \hfill
  \begin{subfigure}{0.24\linewidth}
    \centering
    \begin{overpic}[width=\linewidth,percent]{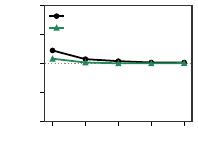}
      \put(26.63,10.21){\makebox[0pt][c]{\tiny 2.5}} \put(43.31,10.21){\makebox[0pt][c]{\tiny 5}} \put(60.00,10.21){\makebox[0pt][c]{\tiny 7.5}} \put(76.68,10.21){\makebox[0pt][c]{\tiny 10}} \put(93.37,10.21){\makebox[0pt][c]{\tiny 12.5}}
      \put(20.31,15.05){\makebox[0pt][r]{\tiny 0}} \put(20.31,29.83){\makebox[0pt][r]{\tiny 25}} \put(20.31,44.61){\makebox[0pt][r]{\tiny 50}} \put(20.31,59.40){\makebox[0pt][r]{\tiny 75}} \put(20.31,74.18){\makebox[0pt][r]{\tiny 100}}
      \put(60.00,0.95){\makebox[0pt][c]{\scriptsize Guidance scale $w$}}
      \put(6.31,46.40){\rotatebox{90}{\makebox[0pt][c]{\scriptsize Female ratio (\%)}}}
      \put(34.20,68.60){\tiny WG}
      \put(34.20,62.71){\tiny WG+\method}
    \end{overpic}
    \caption{\textcolor{black}{WG + fashion designer}}
  \end{subfigure}
  \hfill
  \begin{subfigure}{0.24\linewidth}
    \centering
    \begin{overpic}[width=\linewidth,percent]{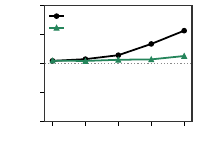}
      \put(26.63,10.21){\makebox[0pt][c]{\tiny 2.5}} \put(43.31,10.21){\makebox[0pt][c]{\tiny 5}} \put(60.00,10.21){\makebox[0pt][c]{\tiny 7.5}} \put(76.68,10.21){\makebox[0pt][c]{\tiny 10}} \put(93.37,10.21){\makebox[0pt][c]{\tiny 12.5}}
      \put(20.31,15.05){\makebox[0pt][r]{\tiny 0}} \put(20.31,29.83){\makebox[0pt][r]{\tiny 25}} \put(20.31,44.61){\makebox[0pt][r]{\tiny 50}} \put(20.31,59.40){\makebox[0pt][r]{\tiny 75}} \put(20.31,74.18){\makebox[0pt][r]{\tiny 100}}
      \put(60.00,0.95){\makebox[0pt][c]{\scriptsize Guidance scale $w$}}
      \put(6.31,46.40){\rotatebox{90}{\makebox[0pt][c]{\scriptsize Female ratio (\%)}}}
      \put(34.20,68.60){\tiny FD}
      \put(34.20,62.71){\tiny FD+\method}
    \end{overpic}
    \caption{FD + librarian}
  \end{subfigure}
  \hfill
  \begin{subfigure}{0.24\linewidth}
    \centering
    \begin{overpic}[width=\linewidth,percent]{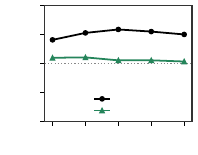}
      \put(26.63,10.21){\makebox[0pt][c]{\tiny 2.5}} \put(43.31,10.21){\makebox[0pt][c]{\tiny 5}} \put(60.00,10.21){\makebox[0pt][c]{\tiny 7.5}} \put(76.68,10.21){\makebox[0pt][c]{\tiny 10}} \put(93.37,10.21){\makebox[0pt][c]{\tiny 12.5}}
      \put(20.31,15.05){\makebox[0pt][r]{\tiny 0}} \put(20.31,29.83){\makebox[0pt][r]{\tiny 25}} \put(20.31,44.61){\makebox[0pt][r]{\tiny 50}} \put(20.31,59.40){\makebox[0pt][r]{\tiny 75}} \put(20.31,74.18){\makebox[0pt][r]{\tiny 100}}
      \put(60.00,0.95){\makebox[0pt][c]{\scriptsize Guidance scale $w$}}
      \put(6.31,46.40){\rotatebox{90}{\makebox[0pt][c]{\scriptsize Female ratio (\%)}}}
      \put(57.28,26.52){\tiny WG}
      \put(57.28,20.62){\tiny WG+\method}
    \end{overpic}
    \caption{WG + librarian}
  \end{subfigure}
  \caption{\textbf{Bias of FD+\method and WG+\method across the guidance scale.}
  Female ratio curves across guidance scales on \textcolor{black}{nurse (a), fashion designer (b), and librarian (c, d)}.}
  \label{fig:stayfair-inference}
\end{figure}

\textcolor{black}{\textbf{\method reduces bias across scales for both methods.} \cref{fig:stayfair-inference} and \cref{tab:app_fairgen_fd_wg_full} show that \method lowers the average bias across the full range ($5.4\%\to3.8\%$ for FD and $6.6\%\to2.9\%$ for WG) while keeping image quality comparable. Guidance bias thus persists in attribute-based methods, and \method mitigates it.}

\begin{table}[!htbp]
    \centering
    \caption{\textbf{Quantitative results for FD and WG.}
    Best in bold within each method block; target metric shaded; $\pm$: std over three seeds.}
    \label{tab:app_fairgen_fd_wg_full}
    \small
    \begin{tabular}{ll cc>{\columncolor{targetshade}}c ccc}
        \toprule
        & & \multicolumn{3}{c}{Bias (\%) $\downarrow$} & \multicolumn{3}{c}{Quality $\uparrow$} \\
        \cmidrule(lr){3-5} \cmidrule(lr){6-8}
        Model & Guidance & Avg. & Worst & Range & CLIP & Aesth. & Pick \\
        \midrule
        \multirow{2}{*}{FD~\citep{friedrich2023fair}}
            & CFG              & \pmv{5.4}{0.1} & \pmv{10.8}{0.3} & \pmv{10.2}{0.3} & \pmvb{27.33}{0.02} & \pmv{6.80}{0.01} & \pmvb{19.93}{0.01} \\
            & \textbf{\method (ours)}   & \pmvb{3.8}{0.2} & \pmvb{6.9}{0.1} & \pmvb{6.4}{0.2} & \pmv{27.20}{0.02} & \pmvb{6.83}{0.01} & \pmv{19.87}{0.01} \\
        \cmidrule(lr){1-8}
        \multirow{2}{*}{WG~\citep{kim2025rethinking}}
            & CFG              & \pmv{6.6}{0.2} & \pmv{12.0}{0.8} & \pmv{8.6}{1.0} & \pmvb{27.38}{0.02} & \pmv{6.59}{0.00} & \pmvb{19.94}{0.01} \\
            & \textbf{\method (ours)}   & \pmvb{2.9}{0.3} & \pmvb{7.5}{0.5} & \pmvb{6.7}{0.5} & \pmv{27.20}{0.01} & \pmvb{6.62}{0.00} & \pmv{19.89}{0.01} \\
        \bottomrule
    \end{tabular}
\end{table}

\endgroup

\clearpage
\section{Theoretical Analysis}
\label{app:proofs}
We prove that Strong Demographic Parity (SDP) preserves group ratios for the target distribution and, under an idealized Gaussian VE-SDE model, for the actual distribution.

\subsection{Target distribution}
\label{app:proofs_target}

\textbf{Setup.}
Recall the VE forward corruption $X_t=X_0+\sigma_t\epsilon$ with $\epsilon\sim\mathcal N(0,I)$, and let $A=g(X_0)\in\mathcal A$ be the group attribute. 
We view the guidance function as a positive potential
$f(\cdot,t):\mathcal X\times[0,T]\to\mathbb R_+$.
For each guidance scale $w$ considered, assume the normalizing constant $Z_t^w(y):=\mathbb E_{p_t^0(\cdot\mid y)}[f(X_t,t)^w]$
is finite. Then $f$ defines the target distribution
\[
    p_t^w(x\mid y)
    =
    \frac{p_t^0(x\mid y)f(x,t)^w}{Z_t^w(y)}.
\]
The same reweighting gives $p_t^w(x\mid a,y)\propto p_t^0(x\mid a,y)f(x,t)^w$. Hence Bayes' rule gives
\begin{equation}
    \frac{p_t^w(a\mid y)}{p_t^w(a'\mid y)}
    =\frac{p_t^0(a\mid y)}{p_t^0(a'\mid y)}
    \frac{\mathbb E[f(X_t,t)^w\mid A=a,y]}
         {\mathbb E[f(X_t,t)^w\mid A=a',y]}.
    \label{eq:target-ratio}
\end{equation}

\begin{lemma}[SDP implies group-wise moment equality]
\label{lem:sdp-moment}
Suppose $f$ satisfies SDP, \ie, for all $t\in[0,T]$ and all $a,a'\in\mathcal A$,
\begin{equation}
    \textcolor{black}{\mathrm{Law}}(f(X_t,t)\mid A=a,y)
    =\textcolor{black}{\mathrm{Law}}(f(X_t,t)\mid A=a',y).
    \label{eq:sdp-restate}
\end{equation}
Then for every $w\ge0$, 
\begin{equation}
    \mathbb E[f(X_t,t)^w\mid A=a,y]
    =\mathbb E[f(X_t,t)^w\mid A=a',y].
\end{equation}
\end{lemma}
\begin{proof}
By SDP, the random variables $f(X_t,t)\mid A=a,y$ and
$f(X_t,t)\mid A=a',y$ have the same distribution.
Applying the nonnegative measurable map $z\mapsto z^w$ gives equality of their $w$-th moments whenever these moments are finite, which is ensured by the normalizability assumption.
\end{proof}

\renewcommand{\theHtheorem}{app.\thetheorem}
\setcounter{theorem}{0}
\begin{theorem}[SDP preserves fairness in the target distribution]
\label{prop:sdp-target}

If $f$ satisfies SDP, then for every $w\ge0$, $t\in[0,T]$, and $a,a'\in\mathcal A$,
\begin{equation}
    \frac{p_t^w(a\mid y)}{p_t^w(a'\mid y)}
    =\frac{p_t^0(a\mid y)}{p_t^0(a'\mid y)}.
    \label{eq:prop-ratio}
\end{equation}
Consequently, if $p_t^0(a\mid y)=p_t^0(a'\mid y)$, then $p_t^w$ remains fair for all $w$.
\end{theorem}
\begin{proof}
By \cref{lem:sdp-moment}, the second factor in \cref{eq:target-ratio} equals one. This proves \cref{eq:prop-ratio}.
\end{proof}

\textbf{Discussion.}
\Cref{prop:sdp-target} controls the target distribution $p_t^w$. To transfer this statement to the actual distribution $Q^w$, we need to control the reverse dynamics. The next result does this in a Gaussian VE-SDE setting.

\subsection{Actual distribution}
\label{app:proofs_act}

\textbf{Interpretation and scope.}
The theorem formalizes when guidance-scale changes do not alter the final group
marginal of the sampler. The group attribute \(A\) is an endpoint attribute of
the clean sample \(X_0\), and the guidance statistic \(f_t\) is interpreted as a
clean-endpoint semantic score transported to noise level \(t\), in the same
spirit as classifier guidance. Thus \(f_t\) does not define a new group label
along the diffusion path; it represents how a semantic preference on clean
images is seen from the noisy state \(X_t\).

We state the result for the VE-SDE because its reverse dynamics are especially
simple and make the endpoint-reweighting structure explicit. The mechanism is
not specific to the VE choice: the same argument can be extended to sufficiently
regular nondegenerate Markov diffusions, provided the reverse-time SDE is well
defined and the backward weight
\[
    h_t^{(w)}(x)=\mathbb E[f_0(X_0)^w\mid X_t=x]
\]
is compatible with the guidance drift. The Gaussian/log-affine assumptions are
therefore not meant to model all neural guidance functions exactly; they provide
a clean sufficient setting in which the sampler-level group invariance can be
proved exactly.

\textbf{VE reverse dynamics.}
Let the forward VE-SDE be
\begin{equation}
    dX_t=g(t)dW_t,
    \qquad
    X_t=X_0+\sigma_t\epsilon,
    \qquad
    g(t)^2=\frac{d}{dt}\sigma_t^2.
\end{equation}
Write $s_t^0(x)=\nabla_x\log p_t^0(x)$ and set $\tau(s)=T-s$. The reference reverse VE-SDE is
\begin{equation}
    dZ_s=g(\tau)^2s_\tau^0(Z_s)ds+g(\tau)dB_s,
    \qquad Z_0\sim p_T^0,\quad Z_T\sim p_0^0.
\end{equation}
The guided reverse VE-SDE is
\begin{equation}
    dZ_s^w
    =g(\tau)^2\bigl(s_\tau^0(Z_s^w)+w\nabla\log f_\tau(Z_s^w)\bigr)ds
    +g(\tau)dB_s^w.
    \label{eq:ve-guided-reverse}
\end{equation}
Let $\mathbb P^w$ be the path law of \cref{eq:ve-guided-reverse} and write $Q^w(A=a\mid y):=\mathbb P^w(\textcolor{black}{g}(Z_T)=a)$.

\setcounter{theorem}{1}
\begin{theorem}[Formal statement of the informal theorem] %
\label{thm:main-app}
Assume $
    X_0\sim\mathcal N(\mu,\Sigma),    X_t=X_0+\sigma_t\epsilon,
$
and suppose the endpoint guidance potential is log-affine:
\begin{equation}
    \ell_y(x)=\exp(\beta_y^\top x+c_y),
    \qquad
    f_t(x)=\mathbb E[\ell_y(X_0)\mid X_t=x].
\end{equation}
Assume Novikov's condition for \cref{eq:ve-guided-reverse} and the ideal high-noise initialization
\begin{equation}
    \mathbb E[\ell_y(X_0)^w\mid Z_0]
    =\mathbb E[\ell_y(X_0)^w]
    \quad\text{a.s.}
    \label{eq:ve-high-noise}
\end{equation}
for the relevant guidance scales. If $f$ satisfies SDP, then $Q^w(A=a\mid y)=Q^{w_{\mathrm{ref}}}(A=a\mid y)$ for all $w,w_{\mathrm{ref}}\ge 0$ and $a\in\mathcal A$, and hence $\mathrm{Bias}_{\mathrm{G}}^{w}(a\mid y)=0$.
\end{theorem}

\begin{proof}
The Gaussian log-affine setting ensures that the endpoint moments
$\mathbb E[\ell_y(X_0)^w]$ are finite.
Let $\mathbb P=\mathbb P^0$ and define $h_t^{(w)}(x):=\mathbb E[\ell_y(X_0)^w\mid X_t=x]$.
For the Gaussian model, $X_0\mid X_t=x\sim\mathcal N(m_t(x),V_t)$ with
$m_t(x) = \mu+\Sigma(\Sigma+\sigma_t^2I)^{-1}(x-\mu)$ and
$V_t = (\Sigma^{-1}+\sigma_t^{-2}I)^{-1}$.
Thus $f_t(x) = \exp\!\bigl(c_y+\beta_y^\top m_t(x)+\tfrac12\beta_y^\top V_t\beta_y\bigr)$, and
\begin{equation}
    h_t^{(w)}(x)
    =\exp\!\left(wc_y+w\beta_y^\top m_t(x)+\frac{w^2}{2}\beta_y^\top V_t\beta_y\right)
    =C_t(w)f_t(x)^w,
    \label{eq:ve-h-f-relation}
\end{equation}
where $C_t(w)$ is independent of $x$. Differentiating \cref{eq:ve-h-f-relation},
\begin{equation}
    \nabla\log h_t^{(w)}(x)=w\nabla\log f_t(x).
    \label{eq:ve-gradient-compatibility}
\end{equation}

Set $\theta_s=g(\tau(s))\nabla\log f_{\tau(s)}(Z_s)$. By Girsanov's theorem,
\begin{equation}
    \frac{d\mathbb P^w}{d\mathbb P}\bigg|_{\mathcal F_T}
    =\exp\!\left(
        w\int_0^T\theta_s^\top dB_s
        -\frac{w^2}{2}\int_0^T\|\theta_s\|^2ds
    \right).
    \label{eq:ve-girsanov}
\end{equation}
By It\^o's formula applied to $\log h_{\tau(s)}^{(w)}(Z_s)$ and by \cref{eq:ve-gradient-compatibility}, the right-hand side of \cref{eq:ve-girsanov} equals $h_0^{(w)}(Z_T) / h_T^{(w)}(Z_0)$.
Since $Z_T=X_0$ and \cref{eq:ve-high-noise} holds,
\begin{equation}
    \frac{d\mathbb P^w}{d\mathbb P}
    =\frac{\ell_y(X_0)^w}{\mathbb E[\ell_y(X_0)^w]}.
    \label{eq:ve-endpoint-rn}
\end{equation}

Let $\pi_a=\mathbb P(A=a)$ and $M_a(w)=\mathbb E[\ell_y(X_0)^w\mid A=a,y]$. From \cref{eq:ve-endpoint-rn},
\begin{equation}
    Q^w(A=a\mid y)
    =\frac{\pi_a M_a(w)}{\sum_{a'\in\mathcal A}\pi_{a'}M_{a'}(w)}.
    \label{eq:ve-group-reweighting}
\end{equation}
Since $\sigma_0=0$, $f_0(X_0)=\ell_y(X_0)$. SDP at $t=0$ implies that $\ell_y(X_0)\mid A=a,y$ has the same law for every group $a$. Since the corresponding moments are finite, $M_a(w)$ is independent of $a$.
Substituting into \cref{eq:ve-group-reweighting} gives $Q^w(A=a\mid y)=\pi_a$ for all $w$. The same holds for $w_{\mathrm{ref}}$, proving the claim.
\end{proof}

\textbf{Remark.}
The Gaussian and log-affine assumptions are used to obtain \cref{eq:ve-gradient-compatibility}, which collapses the Girsanov path weight to the endpoint reweighting \cref{eq:ve-endpoint-rn}. For general neural guidance functions, SDP should be read as an approximate sufficient condition rather than an exact guarantee on the full path law.

\clearpage
\section{Experimental Details}
\label{app:experimental-details}

\subsection{Class-conditional generation}
\label{app:cg-impl}

\noindent \textbf{Training.}
We build our sampling pipeline on the official ADM-G codebase~\citep{dhariwal2021diffusion} and train the classifier on the CelebA dataset with images resized to \(64 \times 64\) resolution, using the \texttt{Male} attribute as the sensitive attribute. The diffusion process consists of \(1000\) steps with a cosine noise schedule, and learned sigmas are rescaled during training. We train all models using the Adam optimizer with a learning rate of \(3\times10^{-4}\) and weight decay \(1\times10^{-2}\), for a total of $800$k iterations with batch size \(128\). The learning rate is annealed during training.
\textcolor{black}{We run class-conditional experiments on two NVIDIA RTX 4090 GPUs.}

\noindent \textbf{Baselines.}
We apply RW~\citep{idrissi2022simple} and GDRO~\citep{sagawa2020distributionally} to the noisy classifier $f_\phi(x_t, t)$ used for CG.
RW is implemented by reweighting samples so that groups are sampled with equal probability.
For GDRO, we average the classifier loss over sampled timesteps within each group and apply the standard group-DRO update to the resulting group losses.

\subsection{Text-to-image generation}
\label{app:t2i-impl}

\textbf{Common settings.} \cref{tab:t2i_common_settings} lists the sampler, resolution, and guidance scales for each model. \textcolor{black}{We use three random seeds. All models use their publicly released} \texttt{diffusers} pipelines.
\textcolor{black}{We use two NVIDIA RTX PRO 6000 Blackwell GPUs and NVIDIA H100 GPUs.}
\begin{table}[!htbp]
\centering
\small
\setlength{\tabcolsep}{2.5pt}
\caption{Common generation settings for the text-to-image experiments.}
\label{tab:t2i_common_settings}
\begin{tabular*}{\textwidth}{@{\extracolsep{\fill}}llll@{}}
\toprule
 & SD1.5 & SD3-medium & FLUX.2 [klein] 9B Base \\
\midrule
Pipeline / scheduler & DPMSolver multistep & FlowMatchEulerDiscreteScheduler & Flux2Klein \\
Denoising steps & 30 & 28 & 28 \\
Resolution & $512 \times 512$ & $512 \times 512$ & $768 \times 768$ \\
Guidance scales $w$ & $\{2.5, 5.0, 7.5, 10.0, 12.5\}$ & $\{1.5, 3.0, 4.5, 6.0, 7.5\}$ & $\{1.0, 2.0, 3.0, 4.0, 5.0\}$ \\
\bottomrule
\end{tabular*}
\end{table}

\textbf{Occupations and prompts.}
Following \citep{seshadri2024bias}, we use 33 occupations and four prompt templates. \textcolor{black}{\cref{tab:occupations} lists the occupations and \cref{tab:prompts} the templates.}
\begin{table}[!htbp]
\centering
\begingroup
\color{black}
\caption{\textcolor{black}{Occupations used in our gender--occupation bias evaluation and their majority-group assignments. Share denotes the percentage of the assigned majority group in the 2025 BLS Current Population Survey (Table~11)~\citep{bls2025occupation}; -- marks occupations without a published BLS share.}}
\label{tab:occupations}
\small
\setlength{\tabcolsep}{3pt}
\begin{tabular*}{\textwidth}{@{\extracolsep{\fill}}llr@{\hspace{2.5em}}llr@{}}
\toprule
\textbf{Occupation} & \textbf{Majority} & \textbf{Share (\%)} & \textbf{Occupation} & \textbf{Majority} & \textbf{Share (\%)} \\
\cmidrule(r{0.6em}){1-3}\cmidrule(l{0.6em}){4-6}
librarian        & Female & 84.9 & teacher          & Female & 79.2 \\
ceo              & Male   & 67.0 & doctor           & Male   & 58.3 \\
nurse            & Female & 87.3 & pilot            & Male   & 93.0 \\
professor        & Female & 52.3 & writer           & Female & 66.8 \\
cook             & Male   & 56.3 & scientist        & Male   & 50.6 \\
firefighter      & Male   & 94.9 & author           & Female & 66.8 \\
painter          & Female & 59.9 & poet             & Female & 66.8 \\
salesperson      & Male   & 53.0 & mechanic         & Male   & 98.2 \\
engineer         & Male   & 80.7 & programmer       & Male   & 81.5 \\
chef             & Male   & 73.6 & accountant       & Female & 59.1 \\
journalist       & Male   & 58.9 & assistant        & Female & 91.9 \\
photographer     & Female & 50.7 & singer           & Male   & 65.9 \\
receptionist     & Female & 89.6 &                  &        &      \\
\cmidrule(r{0.6em}){1-3}\cmidrule(l{0.6em}){4-6}
technician       & --     &   -- & fashion designer & --     &   -- \\
intern           & --     &   -- & athlete          & --     &   -- \\
politician       & --     &   -- & supervisor       & --     &   -- \\
entrepreneur     & --     &   -- & student          & --     &   -- \\
\bottomrule
\end{tabular*}
\endgroup
\end{table}

\begin{table}[!htbp]
\centering
\caption{Prompt templates used in evaluation, each parameterized by \texttt{\{occupation\}}.}
\label{tab:prompts}
\begin{tabular}{cl}
\toprule
\# & Template \\
\midrule
1 & \texttt{A photo of the face of a/an \{occupation\}} \\
2 & \texttt{A portrait photo of a/an \{occupation\}} \\
3 & \texttt{A photo of a/an \{occupation\} smiling} \\
4 & \texttt{A photo of a/an \{occupation\}, at work} \\
\bottomrule
\end{tabular}
\end{table}

\textbf{Attribute classification.}
We predict gender with a zero-shot CLIP classifier~\citep{radford2021learning}
using the prompt pair
\texttt{[``a photo of a male person'',\allowbreak\,``a photo of a female person'']}.
Images are first filtered with Ultralytics YOLO11 (YOLO11n) person detection~\citep{ultralyticsYOLO11};
among the remaining images, only those with classifier confidence
$\geq 0.7$ are retained for ratio computation, while lower-confidence
images are discarded.

Race is labeled over four groups (White, Black, Asian, Indian), the classes into which the FairFace categories~\citep{karkkainenfairface} are consolidated.
We use the same CLIP classifier and YOLO person filter as for gender, with the prompts \texttt{"a photo of a White person"}, \texttt{"a photo of a Black person"}, \texttt{"a photo of an Asian"}, and \texttt{"a photo of an Indian"}, and assign each image to the highest-probability group without a confidence threshold.

\subsection{Baseline implementations}
\label{app:baselines-impl}

\textbf{Fine-Tuning Diffusion (FT).}
We use the pretrained checkpoint released by the authors. FT provides several checkpoints corresponding to different trained components of Stable Diffusion, including prefix tuning and LoRA fine-tuning. As many existing methods also consider fine-tuning the text encoder, we adopt the checkpoint ``finetune-text-encoder 09190215checkpoint-9800\_exportedtext\_encoder\_lora\_EMA.pth''.

\textbf{Unified Concept Editing (UCE).}
Following the official GitHub repository, we train UCE from scratch. Since UCE requires a predefined set of occupations for training, we adopt the same 33-occupation configuration as in our list for fair comparison with the other baselines.

\textbf{\textcolor{black}{Perturbed-Attention} Guidance (PAG).}
For the PAG baseline, we use the publicly released \texttt{Stable\allowbreak Diffusion\allowbreak PAG\allowbreak Pipeline} from 
\texttt{hyoungwoncho/\allowbreak sd\_perturbed\_\allowbreak attention\_guidance} on Hugging Face.
We set the perturbed layer to the first self-attention layer in the U-Net mid-block (\texttt{m0}) and sweep the PAG scale over \(\{1,2,3,4\}\) without CFG.

\textbf{Fair Diffusion (FD).}
We use [``female person'', ``male person''] as the editing prompts. The guidance direction (addition or subtraction) is randomly selected between the two concepts. The editing guidance scale is set to 4. The threshold, momentum scale, and momentum beta are set to 0.9, 0.5, and 0.6, respectively. Both concepts are assigned equal weights of 1.

\textbf{Weak Guidance (WG).}
Since the official code is not publicly available, we implement the method following the description in the original paper. We confirm that our implementation reproduces the reported performance on the ten occupations considered in the paper. During generation, the modified conditioning incorporating the attribute direction is used exclusively for the first 10\% of the denoising steps. For the remaining steps, the modified conditioning and the original conditioning are alternated.

\subsection{\method implementation}
\label{app:stayfair-impl}

\subsubsection*{Classifier guidance}
\textcolor{black}{\textbf{Training details.}} \method follows the baseline CG training framework, introducing an additional constraint term in the optimization objective after the first $5\%$ of training while keeping all other settings unchanged. We set $\lambda = 3$ in the CG objective of \cref{sec:fair_guidance}, selected from $\{1, 3, 5\}$.
\textcolor{black}{We convert the classifier logits to positive-class softmax probabilities for the WDP penalty. We compare the two sensitive groups within each condition class and sum the penalties over classes. For each comparison, we assign uniform mass to samples within each group and compute the transport plan using \texttt{ot.emd}. The matching cost is the absolute probability difference plus $0.2$ times the absolute logSNR difference; the final penalty sums only the probability differences weighted by this plan.}

\textcolor{black}{For each batch, we draw a center timestep using P2-based weights~\citep{choi2022perception} and sample the remaining timesteps within $\pm 50$ of the center using the same weights restricted to that interval, clipped to the diffusion timestep range. We resample training examples with replacement using inverse-frequency weights for the joint condition--sensitive groups, balancing their sampling probabilities without fixing group counts in each batch.}

\subsubsection*{Classifier-free guidance}
\label{app:alpha-procedure}
\begingroup\color{black}
\textbf{Null embedding construction.} We choose the null input and normalize the attribute direction according to each model's text encoder. Because normalization and encoder representations differ, $\alpha$ is not directly comparable across models. Across models, we observed greater degradation in image quality when the relative positions of special tokens were not preserved.
\begin{itemize}[leftmargin=*,itemsep=2pt,topsep=4pt]
\item \textbf{SD1.5.} We use the empty string as the null input and \textcolor{black}{normalize the direction per token.}

\item \textbf{SD3.} We normalize the full sequence from its two CLIP encoders and T5 using the Frobenius norm. We also shift the pooled CLIP embedding along a separately normalized direction, scaled by $\alpha/\sqrt{L}$, where $L=333$ is the sequence length ($77$ CLIP and $256$ T5 tokens). We build the null input by replacing the attribute word with the end-of-text token for the two CLIP encoders and the padding token for T5.

\item \textbf{FLUX.2.} We encode a space as the null input through Qwen's chat template and normalize the attribute direction over the full sequence.
\end{itemize}

\par\endgroup

\textbf{\method (Search).}
\textcolor{black}{We first evaluate candidate $\alpha$ values at three guidance scales ($w \in \{2.5, 7.5, 12.5\}$ for SD1.5, $\{1.5, 4.5, 7.5\}$ for SD3, and $\{1, 3, 5\}$ for FLUX.2) to locate a promising search interval. We then refine the search within this interval or extend it as needed, and evaluate selected candidates at the remaining scales. We compare their bias ranges across all evaluated scales, preferring smaller $|\alpha|$ when the differences are small.}

\textbf{\method (Estimate).}
\textcolor{black}{Prompts with similar initial group ratios exhibit similar amplification patterns (\cref{fig:ablation_alpha}) and therefore require similar $\alpha$ values to reduce guidance bias. We thus build the lookup table using a $k$-nearest-neighbor (kNN) approach based on the model's group distributions. For each pair of group ratios at the lowest and highest guidance scales, we find the 10 reference prompts with the closest ratio pairs. We then select, from $\alpha \in \{-200, -150, -100, -50, 0, +50\}$, the value that yields the largest average reduction in bias range relative to CFG, giving greater weight to closer neighbors. The selected values form a $6\times6$ table indexed by the number of female images (0--5) among the 5 probe images at each scale (10 images in total).}

\clearpage
\section{Additional Quantitative Results}
\label{app:additional-results}

\subsection{Additional per-occupation results}
\label{app:fixed-case}
\label{app:flux}

\textbf{Bias range over all occupations.}
\cref{fig:bias_range_2x2} reports the bias range of \cref{sec:exp_cfg_cond,sec:exp_cfg_fair} for each of the 33 occupations on the vanilla models SD1.5, SD3, and FLUX.2 and on the debiased models FT and UCE, comparing CFG with \method; occupations are sorted by the CFG value.
\begin{figure}[H]
  \centering
    \begin{overpic}[width=\linewidth,percent]{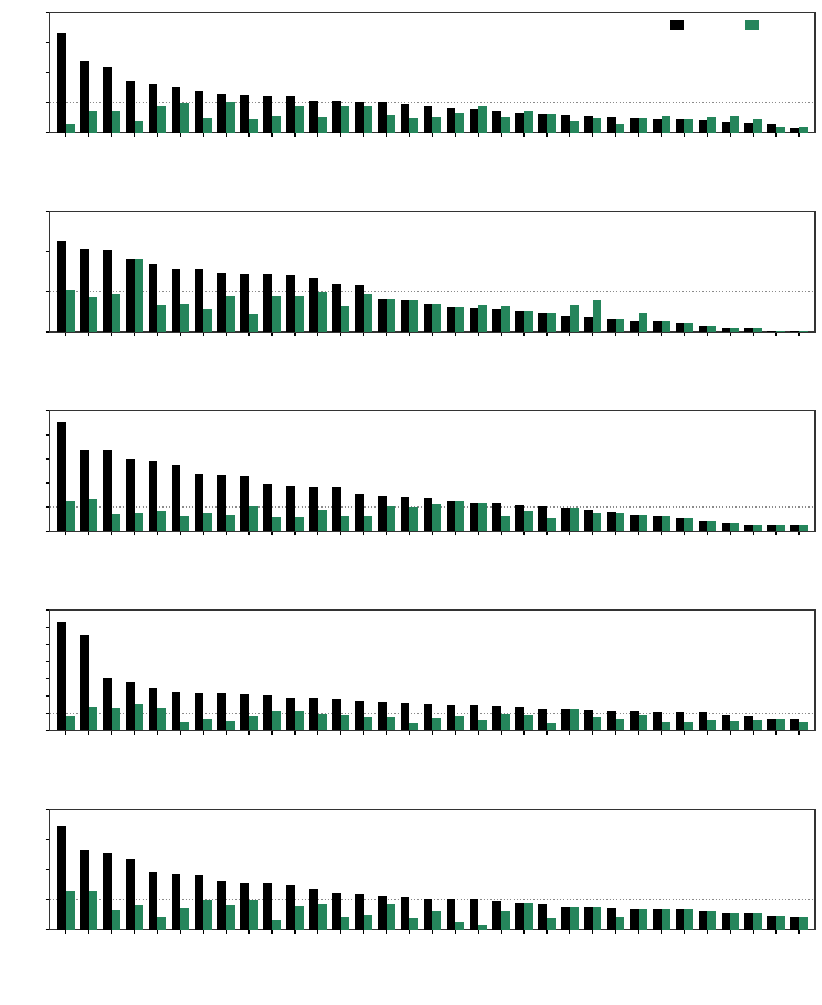}
      \put(4.44,86.45){\makebox[0pt][r]{\tiny 0\%}} \put(4.44,89.44){\makebox[0pt][r]{\tiny 10\%}} \put(4.44,92.44){\makebox[0pt][r]{\tiny 20\%}} \put(4.44,95.43){\makebox[0pt][r]{\tiny 30\%}} \put(4.44,98.43){\makebox[0pt][r]{\tiny 40\%}}
      \put(0.72,92.77){\rotatebox{90}{\makebox[0pt][c]{\scriptsize SD1.5}}}
      \put(6.76,86.05){\rotatebox{60}{\makebox[0pt][r]{\tiny pilot}}}
      \put(9.05,86.05){\rotatebox{60}{\makebox[0pt][r]{\tiny athlete}}}
      \put(11.33,86.05){\rotatebox{60}{\makebox[0pt][r]{\tiny teacher}}}
      \put(13.61,86.05){\rotatebox{60}{\makebox[0pt][r]{\tiny assistant}}}
      \put(15.89,86.05){\rotatebox{60}{\makebox[0pt][r]{\tiny intern}}}
      \put(18.18,86.05){\rotatebox{60}{\makebox[0pt][r]{\tiny writer}}}
      \put(20.46,86.05){\rotatebox{60}{\makebox[0pt][r]{\tiny salesperson}}}
      \put(22.74,86.05){\rotatebox{60}{\makebox[0pt][r]{\tiny chef}}}
      \put(25.02,86.05){\rotatebox{60}{\makebox[0pt][r]{\tiny librarian}}}
      \put(27.31,86.05){\rotatebox{60}{\makebox[0pt][r]{\tiny supervisor}}}
      \put(29.59,86.05){\rotatebox{60}{\makebox[0pt][r]{\tiny author}}}
      \put(31.87,86.05){\rotatebox{60}{\makebox[0pt][r]{\tiny scientist}}}
      \put(34.15,86.05){\rotatebox{60}{\makebox[0pt][r]{\tiny painter}}}
      \put(36.44,86.05){\rotatebox{60}{\makebox[0pt][r]{\tiny poet}}}
      \put(38.72,86.05){\rotatebox{60}{\makebox[0pt][r]{\tiny journalist}}}
      \put(41.00,86.05){\rotatebox{60}{\makebox[0pt][r]{\tiny entrepreneur}}}
      \put(43.29,86.05){\rotatebox{60}{\makebox[0pt][r]{\tiny accountant}}}
      \put(45.57,86.05){\rotatebox{60}{\makebox[0pt][r]{\tiny singer}}}
      \put(47.85,86.05){\rotatebox{60}{\makebox[0pt][r]{\tiny photographer}}}
      \put(50.13,86.05){\rotatebox{60}{\makebox[0pt][r]{\tiny student}}}
      \put(52.42,86.05){\rotatebox{60}{\makebox[0pt][r]{\tiny politician}}}
      \put(54.70,86.05){\rotatebox{60}{\makebox[0pt][r]{\tiny cook}}}
      \put(56.98,86.05){\rotatebox{60}{\makebox[0pt][r]{\tiny doctor}}}
      \put(59.26,86.05){\rotatebox{60}{\makebox[0pt][r]{\tiny receptionist}}}
      \put(61.55,86.05){\rotatebox{60}{\makebox[0pt][r]{\tiny engineer}}}
      \put(63.83,86.05){\rotatebox{60}{\makebox[0pt][r]{\tiny fashion designer}}}
      \put(66.11,86.05){\rotatebox{60}{\makebox[0pt][r]{\tiny programmer}}}
      \put(68.39,86.05){\rotatebox{60}{\makebox[0pt][r]{\tiny professor}}}
      \put(70.68,86.05){\rotatebox{60}{\makebox[0pt][r]{\tiny firefighter}}}
      \put(72.96,86.05){\rotatebox{60}{\makebox[0pt][r]{\tiny mechanic}}}
      \put(75.24,86.05){\rotatebox{60}{\makebox[0pt][r]{\tiny technician}}}
      \put(77.52,86.05){\rotatebox{60}{\makebox[0pt][r]{\tiny ceo}}}
      \put(79.81,86.05){\rotatebox{60}{\makebox[0pt][r]{\tiny nurse}}}
      \put(68.70,97.19){\tiny CFG}
      \put(76.14,97.19){\tiny \method}
      \put(4.44,66.61){\makebox[0pt][r]{\tiny 0\%}} \put(4.44,70.61){\makebox[0pt][r]{\tiny 10\%}} \put(4.44,74.60){\makebox[0pt][r]{\tiny 20\%}} \put(4.44,78.60){\makebox[0pt][r]{\tiny 30\%}}
      \put(0.72,72.93){\rotatebox{90}{\makebox[0pt][c]{\scriptsize SD3}}}
      \put(6.76,66.22){\rotatebox{60}{\makebox[0pt][r]{\tiny entrepreneur}}}
      \put(9.05,66.22){\rotatebox{60}{\makebox[0pt][r]{\tiny doctor}}}
      \put(11.33,66.22){\rotatebox{60}{\makebox[0pt][r]{\tiny salesperson}}}
      \put(13.61,66.22){\rotatebox{60}{\makebox[0pt][r]{\tiny cook}}}
      \put(15.89,66.22){\rotatebox{60}{\makebox[0pt][r]{\tiny technician}}}
      \put(18.18,66.22){\rotatebox{60}{\makebox[0pt][r]{\tiny painter}}}
      \put(20.46,66.22){\rotatebox{60}{\makebox[0pt][r]{\tiny journalist}}}
      \put(22.74,66.22){\rotatebox{60}{\makebox[0pt][r]{\tiny poet}}}
      \put(25.02,66.22){\rotatebox{60}{\makebox[0pt][r]{\tiny writer}}}
      \put(27.31,66.22){\rotatebox{60}{\makebox[0pt][r]{\tiny supervisor}}}
      \put(29.59,66.22){\rotatebox{60}{\makebox[0pt][r]{\tiny athlete}}}
      \put(31.87,66.22){\rotatebox{60}{\makebox[0pt][r]{\tiny professor}}}
      \put(34.15,66.22){\rotatebox{60}{\makebox[0pt][r]{\tiny scientist}}}
      \put(36.44,66.22){\rotatebox{60}{\makebox[0pt][r]{\tiny accountant}}}
      \put(38.72,66.22){\rotatebox{60}{\makebox[0pt][r]{\tiny photographer}}}
      \put(41.00,66.22){\rotatebox{60}{\makebox[0pt][r]{\tiny singer}}}
      \put(43.29,66.22){\rotatebox{60}{\makebox[0pt][r]{\tiny programmer}}}
      \put(45.57,66.22){\rotatebox{60}{\makebox[0pt][r]{\tiny politician}}}
      \put(47.85,66.22){\rotatebox{60}{\makebox[0pt][r]{\tiny engineer}}}
      \put(50.13,66.22){\rotatebox{60}{\makebox[0pt][r]{\tiny firefighter}}}
      \put(52.42,66.22){\rotatebox{60}{\makebox[0pt][r]{\tiny student}}}
      \put(54.70,66.22){\rotatebox{60}{\makebox[0pt][r]{\tiny assistant}}}
      \put(56.98,66.22){\rotatebox{60}{\makebox[0pt][r]{\tiny author}}}
      \put(59.26,66.22){\rotatebox{60}{\makebox[0pt][r]{\tiny teacher}}}
      \put(61.55,66.22){\rotatebox{60}{\makebox[0pt][r]{\tiny ceo}}}
      \put(63.83,66.22){\rotatebox{60}{\makebox[0pt][r]{\tiny intern}}}
      \put(66.11,66.22){\rotatebox{60}{\makebox[0pt][r]{\tiny chef}}}
      \put(68.39,66.22){\rotatebox{60}{\makebox[0pt][r]{\tiny pilot}}}
      \put(70.68,66.22){\rotatebox{60}{\makebox[0pt][r]{\tiny librarian}}}
      \put(72.96,66.22){\rotatebox{60}{\makebox[0pt][r]{\tiny receptionist}}}
      \put(75.24,66.22){\rotatebox{60}{\makebox[0pt][r]{\tiny nurse}}}
      \put(77.52,66.22){\rotatebox{60}{\makebox[0pt][r]{\tiny fashion designer}}}
      \put(79.81,66.22){\rotatebox{60}{\makebox[0pt][r]{\tiny mechanic}}}
      \put(4.44,46.78){\makebox[0pt][r]{\tiny 0\%}} \put(4.44,49.17){\makebox[0pt][r]{\tiny 10\%}} \put(4.44,51.57){\makebox[0pt][r]{\tiny 20\%}} \put(4.44,53.97){\makebox[0pt][r]{\tiny 30\%}} \put(4.44,56.36){\makebox[0pt][r]{\tiny 40\%}} \put(4.44,58.76){\makebox[0pt][r]{\tiny 50\%}}
      \put(0.72,53.10){\rotatebox{90}{\makebox[0pt][c]{\scriptsize FLUX.2}}}
      \put(6.76,46.38){\rotatebox{60}{\makebox[0pt][r]{\tiny painter}}}
      \put(9.05,46.38){\rotatebox{60}{\makebox[0pt][r]{\tiny student}}}
      \put(11.33,46.38){\rotatebox{60}{\makebox[0pt][r]{\tiny doctor}}}
      \put(13.61,46.38){\rotatebox{60}{\makebox[0pt][r]{\tiny photographer}}}
      \put(15.89,46.38){\rotatebox{60}{\makebox[0pt][r]{\tiny singer}}}
      \put(18.18,46.38){\rotatebox{60}{\makebox[0pt][r]{\tiny fashion designer}}}
      \put(20.46,46.38){\rotatebox{60}{\makebox[0pt][r]{\tiny journalist}}}
      \put(22.74,46.38){\rotatebox{60}{\makebox[0pt][r]{\tiny scientist}}}
      \put(25.02,46.38){\rotatebox{60}{\makebox[0pt][r]{\tiny librarian}}}
      \put(27.31,46.38){\rotatebox{60}{\makebox[0pt][r]{\tiny assistant}}}
      \put(29.59,46.38){\rotatebox{60}{\makebox[0pt][r]{\tiny supervisor}}}
      \put(31.87,46.38){\rotatebox{60}{\makebox[0pt][r]{\tiny athlete}}}
      \put(34.15,46.38){\rotatebox{60}{\makebox[0pt][r]{\tiny teacher}}}
      \put(36.44,46.38){\rotatebox{60}{\makebox[0pt][r]{\tiny cook}}}
      \put(38.72,46.38){\rotatebox{60}{\makebox[0pt][r]{\tiny pilot}}}
      \put(41.00,46.38){\rotatebox{60}{\makebox[0pt][r]{\tiny writer}}}
      \put(43.29,46.38){\rotatebox{60}{\makebox[0pt][r]{\tiny accountant}}}
      \put(45.57,46.38){\rotatebox{60}{\makebox[0pt][r]{\tiny firefighter}}}
      \put(47.85,46.38){\rotatebox{60}{\makebox[0pt][r]{\tiny programmer}}}
      \put(50.13,46.38){\rotatebox{60}{\makebox[0pt][r]{\tiny poet}}}
      \put(52.42,46.38){\rotatebox{60}{\makebox[0pt][r]{\tiny chef}}}
      \put(54.70,46.38){\rotatebox{60}{\makebox[0pt][r]{\tiny intern}}}
      \put(56.98,46.38){\rotatebox{60}{\makebox[0pt][r]{\tiny receptionist}}}
      \put(59.26,46.38){\rotatebox{60}{\makebox[0pt][r]{\tiny author}}}
      \put(61.55,46.38){\rotatebox{60}{\makebox[0pt][r]{\tiny entrepreneur}}}
      \put(63.83,46.38){\rotatebox{60}{\makebox[0pt][r]{\tiny salesperson}}}
      \put(66.11,46.38){\rotatebox{60}{\makebox[0pt][r]{\tiny technician}}}
      \put(68.39,46.38){\rotatebox{60}{\makebox[0pt][r]{\tiny mechanic}}}
      \put(70.68,46.38){\rotatebox{60}{\makebox[0pt][r]{\tiny ceo}}}
      \put(72.96,46.38){\rotatebox{60}{\makebox[0pt][r]{\tiny engineer}}}
      \put(75.24,46.38){\rotatebox{60}{\makebox[0pt][r]{\tiny professor}}}
      \put(77.52,46.38){\rotatebox{60}{\makebox[0pt][r]{\tiny nurse}}}
      \put(79.81,46.38){\rotatebox{60}{\makebox[0pt][r]{\tiny politician}}}
      \put(4.44,26.94){\makebox[0pt][r]{\tiny 0\%}} \put(4.44,28.65){\makebox[0pt][r]{\tiny 10\%}} \put(4.44,30.37){\makebox[0pt][r]{\tiny 20\%}} \put(4.44,32.08){\makebox[0pt][r]{\tiny 30\%}} \put(4.44,33.79){\makebox[0pt][r]{\tiny 40\%}} \put(4.44,35.50){\makebox[0pt][r]{\tiny 50\%}} \put(4.44,37.21){\makebox[0pt][r]{\tiny 60\%}} \put(4.44,38.93){\makebox[0pt][r]{\tiny 70\%}}
      \put(0.72,33.26){\rotatebox{90}{\makebox[0pt][c]{\scriptsize FT}}}
      \put(6.76,26.55){\rotatebox{60}{\makebox[0pt][r]{\tiny firefighter}}}
      \put(9.05,26.55){\rotatebox{60}{\makebox[0pt][r]{\tiny pilot}}}
      \put(11.33,26.55){\rotatebox{60}{\makebox[0pt][r]{\tiny mechanic}}}
      \put(13.61,26.55){\rotatebox{60}{\makebox[0pt][r]{\tiny technician}}}
      \put(15.89,26.55){\rotatebox{60}{\makebox[0pt][r]{\tiny painter}}}
      \put(18.18,26.55){\rotatebox{60}{\makebox[0pt][r]{\tiny nurse}}}
      \put(20.46,26.55){\rotatebox{60}{\makebox[0pt][r]{\tiny teacher}}}
      \put(22.74,26.55){\rotatebox{60}{\makebox[0pt][r]{\tiny athlete}}}
      \put(25.02,26.55){\rotatebox{60}{\makebox[0pt][r]{\tiny supervisor}}}
      \put(27.31,26.55){\rotatebox{60}{\makebox[0pt][r]{\tiny librarian}}}
      \put(29.59,26.55){\rotatebox{60}{\makebox[0pt][r]{\tiny receptionist}}}
      \put(31.87,26.55){\rotatebox{60}{\makebox[0pt][r]{\tiny journalist}}}
      \put(34.15,26.55){\rotatebox{60}{\makebox[0pt][r]{\tiny assistant}}}
      \put(36.44,26.55){\rotatebox{60}{\makebox[0pt][r]{\tiny cook}}}
      \put(38.72,26.55){\rotatebox{60}{\makebox[0pt][r]{\tiny engineer}}}
      \put(41.00,26.55){\rotatebox{60}{\makebox[0pt][r]{\tiny singer}}}
      \put(43.29,26.55){\rotatebox{60}{\makebox[0pt][r]{\tiny student}}}
      \put(45.57,26.55){\rotatebox{60}{\makebox[0pt][r]{\tiny fashion designer}}}
      \put(47.85,26.55){\rotatebox{60}{\makebox[0pt][r]{\tiny poet}}}
      \put(50.13,26.55){\rotatebox{60}{\makebox[0pt][r]{\tiny photographer}}}
      \put(52.42,26.55){\rotatebox{60}{\makebox[0pt][r]{\tiny scientist}}}
      \put(54.70,26.55){\rotatebox{60}{\makebox[0pt][r]{\tiny writer}}}
      \put(56.98,26.55){\rotatebox{60}{\makebox[0pt][r]{\tiny chef}}}
      \put(59.26,26.55){\rotatebox{60}{\makebox[0pt][r]{\tiny author}}}
      \put(61.55,26.55){\rotatebox{60}{\makebox[0pt][r]{\tiny programmer}}}
      \put(63.83,26.55){\rotatebox{60}{\makebox[0pt][r]{\tiny politician}}}
      \put(66.11,26.55){\rotatebox{60}{\makebox[0pt][r]{\tiny salesperson}}}
      \put(68.39,26.55){\rotatebox{60}{\makebox[0pt][r]{\tiny entrepreneur}}}
      \put(70.68,26.55){\rotatebox{60}{\makebox[0pt][r]{\tiny professor}}}
      \put(72.96,26.55){\rotatebox{60}{\makebox[0pt][r]{\tiny intern}}}
      \put(75.24,26.55){\rotatebox{60}{\makebox[0pt][r]{\tiny accountant}}}
      \put(77.52,26.55){\rotatebox{60}{\makebox[0pt][r]{\tiny ceo}}}
      \put(79.81,26.55){\rotatebox{60}{\makebox[0pt][r]{\tiny doctor}}}
      \put(4.44,7.11){\makebox[0pt][r]{\tiny 0\%}} \put(4.44,10.10){\makebox[0pt][r]{\tiny 10\%}} \put(4.44,13.10){\makebox[0pt][r]{\tiny 20\%}} \put(4.44,16.10){\makebox[0pt][r]{\tiny 30\%}} \put(4.44,19.09){\makebox[0pt][r]{\tiny 40\%}}
      \put(0.72,13.43){\rotatebox{90}{\makebox[0pt][c]{\scriptsize UCE}}}
      \put(6.76,6.71){\rotatebox{60}{\makebox[0pt][r]{\tiny pilot}}}
      \put(9.05,6.71){\rotatebox{60}{\makebox[0pt][r]{\tiny firefighter}}}
      \put(11.33,6.71){\rotatebox{60}{\makebox[0pt][r]{\tiny poet}}}
      \put(13.61,6.71){\rotatebox{60}{\makebox[0pt][r]{\tiny mechanic}}}
      \put(15.89,6.71){\rotatebox{60}{\makebox[0pt][r]{\tiny singer}}}
      \put(18.18,6.71){\rotatebox{60}{\makebox[0pt][r]{\tiny student}}}
      \put(20.46,6.71){\rotatebox{60}{\makebox[0pt][r]{\tiny supervisor}}}
      \put(22.74,6.71){\rotatebox{60}{\makebox[0pt][r]{\tiny salesperson}}}
      \put(25.02,6.71){\rotatebox{60}{\makebox[0pt][r]{\tiny photographer}}}
      \put(27.31,6.71){\rotatebox{60}{\makebox[0pt][r]{\tiny engineer}}}
      \put(29.59,6.71){\rotatebox{60}{\makebox[0pt][r]{\tiny technician}}}
      \put(31.87,6.71){\rotatebox{60}{\makebox[0pt][r]{\tiny painter}}}
      \put(34.15,6.71){\rotatebox{60}{\makebox[0pt][r]{\tiny nurse}}}
      \put(36.44,6.71){\rotatebox{60}{\makebox[0pt][r]{\tiny intern}}}
      \put(38.72,6.71){\rotatebox{60}{\makebox[0pt][r]{\tiny scientist}}}
      \put(41.00,6.71){\rotatebox{60}{\makebox[0pt][r]{\tiny writer}}}
      \put(43.29,6.71){\rotatebox{60}{\makebox[0pt][r]{\tiny professor}}}
      \put(45.57,6.71){\rotatebox{60}{\makebox[0pt][r]{\tiny receptionist}}}
      \put(47.85,6.71){\rotatebox{60}{\makebox[0pt][r]{\tiny doctor}}}
      \put(50.13,6.71){\rotatebox{60}{\makebox[0pt][r]{\tiny fashion designer}}}
      \put(52.42,6.71){\rotatebox{60}{\makebox[0pt][r]{\tiny athlete}}}
      \put(54.70,6.71){\rotatebox{60}{\makebox[0pt][r]{\tiny ceo}}}
      \put(56.98,6.71){\rotatebox{60}{\makebox[0pt][r]{\tiny chef}}}
      \put(59.26,6.71){\rotatebox{60}{\makebox[0pt][r]{\tiny programmer}}}
      \put(61.55,6.71){\rotatebox{60}{\makebox[0pt][r]{\tiny entrepreneur}}}
      \put(63.83,6.71){\rotatebox{60}{\makebox[0pt][r]{\tiny librarian}}}
      \put(66.11,6.71){\rotatebox{60}{\makebox[0pt][r]{\tiny teacher}}}
      \put(68.39,6.71){\rotatebox{60}{\makebox[0pt][r]{\tiny cook}}}
      \put(70.68,6.71){\rotatebox{60}{\makebox[0pt][r]{\tiny accountant}}}
      \put(72.96,6.71){\rotatebox{60}{\makebox[0pt][r]{\tiny politician}}}
      \put(75.24,6.71){\rotatebox{60}{\makebox[0pt][r]{\tiny journalist}}}
      \put(77.52,6.71){\rotatebox{60}{\makebox[0pt][r]{\tiny assistant}}}
      \put(79.81,6.71){\rotatebox{60}{\makebox[0pt][r]{\tiny author}}}
    \end{overpic}
  \caption{\textbf{Per-occupation bias range across vanilla and debiased models.}}
  \label{fig:bias_range_2x2}
\end{figure}

\newpage
\textbf{Female-ratio curves across occupations.}
\cref{fig:occ_curves_4x4} shows the female-ratio curves across guidance scales for four occupations per model, comparing CFG with \method.
\begin{figure}[H]
  \centering
    \begin{overpic}[width=0.92\linewidth,percent]{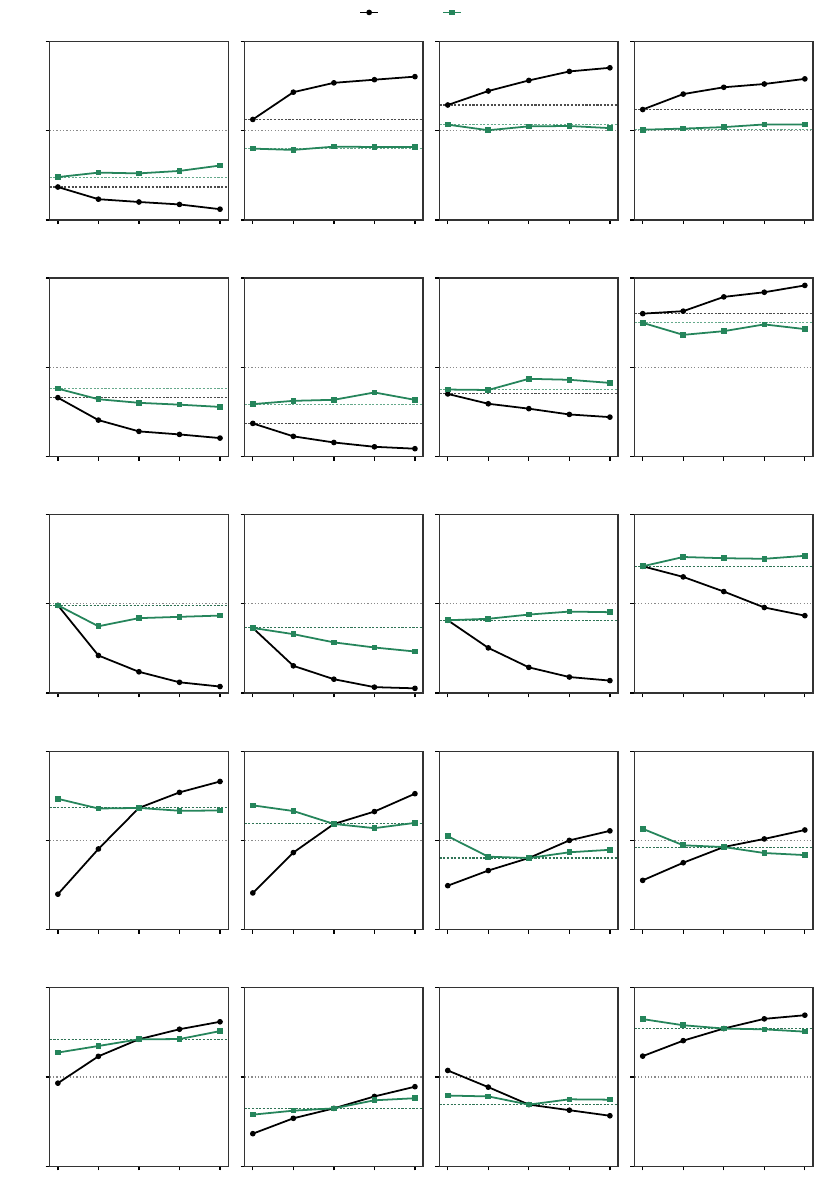}
      \put(11.57,96.98){\makebox[0pt][c]{\tiny chef ($\alpha^\star{=}{-}10$)}}
      \put(4.82,80.57){\makebox[0pt][c]{\tiny 2.5}} \put(8.20,80.57){\makebox[0pt][c]{\tiny 5.0}} \put(11.57,80.57){\makebox[0pt][c]{\tiny 7.5}} \put(14.95,80.57){\makebox[0pt][c]{\tiny 10.0}} \put(18.32,80.57){\makebox[0pt][c]{\tiny 12.5}}
      \put(11.57,79.27){\makebox[0pt][c]{\scriptsize Guidance scale $w$}}
      \put(27.81,96.98){\makebox[0pt][c]{\tiny athlete ($\alpha^\star{=}{+}5$)}}
      \put(21.06,80.57){\makebox[0pt][c]{\tiny 2.5}} \put(24.43,80.57){\makebox[0pt][c]{\tiny 5.0}} \put(27.81,80.57){\makebox[0pt][c]{\tiny 7.5}} \put(31.18,80.57){\makebox[0pt][c]{\tiny 10.0}} \put(34.56,80.57){\makebox[0pt][c]{\tiny 12.5}}
      \put(27.81,79.27){\makebox[0pt][c]{\scriptsize Guidance scale $w$}}
      \put(44.04,96.98){\makebox[0pt][c]{\tiny teacher ($\alpha^\star{=}{+}7.5$)}}
      \put(37.29,80.57){\makebox[0pt][c]{\tiny 2.5}} \put(40.67,80.57){\makebox[0pt][c]{\tiny 5.0}} \put(44.04,80.57){\makebox[0pt][c]{\tiny 7.5}} \put(47.42,80.57){\makebox[0pt][c]{\tiny 10.0}} \put(50.79,80.57){\makebox[0pt][c]{\tiny 12.5}}
      \put(44.04,79.27){\makebox[0pt][c]{\scriptsize Guidance scale $w$}}
      \put(60.28,96.98){\makebox[0pt][c]{\tiny assistant ($\alpha^\star{=}{+}10$)}}
      \put(53.52,80.57){\makebox[0pt][c]{\tiny 2.5}} \put(56.90,80.57){\makebox[0pt][c]{\tiny 5.0}} \put(60.28,80.57){\makebox[0pt][c]{\tiny 7.5}} \put(63.65,80.57){\makebox[0pt][c]{\tiny 10.0}} \put(67.03,80.57){\makebox[0pt][c]{\tiny 12.5}}
      \put(60.28,79.27){\makebox[0pt][c]{\scriptsize Guidance scale $w$}}
      \put(3.71,81.42){\makebox[0pt][r]{\tiny 0\%}} \put(3.71,88.84){\makebox[0pt][r]{\tiny 50\%}} \put(3.71,96.27){\makebox[0pt][r]{\tiny 100\%}}
      \put(0.60,89.12){\rotatebox{90}{\makebox[0pt][c]{\scriptsize SD1.5}}}
      \put(11.57,77.29){\makebox[0pt][c]{\tiny entrepreneur ($\alpha^\star{=}{-}150$)}}
      \put(4.82,60.88){\makebox[0pt][c]{\tiny 1.5}} \put(8.20,60.88){\makebox[0pt][c]{\tiny 3.0}} \put(11.57,60.88){\makebox[0pt][c]{\tiny 4.5}} \put(14.95,60.88){\makebox[0pt][c]{\tiny 6.0}} \put(18.32,60.88){\makebox[0pt][c]{\tiny 7.5}}
      \put(11.57,59.59){\makebox[0pt][c]{\scriptsize Guidance scale $w$}}
      \put(27.81,77.29){\makebox[0pt][c]{\tiny athlete ($\alpha^\star{=}{-}250$)}}
      \put(21.06,60.88){\makebox[0pt][c]{\tiny 1.5}} \put(24.43,60.88){\makebox[0pt][c]{\tiny 3.0}} \put(27.81,60.88){\makebox[0pt][c]{\tiny 4.5}} \put(31.18,60.88){\makebox[0pt][c]{\tiny 6.0}} \put(34.56,60.88){\makebox[0pt][c]{\tiny 7.5}}
      \put(27.81,59.59){\makebox[0pt][c]{\scriptsize Guidance scale $w$}}
      \put(44.04,77.29){\makebox[0pt][c]{\tiny professor ($\alpha^\star{=}{-}125$)}}
      \put(37.29,60.88){\makebox[0pt][c]{\tiny 1.5}} \put(40.67,60.88){\makebox[0pt][c]{\tiny 3.0}} \put(44.04,60.88){\makebox[0pt][c]{\tiny 4.5}} \put(47.42,60.88){\makebox[0pt][c]{\tiny 6.0}} \put(50.79,60.88){\makebox[0pt][c]{\tiny 7.5}}
      \put(44.04,59.59){\makebox[0pt][c]{\scriptsize Guidance scale $w$}}
      \put(60.28,77.29){\makebox[0pt][c]{\tiny painter ($\alpha^\star{=}{+}90$)}}
      \put(53.52,60.88){\makebox[0pt][c]{\tiny 1.5}} \put(56.90,60.88){\makebox[0pt][c]{\tiny 3.0}} \put(60.28,60.88){\makebox[0pt][c]{\tiny 4.5}} \put(63.65,60.88){\makebox[0pt][c]{\tiny 6.0}} \put(67.03,60.88){\makebox[0pt][c]{\tiny 7.5}}
      \put(60.28,59.59){\makebox[0pt][c]{\scriptsize Guidance scale $w$}}
      \put(3.71,61.73){\makebox[0pt][r]{\tiny 0\%}} \put(3.71,69.15){\makebox[0pt][r]{\tiny 50\%}} \put(3.71,76.58){\makebox[0pt][r]{\tiny 100\%}}
      \put(0.60,69.43){\rotatebox{90}{\makebox[0pt][c]{\scriptsize SD3}}}
      \put(11.57,57.60){\makebox[0pt][c]{\tiny painter ($\alpha^\star{=}{-}3000$)}}
      \put(4.82,41.19){\makebox[0pt][c]{\tiny 1.0}} \put(8.20,41.19){\makebox[0pt][c]{\tiny 2.0}} \put(11.57,41.19){\makebox[0pt][c]{\tiny 3.0}} \put(14.95,41.19){\makebox[0pt][c]{\tiny 4.0}} \put(18.32,41.19){\makebox[0pt][c]{\tiny 5.0}}
      \put(11.57,39.90){\makebox[0pt][c]{\scriptsize Guidance scale $w$}}
      \put(27.81,57.60){\makebox[0pt][c]{\tiny student ($\alpha^\star{=}{-}6500$)}}
      \put(21.06,41.19){\makebox[0pt][c]{\tiny 1.0}} \put(24.43,41.19){\makebox[0pt][c]{\tiny 2.0}} \put(27.81,41.19){\makebox[0pt][c]{\tiny 3.0}} \put(31.18,41.19){\makebox[0pt][c]{\tiny 4.0}} \put(34.56,41.19){\makebox[0pt][c]{\tiny 5.0}}
      \put(27.81,39.90){\makebox[0pt][c]{\scriptsize Guidance scale $w$}}
      \put(44.04,57.60){\makebox[0pt][c]{\tiny doctor ($\alpha^\star{=}{-}4250$)}}
      \put(37.29,41.19){\makebox[0pt][c]{\tiny 1.0}} \put(40.67,41.19){\makebox[0pt][c]{\tiny 2.0}} \put(44.04,41.19){\makebox[0pt][c]{\tiny 3.0}} \put(47.42,41.19){\makebox[0pt][c]{\tiny 4.0}} \put(50.79,41.19){\makebox[0pt][c]{\tiny 5.0}}
      \put(44.04,39.90){\makebox[0pt][c]{\scriptsize Guidance scale $w$}}
      \put(60.28,57.60){\makebox[0pt][c]{\tiny fashion designer ($\alpha^\star{=}{-}2250$)}}
      \put(53.52,41.19){\makebox[0pt][c]{\tiny 1.0}} \put(56.90,41.19){\makebox[0pt][c]{\tiny 2.0}} \put(60.28,41.19){\makebox[0pt][c]{\tiny 3.0}} \put(63.65,41.19){\makebox[0pt][c]{\tiny 4.0}} \put(67.03,41.19){\makebox[0pt][c]{\tiny 5.0}}
      \put(60.28,39.90){\makebox[0pt][c]{\scriptsize Guidance scale $w$}}
      \put(3.71,42.04){\makebox[0pt][r]{\tiny 0\%}} \put(3.71,49.46){\makebox[0pt][r]{\tiny 50\%}} \put(3.71,56.89){\makebox[0pt][r]{\tiny 100\%}}
      \put(0.60,49.74){\rotatebox{90}{\makebox[0pt][c]{\scriptsize FLUX.2}}}
      \put(11.57,37.91){\makebox[0pt][c]{\tiny firefighter ($\alpha^\star{=}{+}12.5$)}}
      \put(4.82,21.50){\makebox[0pt][c]{\tiny 2.5}} \put(8.20,21.50){\makebox[0pt][c]{\tiny 5.0}} \put(11.57,21.50){\makebox[0pt][c]{\tiny 7.5}} \put(14.95,21.50){\makebox[0pt][c]{\tiny 10.0}} \put(18.32,21.50){\makebox[0pt][c]{\tiny 12.5}}
      \put(11.57,20.21){\makebox[0pt][c]{\scriptsize Guidance scale $w$}}
      \put(27.81,37.91){\makebox[0pt][c]{\tiny pilot ($\alpha^\star{=}{+}10$)}}
      \put(21.06,21.50){\makebox[0pt][c]{\tiny 2.5}} \put(24.43,21.50){\makebox[0pt][c]{\tiny 5.0}} \put(27.81,21.50){\makebox[0pt][c]{\tiny 7.5}} \put(31.18,21.50){\makebox[0pt][c]{\tiny 10.0}} \put(34.56,21.50){\makebox[0pt][c]{\tiny 12.5}}
      \put(27.81,20.21){\makebox[0pt][c]{\scriptsize Guidance scale $w$}}
      \put(44.04,37.91){\makebox[0pt][c]{\tiny mechanic ($\alpha^\star{=}{+}7.5$)}}
      \put(37.29,21.50){\makebox[0pt][c]{\tiny 2.5}} \put(40.67,21.50){\makebox[0pt][c]{\tiny 5.0}} \put(44.04,21.50){\makebox[0pt][c]{\tiny 7.5}} \put(47.42,21.50){\makebox[0pt][c]{\tiny 10.0}} \put(50.79,21.50){\makebox[0pt][c]{\tiny 12.5}}
      \put(44.04,20.21){\makebox[0pt][c]{\scriptsize Guidance scale $w$}}
      \put(60.28,37.91){\makebox[0pt][c]{\tiny technician ($\alpha^\star{=}{+}7.5$)}}
      \put(53.52,21.50){\makebox[0pt][c]{\tiny 2.5}} \put(56.90,21.50){\makebox[0pt][c]{\tiny 5.0}} \put(60.28,21.50){\makebox[0pt][c]{\tiny 7.5}} \put(63.65,21.50){\makebox[0pt][c]{\tiny 10.0}} \put(67.03,21.50){\makebox[0pt][c]{\tiny 12.5}}
      \put(60.28,20.21){\makebox[0pt][c]{\scriptsize Guidance scale $w$}}
      \put(3.71,22.35){\makebox[0pt][r]{\tiny 0\%}} \put(3.71,29.78){\makebox[0pt][r]{\tiny 50\%}} \put(3.71,37.20){\makebox[0pt][r]{\tiny 100\%}}
      \put(0.60,30.05){\rotatebox{90}{\makebox[0pt][c]{\scriptsize FT}}}
      \put(11.57,18.22){\makebox[0pt][c]{\tiny pilot ($\alpha^\star{=}{+}2.5$)}}
      \put(4.82,1.81){\makebox[0pt][c]{\tiny 2.5}} \put(8.20,1.81){\makebox[0pt][c]{\tiny 5.0}} \put(11.57,1.81){\makebox[0pt][c]{\tiny 7.5}} \put(14.95,1.81){\makebox[0pt][c]{\tiny 10.0}} \put(18.32,1.81){\makebox[0pt][c]{\tiny 12.5}}
      \put(11.57,0.52){\makebox[0pt][c]{\scriptsize Guidance scale $w$}}
      \put(27.81,18.22){\makebox[0pt][c]{\tiny firefighter ($\alpha^\star{=}{+}1.25$)}}
      \put(21.06,1.81){\makebox[0pt][c]{\tiny 2.5}} \put(24.43,1.81){\makebox[0pt][c]{\tiny 5.0}} \put(27.81,1.81){\makebox[0pt][c]{\tiny 7.5}} \put(31.18,1.81){\makebox[0pt][c]{\tiny 10.0}} \put(34.56,1.81){\makebox[0pt][c]{\tiny 12.5}}
      \put(27.81,0.52){\makebox[0pt][c]{\scriptsize Guidance scale $w$}}
      \put(44.04,18.22){\makebox[0pt][c]{\tiny poet ($\alpha^\star{=}{-}2.5$)}}
      \put(37.29,1.81){\makebox[0pt][c]{\tiny 2.5}} \put(40.67,1.81){\makebox[0pt][c]{\tiny 5.0}} \put(44.04,1.81){\makebox[0pt][c]{\tiny 7.5}} \put(47.42,1.81){\makebox[0pt][c]{\tiny 10.0}} \put(50.79,1.81){\makebox[0pt][c]{\tiny 12.5}}
      \put(44.04,0.52){\makebox[0pt][c]{\scriptsize Guidance scale $w$}}
      \put(60.28,18.22){\makebox[0pt][c]{\tiny mechanic ($\alpha^\star{=}{+}5$)}}
      \put(53.52,1.81){\makebox[0pt][c]{\tiny 2.5}} \put(56.90,1.81){\makebox[0pt][c]{\tiny 5.0}} \put(60.28,1.81){\makebox[0pt][c]{\tiny 7.5}} \put(63.65,1.81){\makebox[0pt][c]{\tiny 10.0}} \put(67.03,1.81){\makebox[0pt][c]{\tiny 12.5}}
      \put(60.28,0.52){\makebox[0pt][c]{\scriptsize Guidance scale $w$}}
      \put(3.71,2.66){\makebox[0pt][r]{\tiny 0\%}} \put(3.71,10.09){\makebox[0pt][r]{\tiny 50\%}} \put(3.71,17.51){\makebox[0pt][r]{\tiny 100\%}}
      \put(0.60,10.36){\rotatebox{90}{\makebox[0pt][c]{\scriptsize UCE}}}
      \put(31.87,98.69){\tiny CFG}
      \put(38.77,98.69){\tiny \method}
    \end{overpic}
  \caption{\textbf{Per-occupation female ratio curves across guidance scales.} \textcolor{black}{$\alpha^\star$ is \method (Search).}}
  \label{fig:occ_curves_4x4}
\end{figure}

\newpage
\textbf{Controlling guidance bias through $\alpha$.}
\vspace{-8pt}
\begin{figure}[H]
    \centering
    \includegraphics[width=\linewidth,trim=0 0bp 0 11bp,clip]{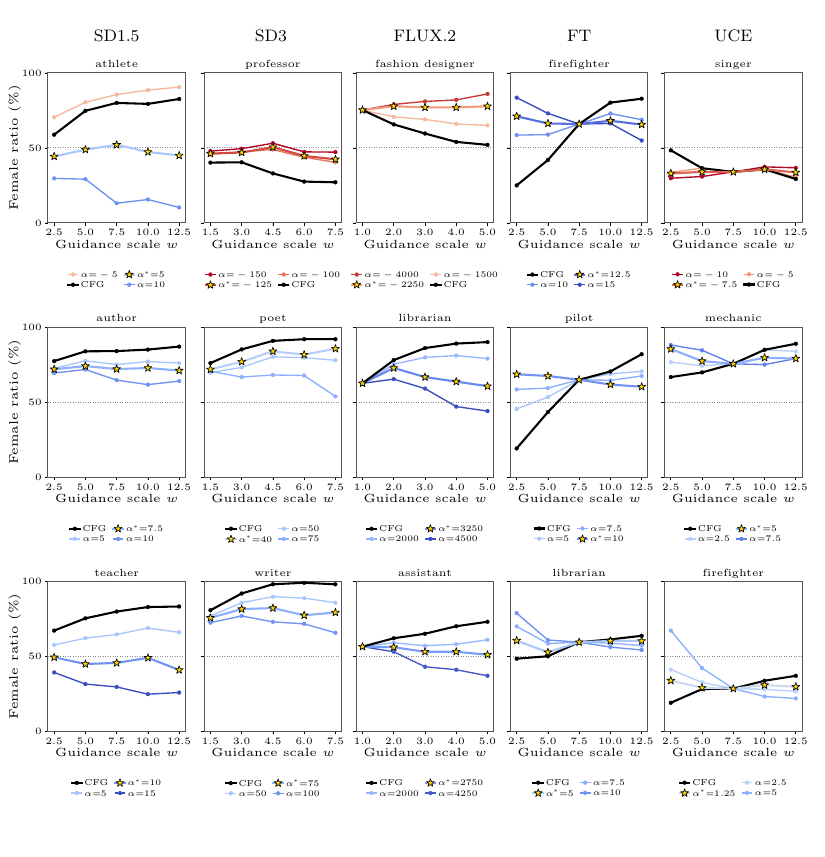}
    \caption{\textbf{Additional $\alpha$ ablations across models and occupations.} \textcolor{black}{$\alpha^\star$ is \method (Search).}}
    \label{fig:app_alpha_sweep_3x5}
\end{figure}

\vspace{-9pt}
\newpage
\textbf{Mitigating guidance bias with \method.}
\textcolor{black}{\Cref{fig:app_occ_scatter_stayfair} shows similar female ratios at the lowest and highest guidance scales, with points near $y=x$, indicating reduced guidance bias across occupations. \Cref{fig:app_fr_curves_stayfair} shows flatter curves under \method (bottom) than CFG (top) across the scale range. \Cref{fig:app_sd15_mitigating_occ} collects cases where increasing guidance reduces total bias.}
\begingroup
\setlength{\intextsep}{6pt}
\captionsetup{skip=3pt}
\begin{figure}[H]
  \centering
    \begin{overpic}[width=\linewidth,percent]{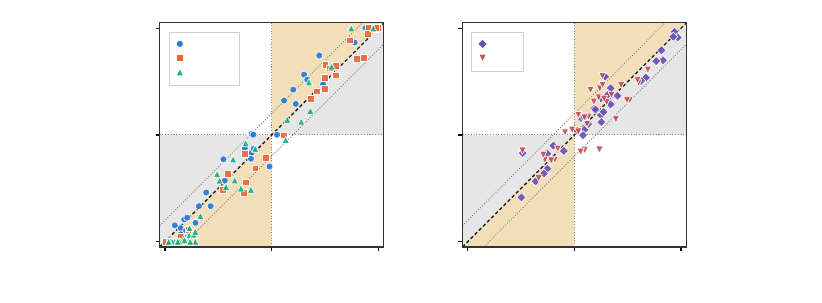}
      \put(23.03,31.12){\tiny SD1.5}
      \put(23.03,29.40){\tiny SD3}
      \put(23.03,27.69){\tiny FLUX.2}
      \put(20.09,5.18){\makebox[0pt][c]{\tiny 0}} \put(33.08,5.18){\makebox[0pt][c]{\tiny 50}} \put(46.07,5.18){\makebox[0pt][c]{\tiny 100}}
      \put(33.08,3.03){\makebox[0pt][c]{\scriptsize Female ratio (\%, lowest $w$)}}
      \put(18.81,7.06){\makebox[0pt][r]{\tiny 0}} \put(18.81,20.05){\makebox[0pt][r]{\tiny 50}} \put(18.81,33.04){\makebox[0pt][r]{\tiny 100}}
      \put(15.66,20.45){\rotatebox{90}{\makebox[0pt][c]{\scriptsize Female ratio (\%, highest $w$)}}}
      \put(33.08,0.63){\makebox[0pt][c]{\small (a) Vanilla models}}
      \put(59.90,31.12){\tiny UCE}
      \put(59.90,29.40){\tiny FT}
      \put(56.96,5.18){\makebox[0pt][c]{\tiny 0}} \put(69.95,5.18){\makebox[0pt][c]{\tiny 50}} \put(82.94,5.18){\makebox[0pt][c]{\tiny 100}}
      \put(69.95,3.03){\makebox[0pt][c]{\scriptsize Female ratio (\%, lowest $w$)}}
      \put(55.68,7.06){\makebox[0pt][r]{\tiny 0}} \put(55.68,20.05){\makebox[0pt][r]{\tiny 50}} \put(55.68,33.04){\makebox[0pt][r]{\tiny 100}}
      \put(52.53,20.45){\rotatebox{90}{\makebox[0pt][c]{\scriptsize Female ratio (\%, highest $w$)}}}
      \put(69.95,0.63){\makebox[0pt][c]{\small (b) Debiased models}}
      \put(83.59,35.10){\makebox[0pt][r]{\tiny Guidance bias:\,{\setlength{\fboxsep}{0pt}\setlength{\fboxrule}{0.2pt}\fcolorbox[HTML]{BBBBBB}{F2DFB8}{\rule{0pt}{3pt}\rule{3.6pt}{0pt}}}\,Amplify\ {\setlength{\fboxsep}{0pt}\setlength{\fboxrule}{0.2pt}\fcolorbox[HTML]{BBBBBB}{E6E6E6}{\rule{0pt}{3pt}\rule{3.6pt}{0pt}}}\,Mitigate}}
    \end{overpic}
  \caption{\textcolor{black}{\textbf{Per-occupation female ratios at the lowest and highest guidance scales under \method.}
  Each point represents one of the 33 occupations on (a) vanilla models and (b) debiased models.}}
  \label{fig:app_occ_scatter_stayfair}%
  \renewcommand{\thesubfigure}{(\alph{subfigure})}%
  \phantomsubcaption\label{fig:app_occ_scatter_stayfair-a}%
  \phantomsubcaption\label{fig:app_occ_scatter_stayfair-b}%
\end{figure}

\begin{figure}[H]
  \centering
    \begin{overpic}[width=\linewidth,percent]{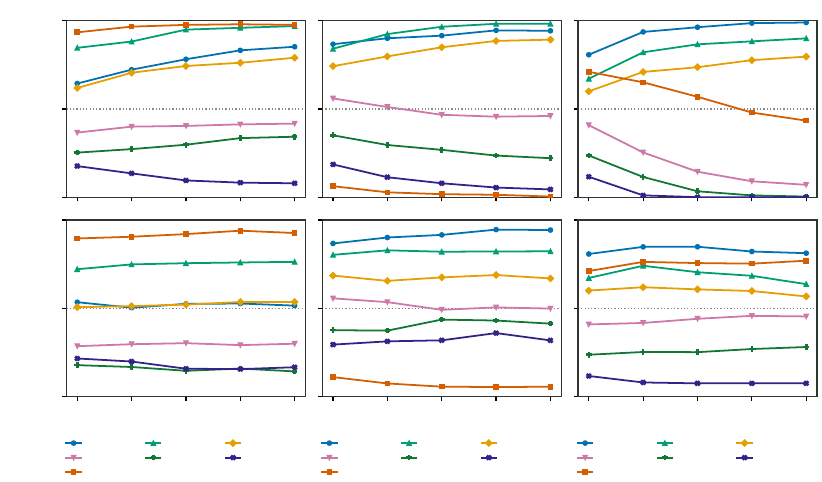}
      \put(9.40,32.95){\makebox[0pt][c]{\tiny 2.5}} \put(16.02,32.95){\makebox[0pt][c]{\tiny 5.0}} \put(22.64,32.95){\makebox[0pt][c]{\tiny 7.5}} \put(29.26,32.95){\makebox[0pt][c]{\tiny 10.0}} \put(35.88,32.95){\makebox[0pt][c]{\tiny 12.5}}
      \put(40.55,32.95){\makebox[0pt][c]{\tiny 1.5}} \put(47.17,32.95){\makebox[0pt][c]{\tiny 3.0}} \put(53.79,32.95){\makebox[0pt][c]{\tiny 4.5}} \put(60.41,32.95){\makebox[0pt][c]{\tiny 6.0}} \put(67.03,32.95){\makebox[0pt][c]{\tiny 7.5}}
      \put(71.69,32.95){\makebox[0pt][c]{\tiny 1.0}} \put(78.31,32.95){\makebox[0pt][c]{\tiny 2.0}} \put(84.93,32.95){\makebox[0pt][c]{\tiny 3.0}} \put(91.55,32.95){\makebox[0pt][c]{\tiny 4.0}} \put(98.17,32.95){\makebox[0pt][c]{\tiny 5.0}}
      \put(7.45,34.19){\makebox[0pt][r]{\tiny 0}} \put(7.45,44.92){\makebox[0pt][r]{\tiny 50}} \put(7.45,55.66){\makebox[0pt][r]{\tiny 100}}
      \put(3.28,45.33){\rotatebox{90}{\makebox[0pt][c]{\scriptsize Female ratio (\%)}}}
      \put(0.88,45.33){\rotatebox{90}{\makebox[0pt][c]{\small CFG}}}
      \put(9.40,8.71){\makebox[0pt][c]{\tiny 2.5}} \put(16.02,8.71){\makebox[0pt][c]{\tiny 5.0}} \put(22.64,8.71){\makebox[0pt][c]{\tiny 7.5}} \put(29.26,8.71){\makebox[0pt][c]{\tiny 10.0}} \put(35.88,8.71){\makebox[0pt][c]{\tiny 12.5}}
      \put(40.55,8.71){\makebox[0pt][c]{\tiny 1.5}} \put(47.17,8.71){\makebox[0pt][c]{\tiny 3.0}} \put(53.79,8.71){\makebox[0pt][c]{\tiny 4.5}} \put(60.41,8.71){\makebox[0pt][c]{\tiny 6.0}} \put(67.03,8.71){\makebox[0pt][c]{\tiny 7.5}}
      \put(71.69,8.71){\makebox[0pt][c]{\tiny 1.0}} \put(78.31,8.71){\makebox[0pt][c]{\tiny 2.0}} \put(84.93,8.71){\makebox[0pt][c]{\tiny 3.0}} \put(91.55,8.71){\makebox[0pt][c]{\tiny 4.0}} \put(98.17,8.71){\makebox[0pt][c]{\tiny 5.0}}
      \put(7.45,9.95){\makebox[0pt][r]{\tiny 0}} \put(7.45,20.68){\makebox[0pt][r]{\tiny 50}} \put(7.45,31.41){\makebox[0pt][r]{\tiny 100}}
      \put(3.28,21.09){\rotatebox{90}{\makebox[0pt][c]{\scriptsize Female ratio (\%)}}}
      \put(0.88,21.09){\rotatebox{90}{\makebox[0pt][c]{\small \method}}}
      \put(22.64,56.82){\makebox[0pt][c]{\small SD1.5}}
      \put(53.79,56.82){\makebox[0pt][c]{\small SD3}}
      \put(84.93,56.82){\makebox[0pt][c]{\small FLUX.2}}
      \put(53.79,6.82){\makebox[0pt][c]{\scriptsize Guidance scale $w$}}
      \put(10.48,4.27){\tiny teacher}
      \put(20.19,4.27){\tiny librarian}
      \put(29.90,4.27){\tiny assistant}
      \put(10.48,2.50){\tiny doctor}
      \put(20.19,2.50){\tiny scientist}
      \put(29.90,2.50){\tiny poet}
      \put(10.48,0.73){\tiny fashion designer}
      \put(41.62,4.27){\tiny singer}
      \put(51.33,4.27){\tiny writer}
      \put(61.04,4.27){\tiny journalist}
      \put(41.62,2.50){\tiny accountant}
      \put(51.33,2.50){\tiny professor}
      \put(61.04,2.50){\tiny athlete}
      \put(41.62,0.73){\tiny engineer}
      \put(72.77,4.27){\tiny teacher}
      \put(82.48,4.27){\tiny librarian}
      \put(92.19,4.27){\tiny assistant}
      \put(72.77,2.50){\tiny doctor}
      \put(82.48,2.50){\tiny scientist}
      \put(92.19,2.50){\tiny poet}
      \put(72.77,0.73){\tiny fashion designer}
    \end{overpic}
  \caption{\textbf{Female ratios across guidance scales under CFG (top) and \method (bottom).}}
  \label{fig:app_fr_curves_stayfair}
  \label{fig:app_fr_curves_sd15_sd3}
\end{figure}

\begin{figure}[H]
  \centering
    \begin{overpic}[width=\linewidth,percent]{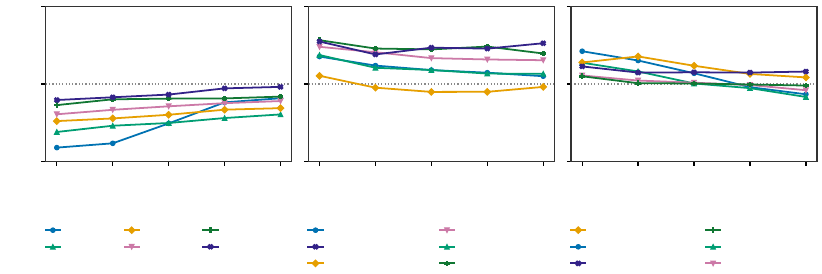}
      \put(7.95,4.90){\fontsize{5.5}{6.5}\selectfont pilot}
      \put(7.95,2.88){\fontsize{5.5}{6.5}\selectfont supervisor}
      \put(17.55,4.90){\fontsize{5.5}{6.5}\selectfont scientist}
      \put(17.55,2.88){\fontsize{5.5}{6.5}\selectfont accountant}
      \put(27.15,4.90){\fontsize{5.5}{6.5}\selectfont doctor}
      \put(27.15,2.88){\fontsize{5.5}{6.5}\selectfont painter}
      \put(6.92,11.99){\makebox[0pt][c]{\tiny 2.5}} \put(13.73,11.99){\makebox[0pt][c]{\tiny 5.0}} \put(20.54,11.99){\makebox[0pt][c]{\tiny 7.5}} \put(27.35,11.99){\makebox[0pt][c]{\tiny 10.0}} \put(34.16,11.99){\makebox[0pt][c]{\tiny 12.5}}
      \put(20.54,9.85){\makebox[0pt][c]{\scriptsize Guidance scale $w$}}
      \put(20.54,7.70){\makebox[0pt][c]{\small (a) SD1.5}}
      \put(39.94,4.90){\fontsize{5.5}{6.5}\selectfont photographer (smiling)}
      \put(55.94,2.88){\fontsize{5.5}{6.5}\selectfont painter (portrait)}
      \put(39.94,0.86){\fontsize{5.5}{6.5}\selectfont scientist (portrait)}
      \put(55.94,4.90){\fontsize{5.5}{6.5}\selectfont student (smiling)}
      \put(55.94,0.86){\fontsize{5.5}{6.5}\selectfont singer (smiling)}
      \put(39.94,2.88){\fontsize{5.5}{6.5}\selectfont photographer (face of)}
      \put(38.90,11.99){\makebox[0pt][c]{\tiny 1.5}} \put(45.71,11.99){\makebox[0pt][c]{\tiny 3.0}} \put(52.53,11.99){\makebox[0pt][c]{\tiny 4.5}} \put(59.34,11.99){\makebox[0pt][c]{\tiny 6.0}} \put(66.15,11.99){\makebox[0pt][c]{\tiny 7.5}}
      \put(52.53,9.85){\makebox[0pt][c]{\scriptsize Guidance scale $w$}}
      \put(52.53,7.70){\makebox[0pt][c]{\small (b) SD3}}
      \put(71.93,2.88){\fontsize{5}{6}\selectfont fashion designer (face of)}
      \put(88.35,2.88){\fontsize{5}{6}\selectfont intern (smiling)}
      \put(71.93,4.90){\fontsize{5}{6}\selectfont fashion designer (at work)}
      \put(88.35,0.86){\fontsize{5}{6}\selectfont intern (face of)}
      \put(88.35,4.90){\fontsize{5}{6}\selectfont salesperson (face of)}
      \put(71.93,0.86){\fontsize{5}{6}\selectfont salesperson (portrait)}
      \put(70.89,11.99){\makebox[0pt][c]{\tiny 1.0}} \put(77.70,11.99){\makebox[0pt][c]{\tiny 2.0}} \put(84.51,11.99){\makebox[0pt][c]{\tiny 3.0}} \put(91.32,11.99){\makebox[0pt][c]{\tiny 4.0}} \put(98.13,11.99){\makebox[0pt][c]{\tiny 5.0}}
      \put(84.51,9.85){\makebox[0pt][c]{\scriptsize Guidance scale $w$}}
      \put(84.51,7.70){\makebox[0pt][c]{\small (c) FLUX.2}}
      \put(4.92,13.23){\makebox[0pt][r]{\tiny 0}} \put(4.92,22.70){\makebox[0pt][r]{\tiny 50}} \put(4.92,32.17){\makebox[0pt][r]{\tiny 100}}
      \put(0.76,23.11){\rotatebox{90}{\makebox[0pt][c]{\scriptsize Female ratio (\%)}}}
    \end{overpic}
  \caption{\textbf{Occupations where guidance moves the female ratio toward parity.}}
  \label{fig:app_sd15_mitigating_occ}
\end{figure}

\endgroup
\newpage
\textbf{Guidance bias and condition alignment with PAG.}
To support the discussion of autoguidance in \cref{sec:fair_guidance}, we compare CFG, PAG~\citep{ahn2024self}, and \method on SD1.5.
PAG is evaluated without CFG over guidance scales from $1$ to $4$.
\cref{tab:app_sd15_pag} shows that \method reduces the bias range of CFG to near that of PAG while retaining alignment and quality close to CFG.
\begin{table}[H]
  \centering
  \footnotesize
  \renewcommand{\arraystretch}{0.95}
  \setlength{\tabcolsep}{2pt}
  \caption{\textbf{Comparison of CFG, PAG, and \method on SD1.5.}
  Bias in \%. Best in bold; target metrics shaded; $\pm$: std over seeds.}
  \label{tab:app_sd15_pag}
  \begin{tabular*}{\textwidth}{@{}p{\dimexpr\linewidth-302pt\relax} p{64pt} *{2}{>{\centering\arraybackslash}p{31pt}} >{\columncolor{targetshade}[\tabcolsep][0.5pt]\centering\arraybackslash}p{31pt} >{\columncolor{targetshade}[0.5pt][\tabcolsep]\centering\arraybackslash}p{40pt} >{\centering\arraybackslash}p{37pt} >{\centering\arraybackslash}p{40pt}@{}}
    \toprule
    & & \multicolumn{3}{c}{Bias (\%) $\downarrow$} & \multicolumn{3}{c}{Quality $\uparrow$} \\
    \cmidrule(lr){3-5} \cmidrule(lr){6-8}
    Model & Guidance & Avg. & Worst & Range & CLIP & Aesth. & Pick \\
    \midrule
    \multirow{3}{*}{SD1.5}
      & CFG & \pmv{28.0}{0.3} & \pmv{32.5}{0.5} & \pmv{10.7}{0.7} & \pmv{27.61}{0.02} & \pmvb{6.70}{0.01} & \pmvb{19.99}{0.02} \\
      & PAG & \pmvb{21.8}{0.5} & \pmvb{24.8}{0.6} & \pmvb{\phantom{0}6.4}{0.1} & \pmv{25.38}{0.06} & \pmv{6.44}{0.00} & \pmv{19.22}{0.02} \\
      & \textbf{\method (ours)} & \pmv{27.6}{0.6} & \pmv{30.8}{0.5} & \pmv{\phantom{0}7.2}{0.3} & \pmvb{27.62}{0.02} & \pmvb{6.70}{0.01} & \pmv{19.96}{0.01} \\
    \bottomrule
  \end{tabular*}
\end{table}

\clearpage
\section{Additional Qualitative Results}
\label{app:additional-qualitative}

\subsection{Text-to-image generation examples}
\label{app:qual-alpha}

\textbf{Effect of the guidance scale.}
\cref{fig:qual-t2i-scales} extends \cref{fig:qual_t2i} \textcolor{black}{across $w$} to FLUX.2, SD3, SD1.5, FT, and UCE: CFG drifts to the majority gender as $w$ grows, \textcolor{black}{whereas} \method keeps the minority gender.
For FT and UCE the rows coincide at $w{=}7.5$, and only \method carries that gender to the other scales.
\begin{figure}[H]
  \centering
  \setlength{\tabcolsep}{1.5pt}%
  \newlength{\qcw}\setlength{\qcw}{33pt}%
  \definecolor{gF}{RGB}{214,69,65}\definecolor{gM}{RGB}{46,117,182}\definecolor{gN}{RGB}{160,160,160}
  \setlength{\fboxsep}{0pt}\setlength{\fboxrule}{0.6pt}
  \newcommand{\qcropx}[3]{\fcolorbox{#1}{white}{\includegraphics[width=\qcw]{#2/#3}}}
  \newcommand{\rowlab}[1]{\makebox[9pt][c]{\raisebox{0.44\qcw}{\rotatebox[origin=c]{90}{\footnotesize #1}}}}
  \begin{tabular}{@{} c@{\hspace{3pt}} ccccc ccccc @{}}
    & \multicolumn{5}{c}{\footnotesize FLUX.2 -- Scientist (minority: \textcolor{gF}{female})} & \multicolumn{5}{c}{\footnotesize FLUX.2 -- Librarian (minority: \textcolor{gM}{male})} \\[1pt]
    & \makebox[0pt][c]{\scriptsize $w{=}1.0$} & \makebox[0pt][c]{\scriptsize $w{=}2.0$} & \makebox[0pt][c]{\scriptsize $w{=}3.0$} & \makebox[0pt][c]{\scriptsize $w{=}4.0$} & \makebox[0pt][c]{\scriptsize $w{=}5.0$}
    & \makebox[0pt][c]{\scriptsize $w{=}1.0$} & \makebox[0pt][c]{\scriptsize $w{=}2.0$} & \makebox[0pt][c]{\scriptsize $w{=}3.0$} & \makebox[0pt][c]{\scriptsize $w{=}4.0$} & \makebox[0pt][c]{\scriptsize $w{=}5.0$} \\[1pt]
    \rowlab{CFG}
      & \qcropx{gN}{qual_t2i_flux}{sci_cfg_w1.jpg} & \qcropx{gM}{qual_t2i_flux}{sci_cfg_w2.jpg} & \qcropx{gM}{qual_t2i_flux}{sci_cfg_w3.jpg} & \qcropx{gM}{qual_t2i_flux}{sci_cfg_w4.jpg} & \qcropx{gM}{qual_t2i_flux}{sci_cfg_w5.jpg}
      & \qcropx{gN}{qual_t2i_flux}{lib_cfg_w1.jpg} & \qcropx{gM}{qual_t2i_flux}{lib_cfg_w2.jpg} & \qcropx{gF}{qual_t2i_flux}{lib_cfg_w3.jpg} & \qcropx{gF}{qual_t2i_flux}{lib_cfg_w4.jpg} & \qcropx{gF}{qual_t2i_flux}{lib_cfg_w5.jpg} \\[1pt]
    \rowlab{\method}
      & \qcropx{gN}{qual_t2i_flux}{sci_sf_w1.jpg} & \qcropx{gF}{qual_t2i_flux}{sci_sf_w2.jpg} & \qcropx{gF}{qual_t2i_flux}{sci_sf_w3.jpg} & \qcropx{gF}{qual_t2i_flux}{sci_sf_w4.jpg} & \qcropx{gF}{qual_t2i_flux}{sci_sf_w5.jpg}
      & \qcropx{gN}{qual_t2i_flux}{lib_sf_w1.jpg} & \qcropx{gM}{qual_t2i_flux}{lib_sf_w2.jpg} & \qcropx{gM}{qual_t2i_flux}{lib_sf_w3.jpg} & \qcropx{gM}{qual_t2i_flux}{lib_sf_w4.jpg} & \qcropx{gM}{qual_t2i_flux}{lib_sf_w5.jpg} \\[1pt]
    & \multicolumn{5}{c}{\footnotesize SD3 -- Professor (minority: \textcolor{gF}{female})} & \multicolumn{5}{c}{\footnotesize SD3 -- Poet (minority: \textcolor{gM}{male})} \\[1pt]
    & \makebox[0pt][c]{\scriptsize $w{=}1.5$} & \makebox[0pt][c]{\scriptsize $w{=}3.0$} & \makebox[0pt][c]{\scriptsize $w{=}4.5$} & \makebox[0pt][c]{\scriptsize $w{=}6.0$} & \makebox[0pt][c]{\scriptsize $w{=}7.5$}
    & \makebox[0pt][c]{\scriptsize $w{=}1.5$} & \makebox[0pt][c]{\scriptsize $w{=}3.0$} & \makebox[0pt][c]{\scriptsize $w{=}4.5$} & \makebox[0pt][c]{\scriptsize $w{=}6.0$} & \makebox[0pt][c]{\scriptsize $w{=}7.5$} \\[1pt]
    \rowlab{CFG}
      & \qcropx{gF}{sd3_h100_20260924}{professor_cfg_w1.jpg} & \qcropx{gF}{sd3_h100_20260924}{professor_cfg_w2.jpg} & \qcropx{gM}{sd3_h100_20260924}{professor_cfg_w3.jpg} & \qcropx{gM}{sd3_h100_20260924}{professor_cfg_w4.jpg} & \qcropx{gM}{sd3_h100_20260924}{professor_cfg_w5.jpg}
      & \qcropx{gM}{sd3_h100_20260924}{poet_cfg_w1.jpg} & \qcropx{gM}{sd3_h100_20260924}{poet_cfg_w2.jpg} & \qcropx{gF}{sd3_h100_20260924}{poet_cfg_w3.jpg} & \qcropx{gF}{sd3_h100_20260924}{poet_cfg_w4.jpg} & \qcropx{gF}{sd3_h100_20260924}{poet_cfg_w5.jpg} \\[1pt]
    \rowlab{\method}
      & \qcropx{gF}{sd3_h100_20260924}{professor_sf_w1.jpg} & \qcropx{gF}{sd3_h100_20260924}{professor_sf_w2.jpg} & \qcropx{gF}{sd3_h100_20260924}{professor_sf_w3.jpg} & \qcropx{gF}{sd3_h100_20260924}{professor_sf_w4.jpg} & \qcropx{gF}{sd3_h100_20260924}{professor_sf_w5.jpg}
      & \qcropx{gM}{sd3_h100_20260924}{poet_sf_w1.jpg} & \qcropx{gM}{sd3_h100_20260924}{poet_sf_w2.jpg} & \qcropx{gM}{sd3_h100_20260924}{poet_sf_w3.jpg} & \qcropx{gM}{sd3_h100_20260924}{poet_sf_w4.jpg} & \qcropx{gM}{sd3_h100_20260924}{poet_sf_w5.jpg} \\[1pt]
    & \multicolumn{5}{c}{\footnotesize SD1.5 -- Chef (minority: \textcolor{gF}{female})} & \multicolumn{5}{c}{\footnotesize SD1.5 -- Athlete (minority: \textcolor{gM}{male})} \\[1pt]
    & \makebox[0pt][c]{\scriptsize $w{=}2.5$} & \makebox[0pt][c]{\scriptsize $w{=}5.0$} & \makebox[0pt][c]{\scriptsize $w{=}7.5$} & \makebox[0pt][c]{\scriptsize $w{=}10.0$} & \makebox[0pt][c]{\scriptsize $w{=}12.5$}
    & \makebox[0pt][c]{\scriptsize $w{=}2.5$} & \makebox[0pt][c]{\scriptsize $w{=}5.0$} & \makebox[0pt][c]{\scriptsize $w{=}7.5$} & \makebox[0pt][c]{\scriptsize $w{=}10.0$} & \makebox[0pt][c]{\scriptsize $w{=}12.5$} \\[1pt]
    \rowlab{CFG}
      & \qcropx{gF}{qual_t2i_sd15}{chef_cfg_w1.jpg} & \qcropx{gF}{qual_t2i_sd15}{chef_cfg_w2.jpg} & \qcropx{gM}{qual_t2i_sd15}{chef_cfg_w3.jpg} & \qcropx{gM}{qual_t2i_sd15}{chef_cfg_w4.jpg} & \qcropx{gM}{qual_t2i_sd15}{chef_cfg_w5.jpg}
      & \qcropx{gM}{qual_t2i_sd15}{athl_cfg_w1.jpg} & \qcropx{gM}{qual_t2i_sd15}{athl_cfg_w2.jpg} & \qcropx{gM}{qual_t2i_sd15}{athl_cfg_w3.jpg} & \qcropx{gF}{qual_t2i_sd15}{athl_cfg_w4.jpg} & \qcropx{gF}{qual_t2i_sd15}{athl_cfg_w5.jpg} \\[1pt]
    \rowlab{\method}
      & \qcropx{gF}{qual_t2i_sd15}{chef_sf_w1.jpg} & \qcropx{gF}{qual_t2i_sd15}{chef_sf_w2.jpg} & \qcropx{gF}{qual_t2i_sd15}{chef_sf_w3.jpg} & \qcropx{gF}{qual_t2i_sd15}{chef_sf_w4.jpg} & \qcropx{gF}{qual_t2i_sd15}{chef_sf_w5.jpg}
      & \qcropx{gM}{qual_t2i_sd15}{athl_sf_w1.jpg} & \qcropx{gM}{qual_t2i_sd15}{athl_sf_w2.jpg} & \qcropx{gM}{qual_t2i_sd15}{athl_sf_w3.jpg} & \qcropx{gM}{qual_t2i_sd15}{athl_sf_w4.jpg} & \qcropx{gM}{qual_t2i_sd15}{athl_sf_w5.jpg} \\[1pt]
    & \multicolumn{5}{c}{\footnotesize FT -- Technician (minority: \textcolor{gM}{male})} & \multicolumn{5}{c}{\footnotesize FT -- Pilot (minority: \textcolor{gM}{male})} \\[1pt]
    & \makebox[0pt][c]{\scriptsize $w{=}2.5$} & \makebox[0pt][c]{\scriptsize $w{=}5.0$} & \makebox[0pt][c]{\scriptsize $w{=}7.5$} & \makebox[0pt][c]{\scriptsize $w{=}10.0$} & \makebox[0pt][c]{\scriptsize $w{=}12.5$}
    & \makebox[0pt][c]{\scriptsize $w{=}2.5$} & \makebox[0pt][c]{\scriptsize $w{=}5.0$} & \makebox[0pt][c]{\scriptsize $w{=}7.5$} & \makebox[0pt][c]{\scriptsize $w{=}10.0$} & \makebox[0pt][c]{\scriptsize $w{=}12.5$} \\[1pt]
    \rowlab{CFG}
      & \qcropx{gM}{qual_t2i_ft}{tech_cfg_w1.jpg} & \qcropx{gM}{qual_t2i_ft}{tech_cfg_w2.jpg} & \qcropx{gN}{qual_t2i_ft}{tech_cfg_w3.jpg} & \qcropx{gF}{qual_t2i_ft}{tech_cfg_w4.jpg} & \qcropx{gF}{qual_t2i_ft}{tech_cfg_w5.jpg}
      & \qcropx{gM}{qual_t2i_ft}{pilot_cfg_w1.jpg} & \qcropx{gM}{qual_t2i_ft}{pilot_cfg_w2.jpg} & \qcropx{gN}{qual_t2i_ft}{pilot_cfg_w3.jpg} & \qcropx{gF}{qual_t2i_ft}{pilot_cfg_w4.jpg} & \qcropx{gF}{qual_t2i_ft}{pilot_cfg_w5.jpg} \\[1pt]
    \rowlab{\method}
      & \qcropx{gM}{qual_t2i_ft}{tech_sf_w1.jpg} & \qcropx{gM}{qual_t2i_ft}{tech_sf_w2.jpg} & \qcropx{gN}{qual_t2i_ft}{tech_sf_w3.jpg} & \qcropx{gM}{qual_t2i_ft}{tech_sf_w4.jpg} & \qcropx{gM}{qual_t2i_ft}{tech_sf_w5.jpg}
      & \qcropx{gM}{qual_t2i_ft}{pilot_sf_w1.jpg} & \qcropx{gM}{qual_t2i_ft}{pilot_sf_w2.jpg} & \qcropx{gN}{qual_t2i_ft}{pilot_sf_w3.jpg} & \qcropx{gM}{qual_t2i_ft}{pilot_sf_w4.jpg} & \qcropx{gM}{qual_t2i_ft}{pilot_sf_w5.jpg} \\[1pt]
    & \multicolumn{5}{c}{\footnotesize UCE -- Poet (minority: \textcolor{gF}{female})} & \multicolumn{5}{c}{\footnotesize UCE -- Firefighter (minority: \textcolor{gM}{male})} \\[1pt]
    & \makebox[0pt][c]{\scriptsize $w{=}2.5$} & \makebox[0pt][c]{\scriptsize $w{=}5.0$} & \makebox[0pt][c]{\scriptsize $w{=}7.5$} & \makebox[0pt][c]{\scriptsize $w{=}10.0$} & \makebox[0pt][c]{\scriptsize $w{=}12.5$}
    & \makebox[0pt][c]{\scriptsize $w{=}2.5$} & \makebox[0pt][c]{\scriptsize $w{=}5.0$} & \makebox[0pt][c]{\scriptsize $w{=}7.5$} & \makebox[0pt][c]{\scriptsize $w{=}10.0$} & \makebox[0pt][c]{\scriptsize $w{=}12.5$} \\[1pt]
    \rowlab{CFG}
      & \qcropx{gF}{qual_t2i_uce}{poet_cfg_w1.jpg} & \qcropx{gF}{qual_t2i_uce}{poet_cfg_w2.jpg} & \qcropx{gN}{qual_t2i_uce}{poet_cfg_w3.jpg} & \qcropx{gF}{qual_t2i_uce}{poet_cfg_w4.jpg} & \qcropx{gM}{qual_t2i_uce}{poet_cfg_w5.jpg}
      & \qcropx{gM}{qual_t2i_uce}{fire_cfg_w1.jpg} & \qcropx{gM}{qual_t2i_uce}{fire_cfg_w2.jpg} & \qcropx{gN}{qual_t2i_uce}{fire_cfg_w3.jpg} & \qcropx{gM}{qual_t2i_uce}{fire_cfg_w4.jpg} & \qcropx{gF}{qual_t2i_uce}{fire_cfg_w5.jpg} \\[1pt]
    \rowlab{\method}
      & \qcropx{gF}{qual_t2i_uce}{poet_sf_w1.jpg} & \qcropx{gF}{qual_t2i_uce}{poet_sf_w2.jpg} & \qcropx{gN}{qual_t2i_uce}{poet_sf_w3.jpg} & \qcropx{gF}{qual_t2i_uce}{poet_sf_w4.jpg} & \qcropx{gF}{qual_t2i_uce}{poet_sf_w5.jpg}
      & \qcropx{gM}{qual_t2i_uce}{fire_sf_w1.jpg} & \qcropx{gM}{qual_t2i_uce}{fire_sf_w2.jpg} & \qcropx{gN}{qual_t2i_uce}{fire_sf_w3.jpg} & \qcropx{gM}{qual_t2i_uce}{fire_sf_w4.jpg} & \qcropx{gM}{qual_t2i_uce}{fire_sf_w5.jpg} \\[1pt]
  \end{tabular}
  \caption{\textbf{Qualitative examples across guidance scales.}
  Minority groups are defined per prompt using each model's output distribution at $w=w_{\mathrm{ref}}$.
  \textcolor{gF}{Red} and \textcolor{gM}{blue} borders indicate female and male.
  \textcolor{black}{\textcolor{gN}{Gray} borders mark reference-scale images, whose gender \method preserves across guidance scales.}}
  \label{fig:qual-t2i-scales}
\end{figure}

\textbf{Effect of $\alpha$.}
\cref{fig:qual-alpha-all} moves $\alpha$ along the search grid \textcolor{black}{at fixed $w$} ($\alpha{=}0$ is CFG, $^{\star}$ the searched $\alpha^\star$).
\textcolor{black}{The sign of $\alpha$ selects the favored gender, while its magnitude controls the strength of the shift.}
\begin{figure}[H]
  \centering
  \setlength{\tabcolsep}{1.5pt}%
  \newlength{\qaw}\setlength{\qaw}{33pt}%
  \definecolor{gF}{RGB}{214,69,65}\definecolor{gM}{RGB}{46,117,182}\definecolor{gN}{RGB}{160,160,160}
  \setlength{\fboxsep}{0pt}\setlength{\fboxrule}{0.6pt}
  \newcommand{\qcropa}[3]{\fcolorbox{#1}{white}{\includegraphics[width=\qaw]{#2/#3}}}
  \newcommand{\ahd}[1]{\makebox[0pt][c]{\scriptsize $\alpha{=}#1$}}%
  \newcommand{\ahds}[1]{\makebox[0pt][c]{\scriptsize $\alpha{=}#1^{\star}$}}
  \newcommand{\ahdt}[1]{\makebox[0pt][c]{\tiny $\alpha{=}#1$}}%
  \newcommand{\ahdst}[1]{\makebox[0pt][c]{\tiny $\alpha{=}#1^{\star}$}}
  \begin{tabular}{@{} c@{\hspace{3pt}} ccccc ccccc @{}}
    \makebox[9pt]{} & \multicolumn{5}{c}{\footnotesize FLUX.2 -- Doctor (minority: \textcolor{gF}{female})} & \multicolumn{5}{c}{\footnotesize FLUX.2 -- Librarian (minority: \textcolor{gM}{male})} \\[1pt]
    \makebox[9pt]{} & \makebox[0pt][c]{\scriptsize CFG} &  &  & \makebox[0pt][c]{\scriptsize \method} & 
    & \makebox[0pt][c]{\scriptsize CFG} &  &  & \makebox[0pt][c]{\scriptsize \method} &  \\[0pt]
    \makebox[9pt]{} & \ahdt{0} & \ahdt{{-}2000} & \ahdt{{-}2750} & \ahdst{{-}4250} & \ahdt{{-}8500}
    & \ahdt{0} & \ahdt{{+}1250} & \ahdt{{+}2000} & \ahdst{{+}3250} & \ahdt{{+}7000} \\[1pt]
      \makebox[9pt]{} & \qcropa{gM}{qual_alpha_flux}{doctor_r1_c1.jpg} & \qcropa{gM}{qual_alpha_flux}{doctor_r1_c2.jpg} & \qcropa{gF}{qual_alpha_flux}{doctor_r1_c3.jpg} & \qcropa{gF}{qual_alpha_flux}{doctor_r1_c4.jpg} & \qcropa{gF}{qual_alpha_flux}{doctor_r1_c5.jpg}
      & \qcropa{gF}{qual_alpha_flux}{librarian_r1_c1.jpg} & \qcropa{gF}{qual_alpha_flux}{librarian_r1_c2.jpg} & \qcropa{gF}{qual_alpha_flux}{librarian_r1_c3.jpg} & \qcropa{gM}{qual_alpha_flux}{librarian_r1_c4.jpg} & \qcropa{gM}{qual_alpha_flux}{librarian_r1_c5.jpg} \\[1pt]
      \makebox[9pt]{} & \qcropa{gM}{qual_alpha_flux}{doctor_r2_c1.jpg} & \qcropa{gM}{qual_alpha_flux}{doctor_r2_c2.jpg} & \qcropa{gF}{qual_alpha_flux}{doctor_r2_c3.jpg} & \qcropa{gF}{qual_alpha_flux}{doctor_r2_c4.jpg} & \qcropa{gF}{qual_alpha_flux}{doctor_r2_c5.jpg}
      & \qcropa{gF}{qual_alpha_flux}{librarian_r2_c1.jpg} & \qcropa{gF}{qual_alpha_flux}{librarian_r2_c2.jpg} & \qcropa{gF}{qual_alpha_flux}{librarian_r2_c3.jpg} & \qcropa{gM}{qual_alpha_flux}{librarian_r2_c4.jpg} & \qcropa{gM}{qual_alpha_flux}{librarian_r2_c5.jpg} \\[1pt]
    \makebox[9pt]{} & \multicolumn{5}{c}{\footnotesize SD3 -- Poet (minority: \textcolor{gM}{male})} & \multicolumn{5}{c}{\footnotesize SD3 -- Professor (minority: \textcolor{gF}{female})} \\[1pt]
    \makebox[9pt]{} & \makebox[0pt][c]{\scriptsize CFG} & & \makebox[0pt][c]{\scriptsize \method} & & & & \makebox[0pt][c]{\scriptsize \method} & & & \makebox[0pt][c]{\scriptsize CFG} \\[0pt]
    \makebox[9pt]{} & \ahd{0} & \ahd{25} & \ahds{40} & \ahd{50} & \ahd{100} & \ahd{{-}150} & \ahds{{-}125} & \ahd{{-}100} & \ahd{{-}50} & \ahd{0} \\[1pt]
    \makebox[9pt]{} & \qcropa{gF}{sd3_h100_20260924}{alpha_poet_r1_c1.jpg} & \qcropa{gF}{sd3_h100_20260924}{alpha_poet_r1_c2.jpg} & \qcropa{gF}{sd3_h100_20260924}{alpha_poet_r1_c3.jpg} & \qcropa{gF}{sd3_h100_20260924}{alpha_poet_r1_c4.jpg} & \qcropa{gM}{sd3_h100_20260924}{alpha_poet_r1_c5.jpg} & \qcropa{gF}{sd3_h100_20260924}{alpha_professor_r1_c1.jpg} & \qcropa{gF}{sd3_h100_20260924}{alpha_professor_r1_c2.jpg} & \qcropa{gF}{sd3_h100_20260924}{alpha_professor_r1_c3.jpg} & \qcropa{gF}{sd3_h100_20260924}{alpha_professor_r1_c4.jpg} & \qcropa{gM}{sd3_h100_20260924}{alpha_professor_r1_c5.jpg} \\[1pt]
    \makebox[9pt]{} & \qcropa{gF}{sd3_h100_20260924}{alpha_poet_r2_c1.jpg} & \qcropa{gM}{sd3_h100_20260924}{alpha_poet_r2_c2.jpg} & \qcropa{gM}{sd3_h100_20260924}{alpha_poet_r2_c3.jpg} & \qcropa{gM}{sd3_h100_20260924}{alpha_poet_r2_c4.jpg} & \qcropa{gM}{sd3_h100_20260924}{alpha_poet_r2_c5.jpg} & \qcropa{gF}{sd3_h100_20260924}{alpha_professor_r2_c1.jpg} & \qcropa{gF}{sd3_h100_20260924}{alpha_professor_r2_c2.jpg} & \qcropa{gF}{sd3_h100_20260924}{alpha_professor_r2_c3.jpg} & \qcropa{gM}{sd3_h100_20260924}{alpha_professor_r2_c4.jpg} & \qcropa{gM}{sd3_h100_20260924}{alpha_professor_r2_c5.jpg} \\[1pt]
    \makebox[9pt]{} & \multicolumn{5}{c}{\footnotesize SD1.5 -- Journalist (minority: \textcolor{gM}{male})} & \multicolumn{5}{c}{\footnotesize SD1.5 -- Athlete (minority: \textcolor{gM}{male})} \\[1pt]
     \makebox[9pt]{} & &  & \makebox[0pt][c]{\scriptsize CFG} & \makebox[0pt][c]{\scriptsize \method} & 
    &  &  & \makebox[0pt][c]{\scriptsize CFG} & \makebox[0pt][c]{\scriptsize \method} &  \\[0pt]
    \makebox[9pt]{} & \ahd{{-}10} & \ahd{{-}5} & \ahd{0} & \ahds{{+}5} & \ahd{{+}10}
    & \ahd{{-}10} & \ahd{{-}5} & \ahd{0} & \ahds{{+}5} & \ahd{{+}10} \\[1pt]
      \makebox[9pt]{} & \qcropa{gF}{qual_alpha_sd15}{journalist_r1_c1.jpg} & \qcropa{gF}{qual_alpha_sd15}{journalist_r1_c2.jpg} & \qcropa{gF}{qual_alpha_sd15}{journalist_r1_c3.jpg} & \qcropa{gM}{qual_alpha_sd15}{journalist_r1_c4.jpg} & \qcropa{gM}{qual_alpha_sd15}{journalist_r1_c5.jpg}
      & \qcropa{gF}{qual_alpha_sd15}{athlete_r1_c1.jpg} & \qcropa{gF}{qual_alpha_sd15}{athlete_r1_c2.jpg} & \qcropa{gF}{qual_alpha_sd15}{athlete_r1_c3.jpg} & \qcropa{gM}{qual_alpha_sd15}{athlete_r1_c4.jpg} & \qcropa{gM}{qual_alpha_sd15}{athlete_r1_c5.jpg} \\[1pt]
      \makebox[9pt]{} & \qcropa{gF}{qual_alpha_sd15}{journalist_r2_c1.jpg} & \qcropa{gF}{qual_alpha_sd15}{journalist_r2_c2.jpg} & \qcropa{gF}{qual_alpha_sd15}{journalist_r2_c3.jpg} & \qcropa{gM}{qual_alpha_sd15}{journalist_r2_c4.jpg} & \qcropa{gM}{qual_alpha_sd15}{journalist_r2_c5.jpg}
      & \qcropa{gF}{qual_alpha_sd15}{athlete_r2_c1.jpg} & \qcropa{gF}{qual_alpha_sd15}{athlete_r2_c2.jpg} & \qcropa{gF}{qual_alpha_sd15}{athlete_r2_c3.jpg} & \qcropa{gM}{qual_alpha_sd15}{athlete_r2_c4.jpg} & \qcropa{gM}{qual_alpha_sd15}{athlete_r2_c5.jpg} \\[1pt]
    \makebox[9pt]{} & \multicolumn{5}{c}{\footnotesize FT -- Pilot (minority: \textcolor{gM}{male})} & \multicolumn{5}{c}{\footnotesize FT -- Technician (minority: \textcolor{gM}{male})} \\[1pt]
    \makebox[9pt]{} & \makebox[0pt][c]{\scriptsize CFG} &  &  & \makebox[0pt][c]{\scriptsize \method} & 
    & \makebox[0pt][c]{\scriptsize CFG} &  & \makebox[0pt][c]{\scriptsize \method} &  &  \\[0pt]
    \makebox[9pt]{} & \ahd{0} & \ahd{2.5} & \ahd{7.5} & \ahds{10} & \ahd{15}
    & \ahd{0} & \ahd{2.5} & \ahds{7.5} & \ahd{12.5} & \ahd{15} \\[1pt]
      \makebox[9pt]{} & \qcropa{gF}{qual_alpha_ft}{pilot_r1_c1.jpg} & \qcropa{gF}{qual_alpha_ft}{pilot_r1_c2.jpg} & \qcropa{gF}{qual_alpha_ft}{pilot_r1_c3.jpg} & \qcropa{gM}{qual_alpha_ft}{pilot_r1_c4.jpg} & \qcropa{gM}{qual_alpha_ft}{pilot_r1_c5.jpg}
      & \qcropa{gF}{qual_alpha_ft}{technician_r1_c1.jpg} & \qcropa{gF}{qual_alpha_ft}{technician_r1_c2.jpg} & \qcropa{gM}{qual_alpha_ft}{technician_r1_c3.jpg} & \qcropa{gM}{qual_alpha_ft}{technician_r1_c4.jpg} & \qcropa{gM}{qual_alpha_ft}{technician_r1_c5.jpg} \\[1pt]
      \makebox[9pt]{} & \qcropa{gF}{qual_alpha_ft}{pilot_r2_c1.jpg} & \qcropa{gF}{qual_alpha_ft}{pilot_r2_c2.jpg} & \qcropa{gM}{qual_alpha_ft}{pilot_r2_c3.jpg} & \qcropa{gM}{qual_alpha_ft}{pilot_r2_c4.jpg} & \qcropa{gM}{qual_alpha_ft}{pilot_r2_c5.jpg}
      & \qcropa{gF}{qual_alpha_ft}{technician_r2_c1.jpg} & \qcropa{gF}{qual_alpha_ft}{technician_r2_c2.jpg} & \qcropa{gF}{qual_alpha_ft}{technician_r2_c3.jpg} & \qcropa{gM}{qual_alpha_ft}{technician_r2_c4.jpg} & \qcropa{gM}{qual_alpha_ft}{technician_r2_c5.jpg} \\[1pt]
    \makebox[9pt]{} & \multicolumn{5}{c}{\footnotesize UCE -- Pilot (minority: \textcolor{gM}{male})} & \multicolumn{5}{c}{\footnotesize UCE -- Poet (minority: \textcolor{gF}{female})} \\[1pt]
    \makebox[9pt]{} & \makebox[0pt][c]{\scriptsize CFG} & \makebox[0pt][c]{\scriptsize \method} &  &  & 
    &  &  &  & \makebox[0pt][c]{\scriptsize \method} & \makebox[0pt][c]{\scriptsize CFG} \\[0pt]
    \makebox[9pt]{} & \ahd{0} & \ahds{2.5} & \ahd{7.5} & \ahd{12.5} & \ahd{15}
    & \ahd{{-}15} & \ahd{{-}12.5} & \ahd{{-}7.5} & \ahds{{-}2.5} & \ahd{0} \\[1pt]
      \makebox[9pt]{} & \qcropa{gF}{qual_alpha_uce}{pilot_r1_c1.jpg} & \qcropa{gM}{qual_alpha_uce}{pilot_r1_c2.jpg} & \qcropa{gM}{qual_alpha_uce}{pilot_r1_c3.jpg} & \qcropa{gM}{qual_alpha_uce}{pilot_r1_c4.jpg} & \qcropa{gM}{qual_alpha_uce}{pilot_r1_c5.jpg}
      & \qcropa{gF}{qual_alpha_uce}{poet_r1_c1.jpg} & \qcropa{gF}{qual_alpha_uce}{poet_r1_c2.jpg} & \qcropa{gF}{qual_alpha_uce}{poet_r1_c3.jpg} & \qcropa{gF}{qual_alpha_uce}{poet_r1_c4.jpg} & \qcropa{gM}{qual_alpha_uce}{poet_r1_c5.jpg} \\[1pt]
      \makebox[9pt]{} & \qcropa{gF}{qual_alpha_uce}{pilot_r2_c1.jpg} & \qcropa{gF}{qual_alpha_uce}{pilot_r2_c2.jpg} & \qcropa{gM}{qual_alpha_uce}{pilot_r2_c3.jpg} & \qcropa{gM}{qual_alpha_uce}{pilot_r2_c4.jpg} & \qcropa{gM}{qual_alpha_uce}{pilot_r2_c5.jpg}
      & \qcropa{gF}{qual_alpha_uce}{poet_r2_c1.jpg} & \qcropa{gF}{qual_alpha_uce}{poet_r2_c2.jpg} & \qcropa{gM}{qual_alpha_uce}{poet_r2_c3.jpg} & \qcropa{gM}{qual_alpha_uce}{poet_r2_c4.jpg} & \qcropa{gM}{qual_alpha_uce}{poet_r2_c5.jpg} \\
\end{tabular}
  \caption{\textbf{Qualitative examples across $\alpha$.}
  Same seed across $\alpha$ at a fixed guidance scale ($w{=}5.0$ on FLUX.2, \textcolor{black}{$w{=}7.5$ on SD3}, \textcolor{black}{$w{=}12.5$ on SD1.5, FT, and UCE}); columns list each occupation's $\alpha$ grid and $^{\star}$ marks the searched $\alpha^\star$\textcolor{black}{, which is \method (Search)}. Border: \textcolor{gF}{female}, \textcolor{gM}{male}, \textcolor{gN}{gray} = low classifier confidence.}
  \label{fig:qual-alpha-all}
\end{figure}

\subsection{Class-conditional generation examples}
\label{app:qual-cg}

\cref{fig:qual-cg} compares CG, RW, and \method on CelebA with \emph{Blond Hair} and \emph{Smiling} conditions, supplementing \cref{fig:cg_combined}\textcolor{black}{.}
Under CG, increasing $w$ shifts generated images toward the majority group (\eg, female for \emph{Blond Hair}).
RW reduces this tendency but does not fully suppress it at high scales.
\begin{figure}[H]
  \centering
  \setlength{\tabcolsep}{1pt}
  \newcommand{\ccrop}[1]{\includegraphics[width=0.09\linewidth]{cg_qual_crops/#1}}
  \newcommand{\rowlab}[1]{\raisebox{0.038\linewidth}[0pt][0pt]{\rotatebox[origin=c]{90}{\scriptsize #1}}}
  \begin{tabular}{@{} c ccccc @{\hspace{6pt}} ccccc @{}}
    & \multicolumn{10}{c}{\small\textit{Condition: Blond Hair}} \\[2pt]
    & {\footnotesize $w{=}0$} & {\footnotesize $w{=}1$} & {\footnotesize $w{=}3$} & {\footnotesize $w{=}5$} & {\footnotesize $w{=}9$}
    & {\footnotesize $w{=}0$} & {\footnotesize $w{=}1$} & {\footnotesize $w{=}3$} & {\footnotesize $w{=}5$} & {\footnotesize $w{=}9$} \\[1pt]
    \rowlab{CG}
      & \ccrop{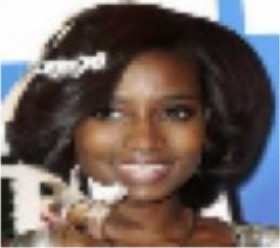} & \ccrop{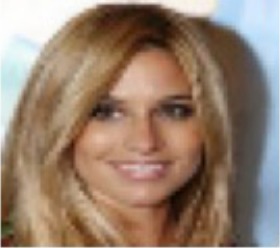} & \ccrop{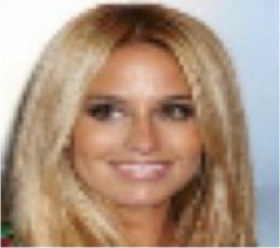} & \ccrop{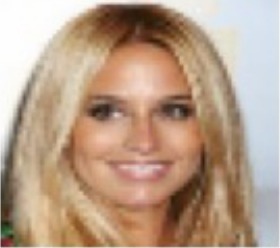} & \ccrop{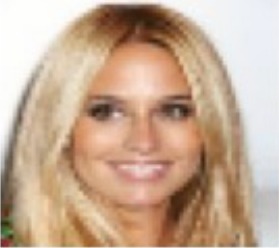}
      & \ccrop{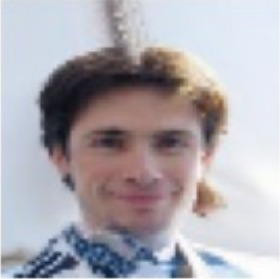} & \ccrop{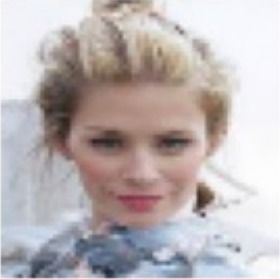} & \ccrop{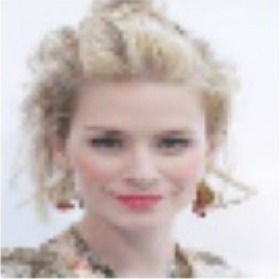} & \ccrop{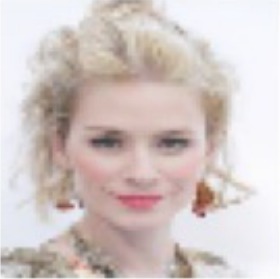} & \ccrop{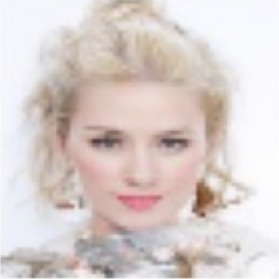} \\[1pt]
    \rowlab{RW}
      & \ccrop{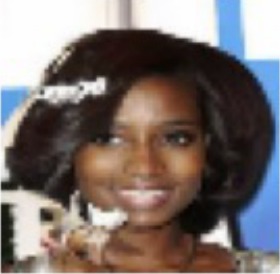} & \ccrop{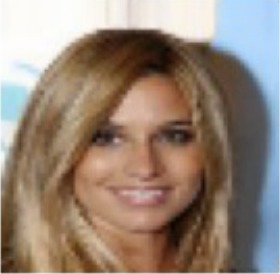} & \ccrop{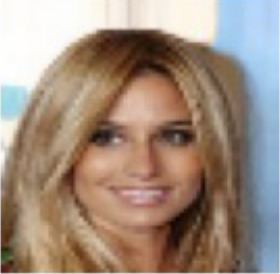} & \ccrop{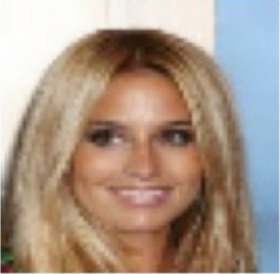} & \ccrop{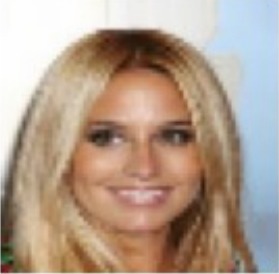}
      & \ccrop{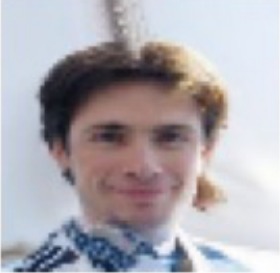} & \ccrop{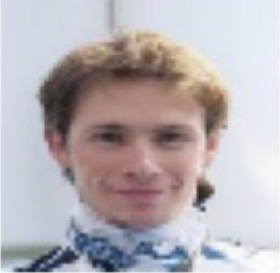} & \ccrop{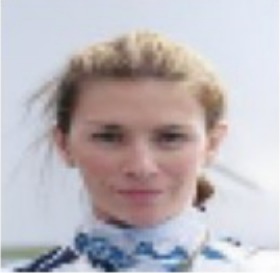} & \ccrop{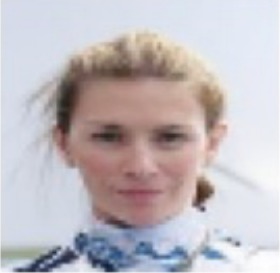} & \ccrop{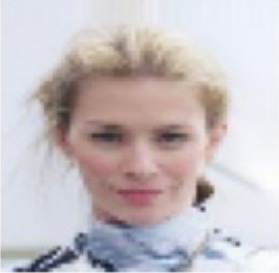} \\[1pt]
    \rowlab{\method}
      & \ccrop{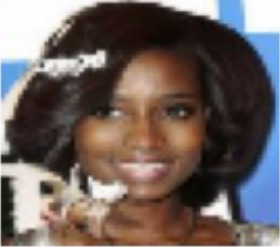} & \ccrop{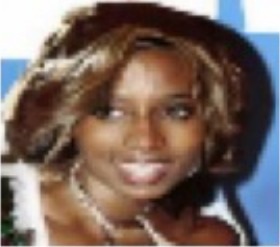} & \ccrop{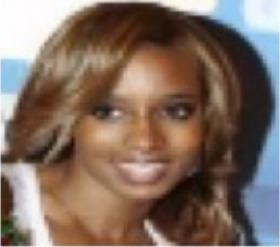} & \ccrop{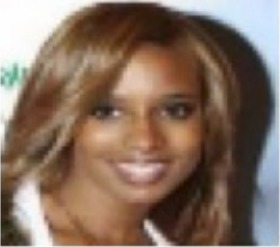} & \ccrop{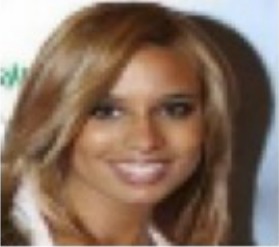}
      & \ccrop{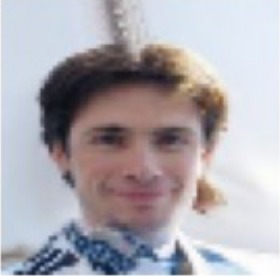} & \ccrop{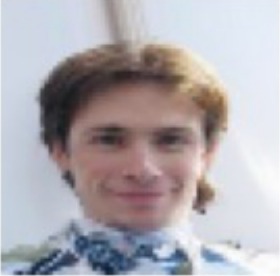} & \ccrop{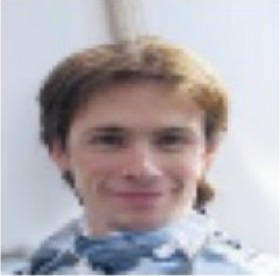} & \ccrop{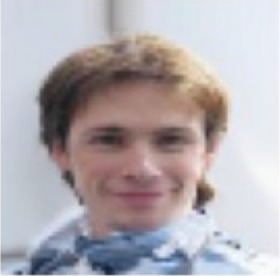} & \ccrop{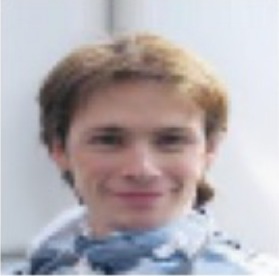} \\[6pt]
    \rowlab{CG}
      & \ccrop{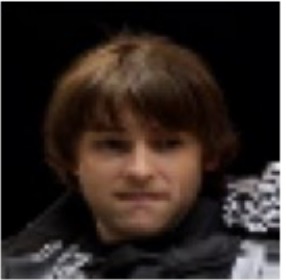} & \ccrop{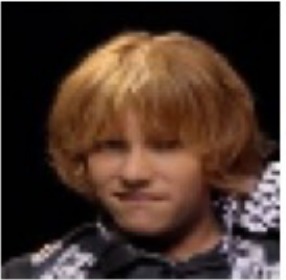} & \ccrop{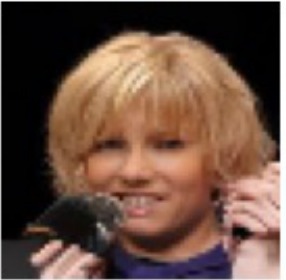} & \ccrop{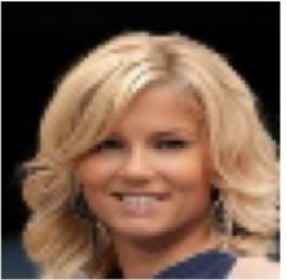} & \ccrop{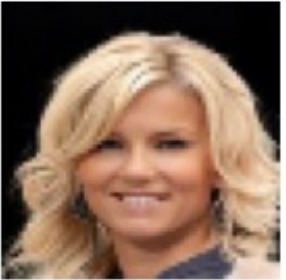}
      & \ccrop{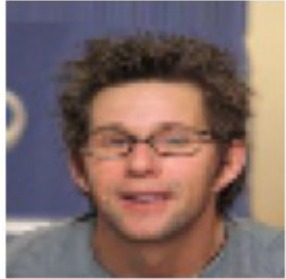} & \ccrop{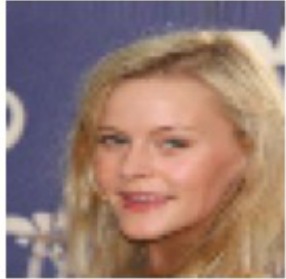} & \ccrop{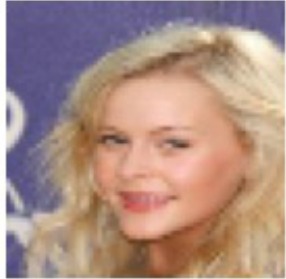} & \ccrop{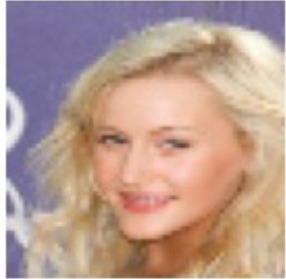} & \ccrop{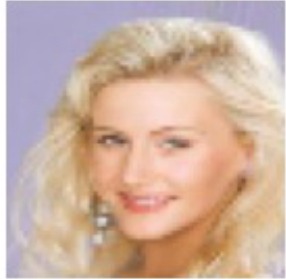} \\[1pt]
    \rowlab{RW}
      & \ccrop{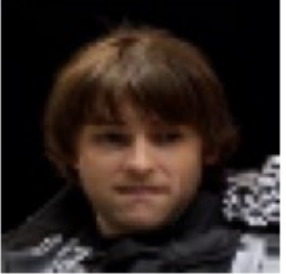} & \ccrop{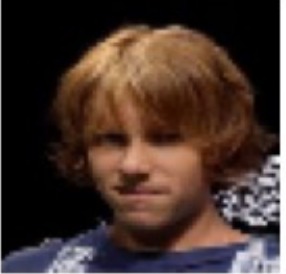} & \ccrop{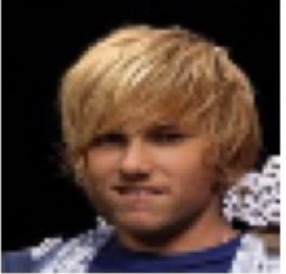} & \ccrop{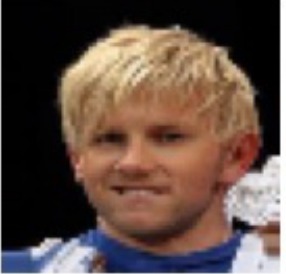} & \ccrop{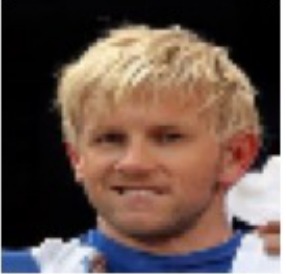}
      & \ccrop{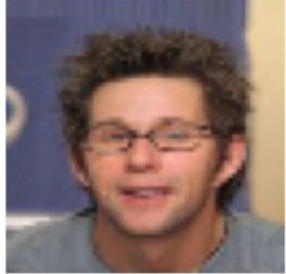} & \ccrop{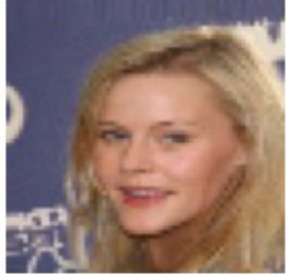} & \ccrop{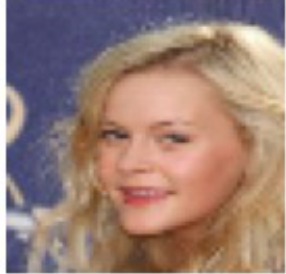} & \ccrop{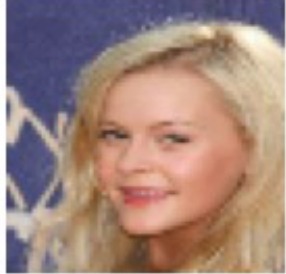} & \ccrop{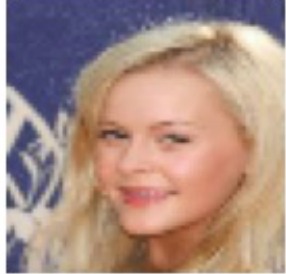} \\[1pt]
    \rowlab{\method}
      & \ccrop{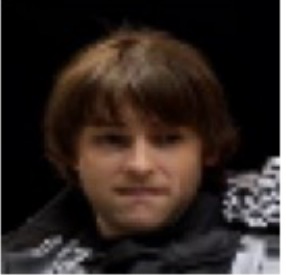} & \ccrop{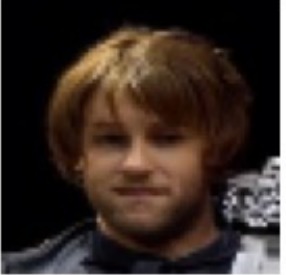} & \ccrop{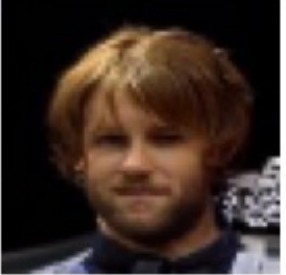} & \ccrop{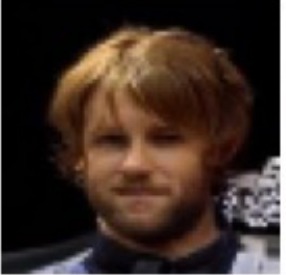} & \ccrop{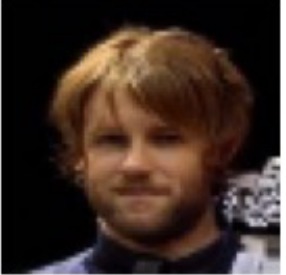}
      & \ccrop{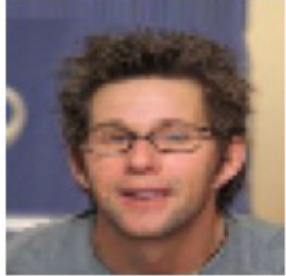} & \ccrop{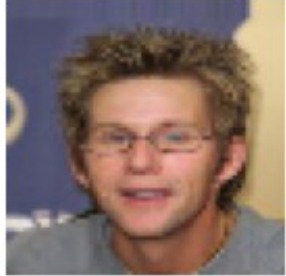} & \ccrop{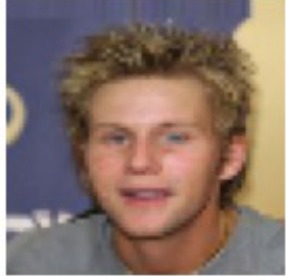} & \ccrop{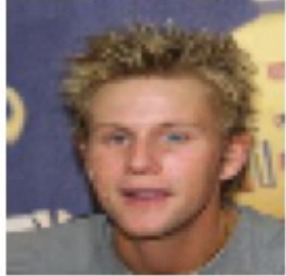} & \ccrop{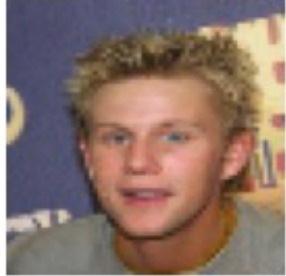} \\[2pt]
    & \multicolumn{10}{c}{\small\textit{Condition: Smiling}} \\[2pt]
    & {\footnotesize $w{=}0$} & {\footnotesize $w{=}1$} & {\footnotesize $w{=}3$} & {\footnotesize $w{=}5$} & {\footnotesize $w{=}9$}
    & {\footnotesize $w{=}0$} & {\footnotesize $w{=}1$} & {\footnotesize $w{=}3$} & {\footnotesize $w{=}5$} & {\footnotesize $w{=}9$} \\[1pt]
    \rowlab{CG}
      & \ccrop{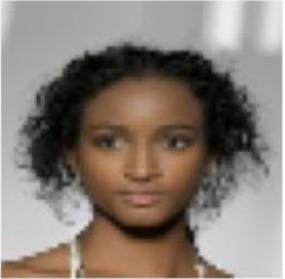} & \ccrop{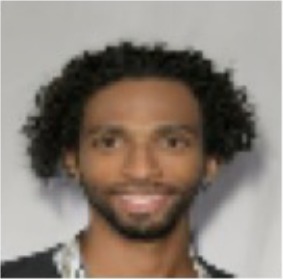} & \ccrop{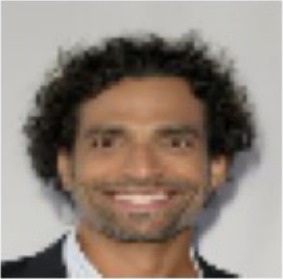} & \ccrop{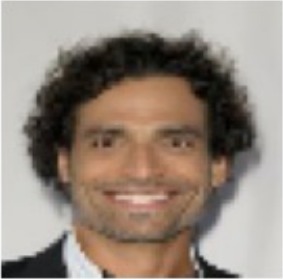} & \ccrop{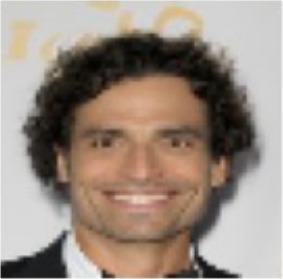}
      & \ccrop{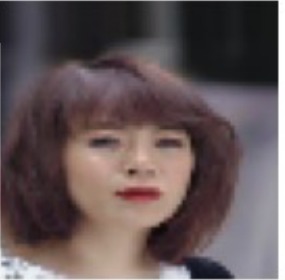} & \ccrop{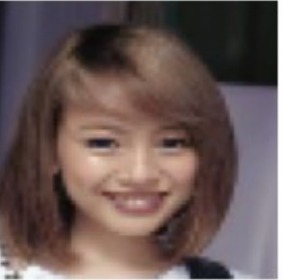} & \ccrop{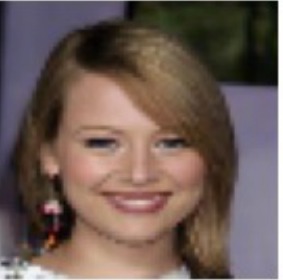} & \ccrop{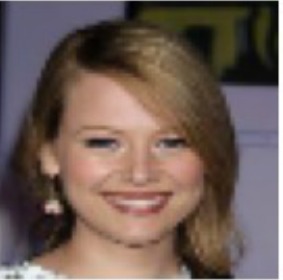} & \ccrop{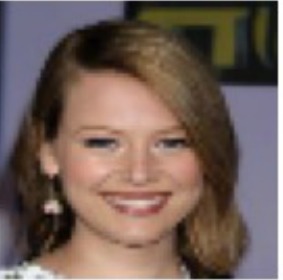} \\[1pt]
    \rowlab{RW}
      & \ccrop{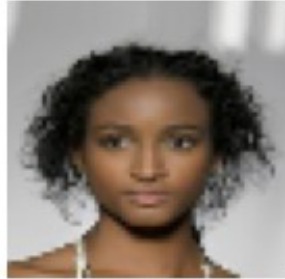} & \ccrop{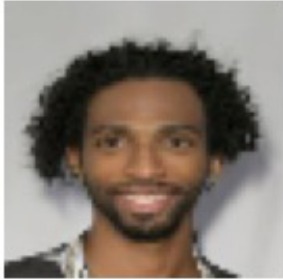} & \ccrop{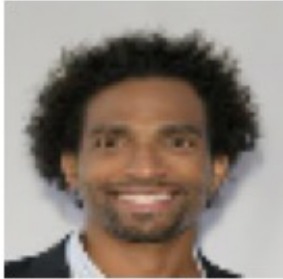} & \ccrop{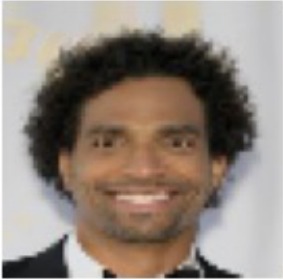} & \ccrop{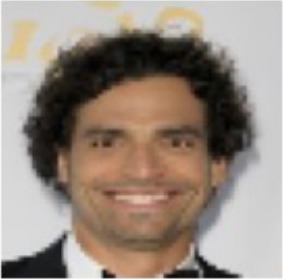}
      & \ccrop{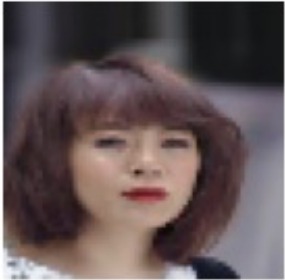} & \ccrop{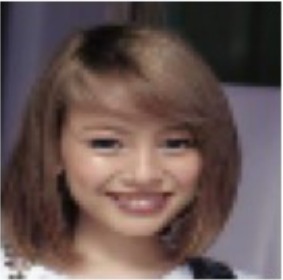} & \ccrop{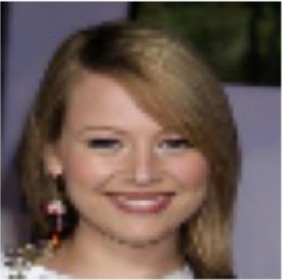} & \ccrop{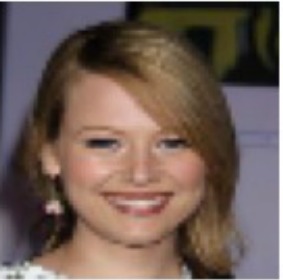} & \ccrop{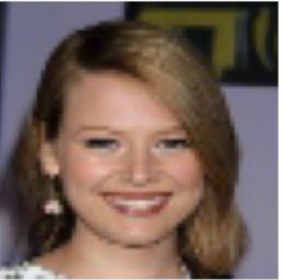} \\[1pt]
    \rowlab{\method}
      & \ccrop{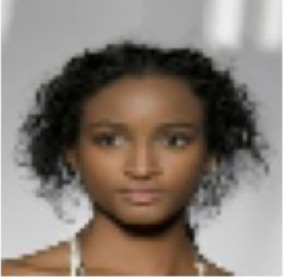} & \ccrop{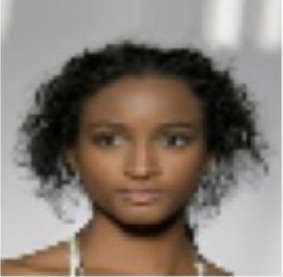} & \ccrop{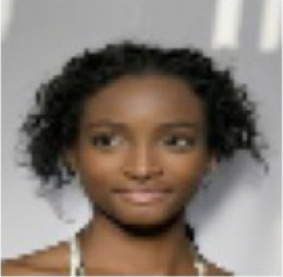} & \ccrop{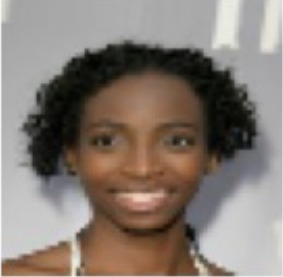} & \ccrop{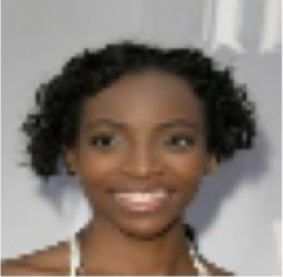}
      & \ccrop{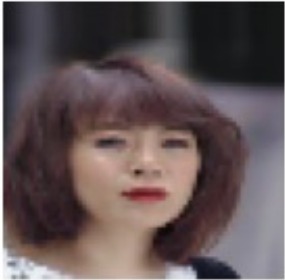} & \ccrop{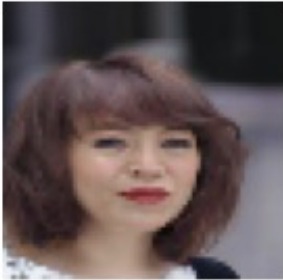} & \ccrop{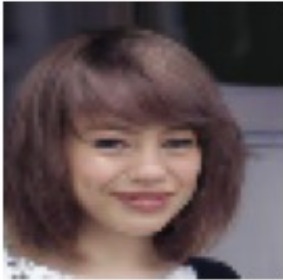} & \ccrop{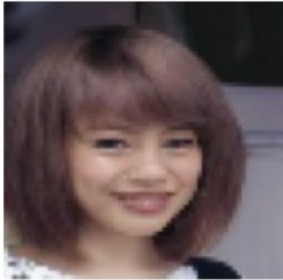} & \ccrop{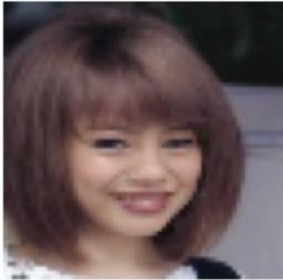} \\[6pt]
    \rowlab{CG}
      & \ccrop{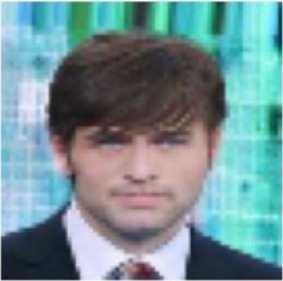} & \ccrop{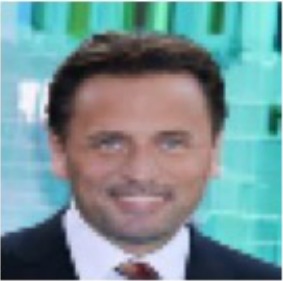} & \ccrop{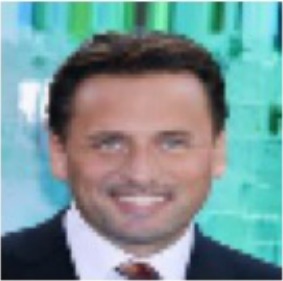} & \ccrop{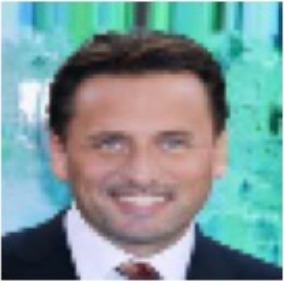} & \ccrop{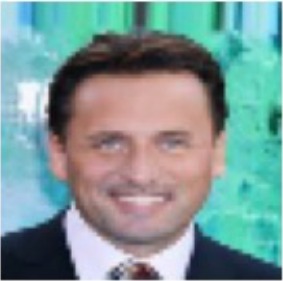}
      & \ccrop{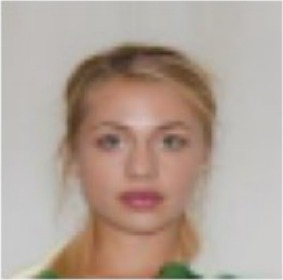} & \ccrop{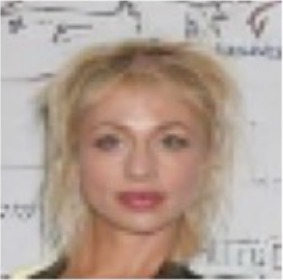} & \ccrop{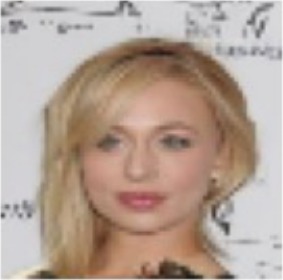} & \ccrop{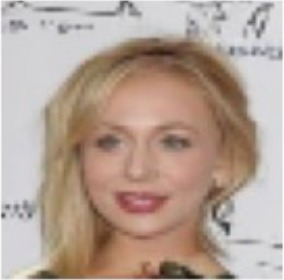} & \ccrop{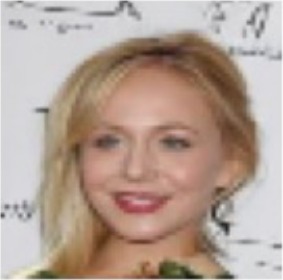} \\[1pt]
    \rowlab{RW}
      & \ccrop{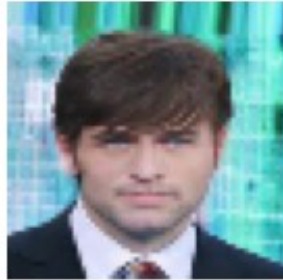} & \ccrop{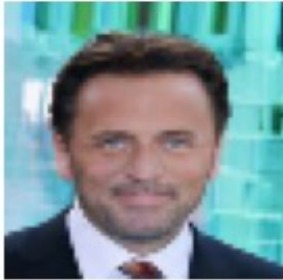} & \ccrop{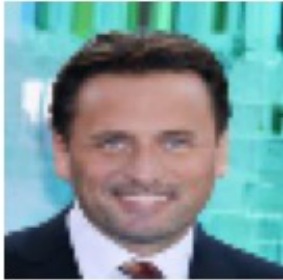} & \ccrop{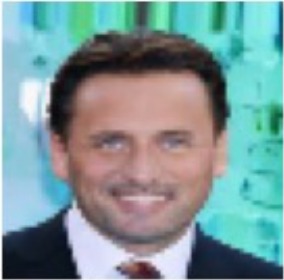} & \ccrop{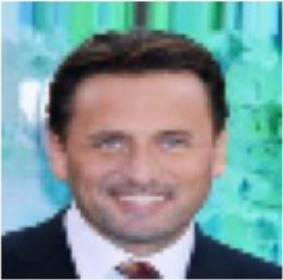}
      & \ccrop{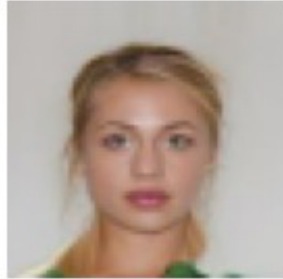} & \ccrop{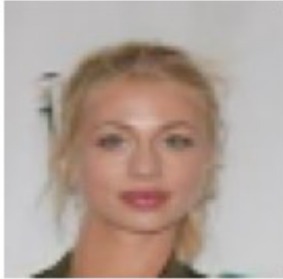} & \ccrop{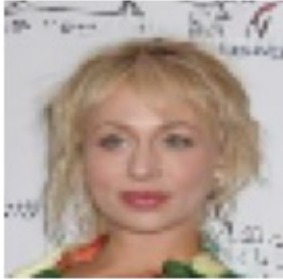} & \ccrop{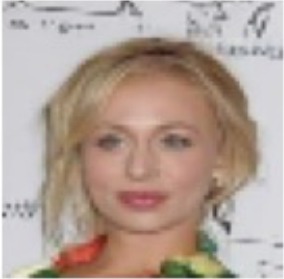} & \ccrop{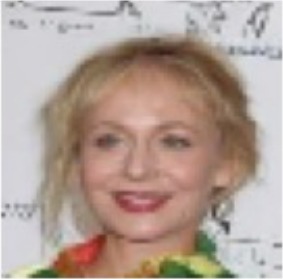} \\[1pt]
    \rowlab{\method}
      & \ccrop{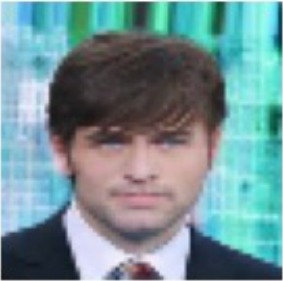} & \ccrop{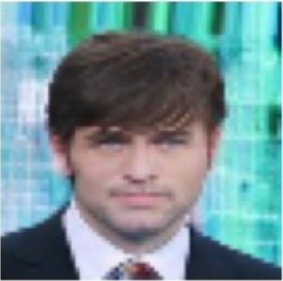} & \ccrop{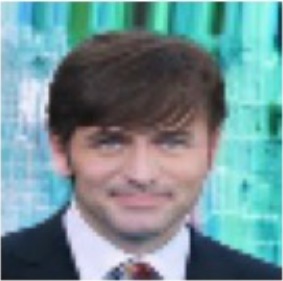} & \ccrop{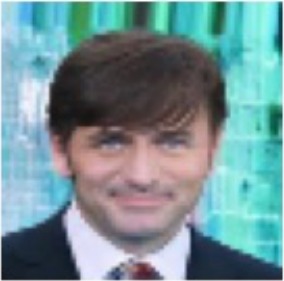} & \ccrop{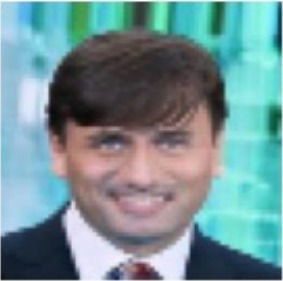}
      & \ccrop{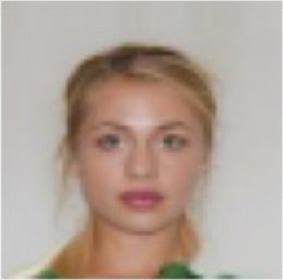} & \ccrop{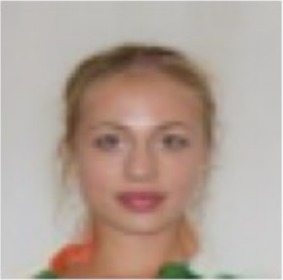} & \ccrop{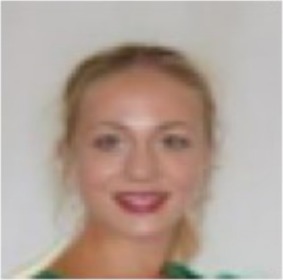} & \ccrop{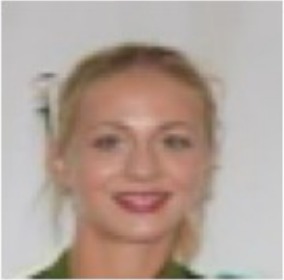} & \ccrop{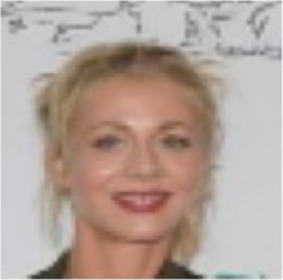} \\[2pt]
  \end{tabular}
  \caption{\textbf{Qualitative examples on class-conditional generation.}
  Four samples per condition (\emph{Blond Hair}, top; \emph{Smiling}, bottom) across guidance scales; rows are CG, RW, and \method.}
  \label{fig:qual-cg}\label{fig:qual-blond}\label{fig:qual-smiling}
\end{figure}

\clearpage
\section{Limitations and Broader Impact}
\label{app:limitations}

\textcolor{black}{We extend the limitations and broader impact discussed in \cref{sec:conclusion}.}

\textbf{Limitations.}
\textcolor{black}{Extending \method to practical multi-group and intersectional settings requires further research on how to select attribute directions.}

\method (Estimate) relies on many model-specific search results to \textcolor{black}{select} prompt-dependent corrections (\cref{app:alpha-estimate}).
Collecting these results requires repeated image generation, adding to the cost of extending the estimator to new models and more diverse prompts.

\textcolor{black}{Our theoretical guarantee is established in an idealized Gaussian VE-SDE setting with a log-affine endpoint potential (\cref{app:proofs_act}). Extending the analysis to more realistic settings with nonlinear neural guidance remains future work.}

\textbf{Broader impact.}
\textcolor{black}{If $\alpha$ is misestimated, \method may shift generation toward one group and worsen the bias it aims to reduce (\cref{app:alpha-estimate}).
Our evaluation relies on binary gender labels from an automatic classifier, which may misrepresent individuals and does not capture non-binary identities.}

\clearpage
\section{Licenses for Existing Assets}
\label{app:asset-licenses}

We list below all third-party assets used in our experiments, together with their licenses.
All assets are used in accordance with their respective terms of use, and the original creators are credited at first use in the main paper.

\textbf{Diffusion Models}
\begin{itemize}
    \item \textbf{Stable Diffusion 1.5}~\citep{rombach2022high}: CreativeML Open RAIL-M License.
    \item \textbf{Stable Diffusion 3 (medium, Diffusers release)}~\citep{esser2024scaling}: Stability AI Non-Commercial Research Community License.
    \item \textbf{ADM-G (guided-diffusion)}~\citep{dhariwal2021diffusion}: MIT License.
    \item \textbf{Unconditional ADM checkpoint}~\citep{ning2023input}: MIT License \textcolor{black}{(from the DDPM-IP codebase)}.
    \item \textbf{FLUX.2 [klein] 9B Base}: FLUX Non-Commercial License (FLUX NCL).
\end{itemize}

\textcolor{black}{\textbf{Software Libraries}}
\begin{itemize}
    \item \textcolor{black}{\textbf{Diffusers}: Apache License 2.0.}
\end{itemize}

\textbf{Baseline Methods}
\begin{itemize}
    \item \textbf{FT (Finetune-Fair-Diffusion)}~\citep{shenfinetuning}: MIT License.
    \item \textbf{UCE (Unified Concept Editing)}~\citep{gandikota2024unified}: MIT License.
    \item \textbf{FD (Fair Diffusion)}~\citep{friedrich2023fair}: Apache License 2.0.
    \item \textbf{PAG (Perturbed-Attention Guidance)}~\citep{ahn2024self}: MIT License.
    \item \textbf{GDRO (Group DRO)}~\citep{sagawa2020distributionally}: MIT License.
\end{itemize}

\textbf{Classifiers and Detectors}
\begin{itemize}
    \item \textbf{CLIP}~\citep{radford2021learning}: MIT License.
    \item \textbf{Ultralytics YOLO11 (YOLO11n)}~\citep{ultralyticsYOLO11}: AGPL-3.0 License.
\end{itemize}

\textcolor{black}{\textbf{Evaluation Metrics and Feature Extractors}}
\begin{itemize}
    \item \textbf{CLIPScore}~\citep{hessel2021clipscore}: MIT License.
    \item \textbf{Aesthetic Scorer (rsinema/aesthetic-scorer)}~\citep{rsinemaAestheticScorer}: MIT License.
    \item \textbf{PickScore}~\citep{kirstain2023pick}: MIT License.
    \item \textcolor{black}{\textbf{DINOv2}~\citep{oquab2024dinov2}: Apache License 2.0.}
\end{itemize}

\textbf{Datasets}
\begin{itemize}
    \item \textbf{CelebA}~\citep{liu2015deep}: Released for non-commercial research purposes only, under the terms specified by the dataset providers (MMLAB, CUHK).
\end{itemize}

\end{document}